\documentclass{article}
\usepackage[a4paper, margin=1in]{geometry}
\usepackage{graphicx} % Required for inserting images
\usepackage{multirow}%
\usepackage{amsmath,amssymb,amsfonts}%
\usepackage{amsthm}%
\usepackage{mathrsfs}%
\usepackage[title]{appendix}%
\usepackage{xcolor}%
\usepackage{textcomp}%
\usepackage{manyfoot}%
\usepackage{booktabs}%
\usepackage{algorithm}%
\usepackage{algorithmicx}%
\usepackage{algpseudocode}%
\usepackage{listings}%

\usepackage{caption}
\usepackage{subcaption}
\usepackage{adjustbox}
\usepackage{authblk}
\usepackage{hyperref}
\usepackage[dvipsnames]{xcolor}
\hypersetup{colorlinks=true, allcolors=NavyBlue}

\newcommand\subsubsubsection{\@startsection{subsubsubsection}{4}{\z@}%
                                     {-3.25ex\@plus -1ex \@minus -.2ex}%
                                     {1.5ex \@plus .2ex}%
                                     {\normalfont\normalsize\bfseries}}
\makeatother

\theoremstyle{thmstyleone}%
\newtheorem{theorem}{Theorem}%  meant for continuous numbers
\newtheorem{lemma}[theorem]{Lemma}% 

\theoremstyle{thmstyletwo}%

\theoremstyle{thmstylethree}%

\makeatletter % added <<<<<<<<<<<<<<<<<<<<<<
\def\@acknow{}%
\long\def\EarlyAcknow#1 \par{%
\def\@acknow{\abstractfont\abstracthead*{Acknowledgments}% or use \subabstracthead <<<<
#1\par}}%

\def\printabstract{\ifx\@acknow\empty\else\@acknow\fi\par%
    \ifx\@abstract\empty\else\@abstract\fi\par}
\makeatother

\renewcommand{\pmb}[1]{\boldsymbol{#1}}
\renewcommand{\P}{\pmb{P}}
\newcommand{\B}{\pmb{B}}
\newcommand{\T}{\pmb{\theta}}
\renewcommand{\b}{\pmb{\beta}}
\renewcommand{\L}{\pmb{L}}
\newcommand{\I}{\pmb{I}}
\newcommand{\Y}{\pmb{Y}}

\DeclareMathOperator*{\argmin}{argmin}
\DeclareMathOperator{\rank}{rank}

\title{Deflation-Free Sparse Optimal Scoring}

\author[1]{{Sharmin} {Afroz}}
\author[2]{{Brendan} {Ames}}
\affil[1]{{Department of Computational $\&$ Informational Sciences}, {Stillman College}, {{Tuscaloosa}, {AL}, {USA}}, \url{safroz@stillman.edu}}
\affil[2]{{School of Mathematical Sciences}, {University of Southampton}, { {Southampton}, {UK}}, \url{b.ames@soton.ac.uk}}

\begin{document}
%+++++++++++++++++++++++++++++++++++++++++++++++++++++++++++++++
%+++++++++++++++++++++++++++++++++++++++++++++++++++++++++++++++
%+++++++++++++++++++++++++++++++++++++++++++++++++++++++++++++++
%Title
\title{Deflation-Free Sparse Optimal Scoring}
\maketitle
%+++++++++++++++++++++++++++++++++++++++++++++++++++++++++++++++
%+++++++++++++++++++++++++++++++++++++++++++++++++++++++++++++++
\begin{abstract}
\emph{Sparse Optimal Scoring (SOS)} reformulates linear discriminant analysis to enable feature selection through elastic net regularization, making it well-suited for high-dimensional settings where the number of features exceeds observations. Most existing SOS methods use deflation-based strategies that compute discriminant vectors sequentially, which can propagate errors and produce suboptimal solutions. We propose a novel approach that estimates all discriminant vectors simultaneously under an explicit global orthogonality constraint, which we call \emph{Deflation-Free Sparse Optimal Scoring (DFSOS)}. DFSOS combines Bregman iteration with orthogonality-constrained optimization, decomposing the problem into tractable subproblems for scoring vectors, discriminant vectors, and orthogonality enforcement. We establish convergence to stationary points of the augmented Lagrangian under mild conditions. Extensive experiments using synthetic data and real-world time series data demonstrate that DFSOS achieves classification accuracy comparable to or better than existing deflation-based methods. These results indicate that deflation-free approaches offer a robust and effective framework for sparse discriminant analysis in high-dimensional problems.
\end{abstract}

%+++++++++++++++++++++++++++++++++++++++++++++++++++++++++++++++

%+++++++++++++++++++++++++++++++++++++++++++++++++++++++++++++++
%+++++++++++++++++++++++++++++++++++++++++++++++++++++++++++++++
%+++++++++++++++++++++++++++++++++++++++++++++++++++++++++++++++
%+++++++++++++++++++++++++++++++++++++++++++++++++++++++++++++++
\section{Introduction}
%+++++++++++++++++++++++++++++++++++++++++++++++++++++++++++++++
%+++++++++++++++++++++++++++++++++++++++++++++++++++++++++++++++
%+++++++++++++++++++++++++++++++++++++++++++++++++++++++++++++++
%+++++++++++++++++++++++++++++++++++++++++++++++++++++++++++++++

% \textbf{Review LDA}\\
Linear Discriminant Analysis (LDA) is a classical supervised learning technique widely used for classification, dimension reduction, and data visualization. Given a dataset with $n$ observations, $p$ features, and $K$ classes, LDA seeks a low-dimensional linear subspace in which class separation is maximized. Under the multivariate Gaussian model, LDA assumes that observations from class $i$ follow a normal distribution $N(\pmb{\mu}_i, \pmb{\Sigma})$ with a common covariance matrix across classes. Classification is performed by maximizing posterior class probabilities, which leads to linear decision boundaries. This Bayesian formulation yields discriminant functions of the form
\begin{equation*}
    \delta_i(x) = x^T\pmb{\Sigma}^{-1}\pmb{\mu}_i - \frac{1}{2}\mu_i^T\pmb{\Sigma}^{-1}\pmb{\mu}_i +\text{log }\pi_i.
\end{equation*}
where $\pi_{i}$ denotes the prior probability of class $i$.
We will focus on an equivalent framework for LDA based on \emph{optimal scoring} \cite{hastie1994flexible, hastie1995penalized}. Here, LDA is reformulated as a linear regression problem with numerical scores assigned to class labels. Under appropriate constraints, this optimal scoring approach is known to coincide with the Bayesian formulation (see \cite[Section 4.3]{hastie2009elements} and \cite{mai2013note}).

%%%%%%%%%%%%%%%%%%%%%%%%%%%%%%%%%%%%%%%%%%%%%%%%%%%%%%%%%%
Fisher proposed an alternative formula of LDA based on maximizing the Rayleigh quotient of the between-class and within-class scatter matrices of the training data. This reduces LDA to a generalized eigenvalue problem.
Unfortunately, when the number of features $p$ is larger than the number of observations $n$, the sample covariance matrix and sample within-class scatter matrices are singular; in this case, the change of basis/variables needed to perform LDA is impossible to compute. This shortcoming of LDA inspired \emph{penalized} or \emph{sparse discriminant analysis} (SDA) approaches which introduce regularization, particularly sparsity-inducing penalties, into the discriminant analysis framework so that only a subset of informative features is used for classification  \cite{hastie1995penalized, zou2005regularization, shao2011sparse, fan2012road, mai2012direct, guo2007regularized, cai2011direct, witten2011penalized, clemmensen2011sparse}.

% A major advance in this direction was proposed by Clemmensen et al. \cite{clemmensen2011sparse}, who incorporated elastic net regularization into the optimal scoring formulation of LDA. The resulting Sparse Optimal Scoring (SOS) problem takes the form
% \begin{equation}\label{eq: Sparse Discriminant Analysis problem}
% \begin{split}
% \underset{\pmb{\beta}_i, \pmb{\theta}_i}{\text{min}}\ &\|\pmb{Y}\pmb{\theta}_i - \pmb{X}\pmb{\beta}_i\|^2 + \gamma \pmb{\beta}_i^T \pmb{\Omega} \pmb{\beta}_i + \lambda \|\pmb{\beta}_i\|_1\\
% \text{s.t.}\ &\frac{1}{n}\pmb{\theta}_i^T \pmb{Y}^T \pmb{Y}\pmb{\theta}_i = 1\\
% &\pmb{\theta}_i^T\pmb{Y}^T\pmb{Y}\pmb{\theta}_{\ell} = 0 \ \ \ \ \forall \ell < i,\\
% \end{split}
% \end{equation}
% where $\pmb{\Omega}$ is the positive definite matrix, and $\lambda$ and $\gamma$ are nonnegative tuning parameters.

% SDA is typically solved using block coordinate descent, alternating between updates of the scoring vectors $\pmb{\theta}_{i}$ and discriminant vectors $\pmb{\beta}_{i}$. Given $\pmb{\beta}_{i}$, the $\pmb{\theta}_{i}$ subproblem admits a closed-form solution, while the $\pmb{\beta}_{i}$ subproblem reduces to a convex elastic net regression problem that can be solved using proximal gradient or ADMM-based methods

% \textbf{Motivation}\\

Many existing sparse discriminant analysis methods find multiple discriminant directions using a deflationary process. After each discriminant direction is computed, the next one is forced to be orthogonal to the previous ones. While this step-by-step process is relatively easy to implement, it has important limitations. Error in the early directions can carry over and affect the later ones, which can hurt both interpretability and classification performance. In addition, this sequential procedure makes parallel computation difficult and hides the overall structure of the optimization problem.

These limitations motivate the development of  deflation-free methods. Instead of computing discriminant vectors one at a time, these methods estimate all directions at once while directly enforcing orthogonality between them. This method maintains sparsity while avoiding the numerical and statistical problems caused by sequential deflation. 
% This concept underlies the deflation-free sparse optimal scoring framework presented later.

%+++++++++++++++++++++++++++++++++++++++++++++++++++++++++++++++
%+++++++++++++++++++++++++++++++++++++++++++++++++++++++++++++++
\subsection{Contribution}
We propose a novel formulation for sparse discriminant analysis that jointly estimates multiple discriminant vectors under a global orthogonality constraint, eliminating the need for sequential deflation and avoiding error propagation inherent in deflation-based approaches. This approach is based on a modification of the sparse optimal scoring approach considered in \cite{clemmensen2011sparse, atkins2023proximal}. Specifically, we reformulate the deflation-based sparse optimal scoring problem using a global orthogonality constraint. We then propose and analyse a block-coordinate descent method for solution of this orthogonality-constrained optimization problem. We establish conditions for convergence of this heuristic and observe via empirical trials that it is competitive in terms of classification performance and computational complexity.

%+++++++++++++++++++++++++++++++++++++++++++++++++++++++++++++++
%+++++++++++++++++++++++++++++++++++++++++++++++++++++++++++++++

%+++++++++++++++++++++++++++++++++++++++++++++++++++++++++++++++
%+++++++++++++++++++++++++++++++++++++++++++++++++++++++++++++++
\subsection{Outline}
In Section 2, we review sparse optimal scoring and existing deflation-based block coordinate descent approaches, including their optimization properties and limitations.
In Section 3, we introduce the proposed deflation-free sparse optimal scoring formulation and present the DFSOS algorithm, including detailed derivations of each subproblem, implementation details, and convergence analysis.
In Section 4, we evaluate the empirical performance of DFSOS on synthetic Gaussian data and real time-series classification benchmarks, comparing accuracy, sparsity, and computational efficiency with existing methods.
Finally, Section 5 concludes the paper with a discussion of implications, limitations, and directions for future research.
%+++++++++++++++++++++++++++++++++++++++++++++++++++++++++++++++
%+++++++++++++++++++++++++++++++++++++++++++++++++++++++++++++++

%+++++++++++++++++++++++++++++++++++++++++++++++++++++++++++++++
%+++++++++++++++++++++++++++++++++++++++++++++++++++++++++++++++
%+++++++++++++++++++++++++++++++++++++++++++++++++++++++++++++++
%+++++++++++++++++++++++++++++++++++++++++++++++++++++++++++++++
\section{Sparse Optimal Scoring}
\subsection{The Optimal Scoring Problem}
%+++++++++++++++++++++++++++++++++++++++++++++++++++++++++++++++
%+++++++++++++++++++++++++++++++++++++++++++++++++++++++++++++++
%+++++++++++++++++++++++++++++++++++++++++++++++++++++++++++++++
%+++++++++++++++++++++++++++++++++++++++++++++++++++++++++++++++
% \td{Define optimal scoring}\\
As mentioned previously, the \textit{optimal scoring problem} seeks to recast the classification problem as a regression problem by using a sequence of scorings to turn the categorical class-labels into quantitative variables. Here, we formally discuss the specialization of this framework to linear discriminant analysis.

Let $\pmb{X} \in \mathbb{R}^{n \times p}$ denote a data matrix with rows encoding $n$ observations of $p$-dimensional data vectors and let $\pmb{Y}$ be an $n \times K$ indicator matrix for the $K$ classes. It is defined as $Y_{ij} = 1$ if the $i$th observation belongs to the $j$th class, $Y_{ij} = 0$ otherwise.  Let $\pmb{\theta}_i \in \mathbb{R}^K$ denote the $i$th scoring vector and $\pmb{\beta}_i \in \mathbb{R}^p$ be the $i$th discriminant vector.
The optimal scoring problem seeks to jointly optimize the scoring-discriminant vector $(\pmb{\theta}_i, \pmb{\beta}_i)$ so that the sum of squared differences between class labels encoded by $\pmb{\theta}_i$ and the linear regression model encoded by $\pmb{\beta}_i$ is minimized:
\begin{equation}\label{general Optimal Scoring Problem}
\begin{array}{cl}
\displaystyle{\min_{\pmb{\beta}_i, \pmb{\theta}_i}}\ &\|\pmb{Y}\pmb{\theta}_i - \pmb{X}\pmb{\beta}_i\|^2\\
\text{s.t.} &\frac{1}{n}\pmb{\theta}_i^T \pmb{Y}^T \pmb{Y}\pmb{\theta}_i = 1\\
&\pmb{\theta}_i^T\pmb{Y}^T\pmb{Y}\pmb{\theta}_{\ell} = 0 \ \ \ \ \forall \ell < i.\\
\end{array}
\end{equation}
We assume that the scoring vector $\pmb{\theta}_i$ is normalized and orthogonal to
the previously computed $i -1$ scoring vectors to make the problem well-posed. This process is inherently deflationary as we must compute scoring-discriminant vector pairs one by one.

%+++++++++++++++++++++++++++++++++++++++++++++++++++++++++++++++
%+++++++++++++++++++++++++++++++++++++++++++++++++++++++++++++++
% SOS problem.
LDA has been observed to yield accurate classifiers when applied to data with a small number of predictor variables, particularly for Gaussian data when each class has a shared population covariance matrix. However, it can overfit the data in problems with large numbers of highly correlated variables. Conversely, when the class boundaries are nonlinear it can underfit the data \cite{hastie1995penalized}. Moreover, when the number of features $p$ is much larger than the number of observations $n$ i.e., $p > n$, the sample covariance matrix $\pmb{\Sigma}$, proportional to $\pmb{X}^T \pmb{X}$, does not have full rank.
To address these issues, Clemmensen et al.~\cite{clemmensen2011sparse} introduce regularization in the form of an elastic net penalty function~\cite{zou2005regularization}.
This yields the \emph{sparse} or \emph{penalized optimal scoring problem (SOS)}:
\begin{equation}\label{eq: Sparse Optimal Scoring Criterion problem}
\begin{array}{cl}
\displaystyle{\min_{\pmb{\beta}_i, \pmb{\theta}_i}} &\|\pmb{Y}\pmb{\theta}_i - \pmb{X}\pmb{\beta}_i\|^2 + \gamma \pmb{\beta}_i^T \pmb{\Omega} \pmb{\beta}_i + \lambda \|\pmb{\beta}_i\|_1\\
\text{s.t.}\ &\frac{1}{n}\pmb{\theta}_i^T \pmb{Y}^T \pmb{Y}\pmb{\theta}_i = 1\\
&\pmb{\theta}_i^T\pmb{Y}^T\pmb{Y}\pmb{\theta}_{\ell} = 0 \ \ \ \ \forall \ell < i.\\
\end{array}
\end{equation}
It utilizes both the lasso penalty ($\ell_1$-norm), to induce sparsity, and the ridge penalty ($\ell_2$-norm), to encourage selection of clusters of correlated variables \cite{zou2005regularization}. More details on the use of the naive elastic net penalty can also be found in \cite[Section 18.4]{hastie2009elements}. Note that the generalized elastic net penalty differs from the naive elastic net only in the use of the $\ell_2$-norm with respect to a different basis.

%++++++++++++++++++++++++++++++++++++++++++++++++++++++++++++++++++
%++++++++++++++++++++++++++++++++++++++++++++++++++++++++++++++++++
%++++++++++++++++++++++++++++++++++++++++++++++++++++++++++++++++++
% Deflationary approach.
\subsection{A Block Coordinate Descent Approach for Sparse Optimal Scoring}
\label{sec: BCD for SOS}
%++++++++++++++++++++++++++++++++++++++++++++++++++++++++++++++++++
%++++++++++++++++++++++++++++++++++++++++++++++++++++++++++++++++++
%++++++++++++++++++++++++++++++++++++++++++++++++++++++++++++++++++
Clemmensen et al.~\cite{clemmensen2011sparse} propose a heuristic for solving~\eqref{general Optimal Scoring Problem} using block coordinate descent with respect to $\pmb{\theta}_i$ and $\pmb{\beta}_i$.
%++++++++++++++++++++++++++++++++++++++++++++++++++++++++++++++++++
% Theta subproblem.
%++++++++++++++++++++++++++++++++++++++++++++++++++++++++++++++++++
To update $\pmb{\theta}_i$, we fix $\pmb{\beta}_i$ and solve the problem
\begin{equation}\label{theta subproblem}
\begin{array}{cl}
\min_{\pmb{\theta}_i} &\|\pmb{Y}\pmb{\theta}_i - \pmb{X}\pmb{\beta}_i\|^2 \\
\text{s.t.}\ &\frac{1}{n}\pmb{\theta}_i^T \pmb{Y}^T\pmb{Y} \pmb{\theta}_i = 1\\
&\pmb{\theta}_i^T\ \pmb{Y}^T\pmb{Y} \pmb{\theta}_{\ell} = 0 \ \ \ \ \forall \ell < i.\\
\end{array}
\end{equation}
Although this is a non-convex quadratic program in $\pmb\theta_i$,~\eqref{theta subproblem} admits a closed-form solution, given in the following theorem.

\begin{theorem}\label{Thm: optimal theta solution}
The problem (\ref{theta subproblem}) has optimal solution 
\begin{equation}
\pmb{\theta}_i^{\text{new}} = s (\pmb{I} - \pmb{Q}_i \pmb{Q}_i^T \pmb{D}) \pmb{D}^{-1} \pmb{Y}^T \pmb{X} \pmb{\beta}_i
\end{equation}
where $\pmb{D}=\frac{1}{n}\pmb{Y}^T\pmb{Y}$, $\pmb{Q}_i$ is the $K \times i$ matrix with columns consisting of the previous $i-1$ scoring vectors $\pmb{\theta}_1, \pmb{\theta}_2,...,\pmb{\theta}_{i-1}$ and the all-ones vector $\textbf{1} \in \mathbb{R}^K$, and $s$ is a proportionality constant ensuring that $\pmb{\theta}_i^T\pmb{D}\pmb{\theta}_i=1$. In particular, $\pmb{\theta}_i$ is given by
\begin{equation}
\textbf{w} = (\pmb{I} - \pmb{Q}_i \pmb{Q}_i^T \pmb{D}) \pmb{D}^{-1} \pmb{Y}^T \pmb{X}	\pmb{\beta}_i, \ \ \ \pmb{\theta}_i = \frac{\pmb{w}}{\sqrt{\pmb{w}^T \pmb{D} \pmb{w}}}.
\end{equation}
\end{theorem}

%++++++++++++++++++++++++++++++++++++++++++++++++++++++++++++++++++
% Beta update.
%++++++++++++++++++++++++++++++++++++++++++++++++++++++++++++++++++
After updating $\pmb{\theta}_i$, we fix $\pmb{\theta}_i$ and solve the $\pmb{\beta}$-subproblem
\begin{equation}\label{beta subproblem}
\pmb{\beta}_i^{\text{new}} = \argmin_{\pmb{\beta}} \|\pmb{Y}\pmb{\theta}_i - \pmb{X}\pmb{\beta}\|^2 + \gamma \pmb{\beta}^T \pmb{\Omega} \pmb{\beta} + \lambda \|\pmb{\beta}\|_1.
\end{equation}
This is a convex optimization problem in $\pmb{\beta}$ and can be solved efficiently using iterative methods. Clemmensen et al.~propose the use of \emph{least angle regression (LARS)} \cite{efron2004least, zou2005regularization, rosset2007piecewise} to solve~\eqref{beta subproblem}. Alternately, Atkins et al.~\cite{atkins2023proximal} consider solution of~\eqref{beta subproblem} using the \emph{accelerated proximal gradient method} and the \emph{alternating direction of method of multipliers}.

In either case, we solve~\eqref{beta subproblem} and then use this new value of $\pmb{\beta}$ to update $\pmb{\theta}$ and so on.
We repeat this process until the iterates $(\pmb{\theta}_i, \pmb{\beta}_i)$ converge.
%++++++++++++++++++++++++++++++++++++++++++++++++++++++++++++++++++
% Classification and convergence.
%++++++++++++++++++++++++++++++++++++++++++++++++++++++++++++++++++
After finding the optimal scoring-discriminant vector pairs, classification is performed using nearest centroid classification on the projection of the data onto the space spanned by the discriminant vectors, given by $[\pmb{X} \pmb{\beta}_1 \dots \pmb{X} \pmb{\beta}_q]$ where $q < K$ ~\cite{clemmensen2011sparse}.

This iterative approach for updating $\pmb{\theta}$ and $\pmb{\beta}$ is summarized in Algorithm \ref{SDA Alg}. The following theorem establishes that the sequence of function values for SOS given by Algorithm \ref{SDA Alg} is convergent; see \cite[Theorem 2.4]{atkins2023proximal}.

\begin{theorem}
Suppose that the sequence of iterates $\{(\pmb{\theta}^t,\pmb{\beta}^t)\}_{t=0}^\infty$ is generated by Algorithm \ref{SDA Alg}. Then the sequence of objective function values $\{F(\pmb{\theta}^t,\pmb{\beta}^t)\}_{t=0}^\infty$ defined by $F(\pmb{\theta},\pmb{\beta})  := \|\pmb{Y}\pmb{\theta} - \pmb{X}\pmb{\beta}\|^2 + \gamma\pmb{\beta}^T\pmb{\Omega\beta} + \lambda\|\pmb{\beta}\|_1$ is convergent.
\end{theorem}

It can also be established that every convergent subsequence of $\{(\pmb{\theta}^t,\pmb{\beta}^t)\}_{t=1}^\infty$ converges to a stationary point of (\ref{eq: Sparse Optimal Scoring Criterion problem}), see \cite[Theorem 2.5]{atkins2023proximal}.
\begin{theorem}
Let $\{(\pmb{\theta}^t,\pmb{\beta}^t)\}_{t=1}^\infty$ be the sequence of points generated by Algorithm \ref{SDA Alg}. Suppose that $\{(\pmb{\theta}^{t_\ell},\pmb{\beta}^{t_\ell})\}_{\ell=1}^\infty$ is a convergent subsequence of $\{(\pmb{\theta}^t,\pmb{\beta}^t)\}_{t=1}^\infty$ with limit $(\pmb{\theta}^{\ast},\pmb{\beta}^{\ast})$. Then $(\pmb{\theta}^{\ast},\pmb{\beta}^{\ast})$ is a stationary point of (\ref{eq: Sparse Optimal Scoring Criterion problem}): $(\pmb{\theta}^{\ast},\pmb{\beta}^{\ast})$ is feasible for (\ref{eq: Sparse Optimal Scoring Criterion problem}) and there exists $\psi^{\ast} \in \mathbb{R}$ and $\pmb{\nu}^{\ast} \in \mathbb{R}^{j-1}$ such that $\textbf{0} \in \partial\pmb{L}(\pmb{\theta}^{\ast},\pmb{\beta}^{\ast},\psi^{\ast},\pmb{\nu}^{\ast})$, where $\partial\pmb{L}(\pmb{\theta},\pmb{\beta},\psi,\pmb{\nu})$ denotes the subdifferential of the Lagrangian function $\pmb{L}$ with respect to the primal variables $(\pmb{\theta},\pmb{\beta})$.
\end{theorem}

%++++++++++++++++++++++++++++++++++++++++++++++++++++++++++++++++++
% BCD Algorithm
%++++++++++++++++++++++++++++++++++++++++++++++++++++++++++++++++++
% \begin{algorithm}[t!]
% \caption{Sparse Discriminant Analysis \cite{clemmensen2011sparse,atkins2023proximal}}
% \label{SDA Alg}
% \begin{algorithmic}[1]
% \State Given initial iterate $\pmb{\theta}^0$
% \State Let $\pmb{Q}_1$ be a $K \times 1$ matrix of 1's.
% \For {i = 1,2,... \text{until converged}}	
% 	\State Update $\pmb{\theta}_i$ by
% 	\begin{equation*}
% 	\begin{array}{cl}
% 	\textbf{w} &= (\pmb{I} - \pmb{Q}_i \pmb{Q}_i^T \pmb{D}) \pmb{D}^{-1} \pmb{Y}^T \pmb{X}	\pmb{\beta}_i\\
% 	\pmb{\theta}_i &= \frac{\pmb{w}}{\sqrt{\pmb{w}				^T \pmb{D} \pmb{w}}}\\
% 	\end{array}
% 	\end{equation*}
% 	\If {$i < q$}
% 		\State Set $\pmb{Q}_{i+1} = (\pmb{Q}_i : 						\pmb{\theta}_i)$
% 	\EndIf
    
% \EndFor
% \State The classification rule results from performing standard LDA with the $n \times q$ matrix $(\pmb{X} \pmb{\beta}_1 \dots \pmb{X} \pmb{\beta}_q)$
% \end{algorithmic}
% \end{algorithm}

\newcommand{\Q}{\bs{Q}}
\newcommand{\bt}{\boldsymbol{\theta}}
\newcommand{\bb}{\boldsymbol{\beta}}
\newcommand{\z}{\boldsymbol{z}}
\newcommand{\bs}[1]{\boldsymbol{#1}}
\begin{algorithm}[t]
\begin{algorithmic}[1]
\caption{Sparse Discriminant Analysis \cite{clemmensen2011sparse,atkins2023proximal}}
\label{SDA Alg}    

    \State
    Given stopping tolerance $\epsilon$ and maximum number of iterations $N$.
  
    \State
    \Return scoring-discriminant vector pairs $(\bt^*_1, \bb^*_1)$, $(\bt^*_2, \bb^*_2)$, \dots, $(\bt^*_{K-1}, \bb^*_{K-1})$.

 \For{$j = 1,2, \dots, K-1$}
 
    \State
    Initialize $\bt_j^0$:
    \[
      \bt_j^0 = \left(\I - \frac{1}{n}\bs Q_j \bs Q_j^T \Y^T\Y \right)(\Y^T \Y)^{-1} \z, \hspace{0.25in}
      \bt_j^0 = \frac{\sqrt{n}\bt_j^0}{\|\Y \bt_j^0\|}.
    \]
    
    \medskip
    \State
    Calculate the $j$th scoring and discriminant vector pair $(\bt^*_j, \bb^*_j)$:
    \smallskip
    
    \For{$i = 0, 1,2 \dots N$}

        \State
        % Update beta.
        Update $\bb^i_j$ as the solution of \eqref{beta subproblem} with $\bt=\bt^i_j$.

        % Update theta.
        \State 
        Update $\bt^{i+1}_j$ by
        \[
          \bs{w} = \left(\I - \frac{1}{n}\bs Q_j \bs Q_j^T \Y^T\Y \right)(\Y^T \Y)^{-1} \Y^T \bs X {\bb^i_j}, %\\
          \hspace{0.25in}
          \bt^{i+1} = \frac{\sqrt{n}\bs{w}}{\|\Y \bs{w}\|};\;
          \]
          
        % Convergence test.
        \State
        Converged if the residual between consecutive iterates is smaller than  tolerance:
        % Let
        % \[
        %   r = \max\bra{ \frac{\|\bt^{i+1} - \bt^i\|}{\|\bt^{i+1}\|}, \frac{\|\bb^{i+1} - \bb^i\|}{\|\bb^{i+1}\|} };
        % \]}
        
        \bigskip
        \If{$\max\left\{ \frac{\|\bt_j^{i+1} - \bt_j^i\|}{\|\bt_j^{i+1}\|}, \frac{\|\bb_j^{i+1} - \bb_j^i\|}{\|\bb_j^{i+1}\|} \right\} < \epsilon$}
            \State 
            The algorithm has converged.
            
            \State
            Increment $j$ and \textbf{break}
        \EndIf
    \EndFor
    
\EndFor

\end{algorithmic}
\end{algorithm}

%++++++++++++++++++++++++++++++++++++++++++++++++++++++++++++++++++
%++++++++++++++++++++++++++++++++++++++++++++++++++++++++++++++++++
%++++++++++++++++++++++++++++++++++++++++++++++++++++++++++++++++++
% Deflation-free SOS
\section{Deflation-Free Sparse Optimal Scoring}
\label{sec: dfsos}

Most existing SOS algorithms rely on deflation-based strategies. Deflation-based methods impose orthogonality implicitly and locally, rather than enforcing a global orthogonality constraint across all discriminant vectors. As a result, the final set of discriminant directions may be sensitive to initialization, ordering, and numerical errors, and may fail to optimally capture class separation in a joint sense. These challenges become more pronounced in high-dimensional settings ( $p\gg n)$, where sparse regularization already introduces nonconvexity and instability. To overcome these limitations, deflation-free methods reformulate SOS as a single orthogonality-constrained optimization problem, allowing all discriminant and scoring vectors to be computed simultaneously. By enforcing orthogonality at the matrix level and avoiding sequential deflation, deflation-free approaches eliminate error accumulation, improve numerical stability, and yield more coherent discriminant subspaces. This motivates the development of algorithms such as Deflation-Free Sparse Optimal Scoring (DFSOS), which leverage splitting techniques and Bregman iteration to efficiently handle orthogonality constraints while preserving sparsity and scalability in high-dimensional settings.

%++++++++++++++++++++++++++++++++++++++++++++++++++++++++++++++++++
% Statement of problem.
%++++++++++++++++++++++++++++++++++++++++++++++++++++++++++++++++++
We consider the following variant of the sparse optimal scoring problem, which we call \emph{deflation-free sparse optimal scoring}:
\begin{equation}\label{eq: DFSOS splitting}
    \begin{array}{cl}
        \underset{\pmb{\theta}\in \mathbb{R}^{K\times q},\pmb{\beta}\in\mathbb{R}^{p\times q}}{\min} &  J(\pmb{\theta}, \pmb{\beta}) := \|\pmb{Y}\pmb{\theta}-\pmb{X}\pmb{\beta}\|_{F}^{2}+\gamma\pmb{\beta}^{T}\pmb{\beta}+\lambda\|\pmb{\beta}\|_{1}\\
        \text{s.t.} & \frac{1}{n}\pmb{\theta}^{T}\pmb{Y}^{T}\pmb{Y}\pmb{\theta}=\pmb{I}.
    \end{array}
\end{equation}
Note that~\eqref{eq: DFSOS splitting} requires minimization of a quadratic function in $(\pmb{\theta}, \pmb{\beta})$ over the manifold of $K\times q$ orthogonal matrices (with respect to the inner product $\left<\pmb{\theta_1}, \pmb{\theta_2} \right> = \frac{1}{n}\pmb{\theta_1}^T \pmb{Y}^{T}\pmb{Y}\pmb{\theta_2}$).
Note further that we substitute $\boldsymbol{\Omega} = \boldsymbol{I}$ in the ridge regression term in~\eqref{eq: Sparse Optimal Scoring Criterion problem}; our problem and algorithm can be modified to allow generalized Tikhonov regression, but we focus on this specialization for the sake of model and algorithmic simplicity.
This problem is non-convex. Moreover, optimization over the manifold of orthogonal matrices is NP-hard; see \cite{lai2025stiefel, lai2025grassmannian} and the set of  reductions given in \cite[Section 1]{jiang2023exact}.

%++++++++++++++++++++++++++++++++++++++++++++++++++++++++++++++++++
% Splitting
\subsection{A Splitting Method for Deflation-Free Sparse Optimal Scoring}
%++++++++++++++++++++++++++++++++++++++++++++++++++++++++++++++++++
Due to this worst-case intractability, much recent research has focused on the design of heuristics for solving orthogonally-constrained optimization problems \cite{absil2008optimization, lai2014splitting, wen2013feasible, gao2021riemannian, li2021weakly, liu2020simple, he2025relatively, chen2024nonsmooth, bonnabel2013stochastic, becigneul2018riemannian, zhang2016first}.
In particular, we will leverage a method recently proposed by Lai and Osher~\cite{lai2014splitting} to approximately solve~\eqref{eq: DFSOS splitting}.
This approach uses matrix splitting and Bregman iteration, to generate a sequence of approximate solutions of~\eqref{eq: DFSOS splitting} from the solution of easier to solve subproblems.

We introduce the
primal variable $\pmb{P}=\frac{1}{\sqrt{n}}\pmb{Y\theta}$. Then the DFSOS problem can be written as 
\begin{equation}\label{eq: DFSOS final splitting}
    \begin{array}{cl}
       \underset{\pmb{\theta}\in \mathbb{R}^{K\times q},\pmb{\beta}\in\mathbb{R}^{p\times q}}{\arg\min} & J(\pmb{\theta}, \pmb{\beta}) \\
       \text{s.t.} &\pmb{P}=\frac{1}{\sqrt{n}}\pmb{Y\theta}, \quad \pmb{P}^{T}\pmb{P}=\pmb{I}.
    \end{array}
\end{equation}
Using Bregman iteration the optimal solution of (\ref{eq: DFSOS final splitting}) can be found iteratively using the update formulas 
 \begin{align} \label{eq: DFSOS iterative equations}
  (\pmb{\theta}^{k}, \pmb{\beta}^{k}, \pmb{P}^{k}) &=  \underset{\pmb{\theta},~\pmb{\beta},~ \pmb{P}\in \mathbb{R}^{n}}{\arg\min} ~J(\pmb{\theta}^{k}, \pmb{\beta}^{k}) + \frac{\rho}{2}\|\pmb{LX}-\pmb{P}^{k-1}+\pmb{B}^{k-1}\|_{F}^{2}, \quad \text{s.t.}  \quad \pmb{P}^{T}\pmb{P} = \pmb{I}, \\
		\pmb{B}^{k} &=  \pmb{B}^{k-1}+\pmb{L}\pmb{X}^{k}-\pmb{P}^{k}, \label{eq: beta update}
\end{align}
where $\pmb{L} := \frac{1}{\sqrt{n}}\pmb{Y}$.
We will approximately solve the subproblem~\eqref{eq: DFSOS iterative equations}
using block-coordinate descent, where we minimize with respect to each primal variable $\pmb{\theta}$, $\pmb{\beta}$, and $\pmb{P}$ individually with the others fixed.

%++++++++++++++++++++++++++++++++++++++++++++++++++++++++++++++++++
%++++++++++++++++++++++++++++++++++++++++++++++++++++++++++++++++++
% Updating Theta.
\subsection{The $\pmb{\theta}$ sub-problem}\label{sec: theta update}
%++++++++++++++++++++++++++++++++++++++++++++++++++++++++++++++++++
%++++++++++++++++++++++++++++++++++++++++++++++++++++++++++++++++++

We first fix $\pmb{\beta}, \pmb{P}, \pmb{B}$ and describe the process for updating  $\pmb{\theta}$.
In this case,
$\pmb{\theta}^{\text{new}}$ is the solution of 
\begin{align*}
    \underset{\pmb{\theta}}{\min} & \quad \sum_{i=1}^{q} \|\pmb{Y}\pmb{\theta}_{i}-\pmb{X}\pmb{\beta}_{i}\|^{2}+\gamma\|\pmb{\beta}\|^{2}+\lambda\|\pmb{\beta}\|_{1}+\frac{\rho}{2} \left\|\frac{1}{\sqrt{n}}\pmb{Y\theta}_{i}-\pmb{P}_{i}+\pmb{B}_{i} \right\|^{2} \\
    \text{s.t.} & \quad \pmb{\theta}^{T}\pmb{Y}^{T}\pmb{Y}\pmb{1}=0, \quad i=1,2,....,q.    
\end{align*}
This problem is separable with respect to the columns of $\pmb{\theta}$.
Thus,
$\pmb{\theta}_{i}^{\text{new}}$, the $i$th column of $\pmb{\theta}$ or $i$th scoring vector, is the minimizer of 
\begin{equation}\label{eq: theta update equation}
    \begin{array}{cl}
        \displaystyle\min_{\pmb{\theta}} & \displaystyle \sum_{i=1}^{q} \|\pmb{Y}\pmb{\theta}_{i}-\pmb{X}\pmb{\beta}_{i}\|^{2}+\gamma\|\pmb{\beta}_{i}\|^{2}+\lambda\|\pmb{\beta}_{i}\|_{1}+\frac{\rho}{2} \left\|\frac{1}{\sqrt{n}}\pmb{Y\theta}_{i}-\pmb{P}_{i}+\pmb{B}_{i} \right\|^{2}\\
        \text{s.t.} & \pmb{\theta}_{i}^{T}\pmb{Y}^{T}\pmb{Y}\pmb{1}=0.
    \end{array}
\end{equation}
The objective function of Equation (\ref{eq: theta update equation}) is a convex quadratic in $\pmb{\theta}_{i}$:
\begin{equation}\label{eq: theta objective equation}
    F(\pmb{\theta}_{i}) = \pmb{\theta}_{i}^{T}\pmb{Y}^{T}\pmb{Y\theta}_{i}-2\pmb{\theta}_{i}^{T}\pmb{Y}^{T}\pmb{X\beta}_{i}+\frac{\rho}{2} \left(\frac{1}{n}\pmb{\theta}_{i}^{T}\pmb{Y}^{T}\pmb{Y\theta}_{i}-\frac{2}{\sqrt{n}}\pmb{\theta}_{i}^{T}\pmb{Y}^{T}(\pmb{P}_{i}-\pmb{B}_{i})\right)+C
\end{equation}
where $C$ represents all the constant terms depending only on $\pmb{\beta}, \pmb{P}$ and $\pmb{B}$. 
Therefore, the sub-problem for $\pmb{\theta}_{i}$ can be reformulated as: 
\begin{equation}
    \begin{array}{cl}
        \underset{\pmb{\theta}_{i}}{\min} &\left(1 + \frac{\rho}{2n} \right)\pmb{\theta}_{i}^{T}\pmb{Y}^{T}\pmb{Y\theta}_{i} - 2\pmb{\theta}_{i}^{T}\left(\pmb{Y}^{T}\left(\pmb{X\beta}_{i} + \frac{\rho}{2\sqrt{n}}(\pmb{P}_{i}-\pmb{B}_{i}) \right)\right)+C\\
       \text{s.t.} & \pmb{\theta}_{i}^{T}\pmb{Y}^{T}\pmb{Y}\pmb{1}=0.
    \end{array}
\end{equation}
This is a strongly convex quadratic program in $\pmb{\theta}_i$. Therefore, the unique solution of the Karush-Kuhn-Tucker conditions is the unique minimizer.
Enforcing the KKT conditions yields the following linear system for $(\pmb{\theta}_{i}, v_{i})$, where $v_i$ is the Lagrange multiplier corresponding to the equality constraint:
\begin{equation}\label{eq: theta update linear system}
   \begin{pmatrix} 2 \pmb{I} +\frac{\rho}{n}\pmb{Y}^{T}\pmb{Y} & \quad & -\pmb{Y}^{T}\pmb{Y}\pmb{1} \vspace{0.5cm} \\ \pmb{1}^{T}\pmb{Y}^{T}\pmb{Y} & \quad & 0 \end{pmatrix} 
\begin{pmatrix}
     \pmb{\theta}_{i} \vspace{0.5cm} \\  v_{i}
 \end{pmatrix}
 =
  \begin{pmatrix}
   2\pmb{Y}^{T}(\pmb{X\beta}_{i}+\frac{\rho}{2\sqrt{n}}(\pmb{P}_{i}-\pmb{B}_{i})) \vspace{0.5cm} \\  {0}
   \end{pmatrix}.
\end{equation}
Thus, we update $\pmb{\theta}_{i}$ as the solution of linear system~\eqref{eq: theta update linear system} using the current values of $\pmb{\beta}_{i},\pmb{P}_{i},\pmb{B}_{i}$
for each $i=1,2,....,q$. 
Note that the coefficient matrix in~\eqref{eq: theta update linear system} depends only on the parameter $\rho$ and the matrix $\pmb{Y}^T \pmb{Y}$; we can pre-compute a Cholesky factorization and use it to update each $\pmb{\theta}_i$ until $\rho$ is changed.
This update scheme is also amenable to parallelization as we can update all $\pmb{\theta}_i$ simultaneously.

%++++++++++++++++++++++++++++++++++++++++++++++++++++++++++++++++++
%++++++++++++++++++++++++++++++++++++++++++++++++++++++++++++++++++
% Updating Beta.
\subsection{Updating $\pmb{\beta}$}\label{sec: beta update}
%++++++++++++++++++++++++++++++++++++++++++++++++++++++++++++++++++
%++++++++++++++++++++++++++++++++++++++++++++++++++++++++++++++++++

We update 
$\pmb{\beta}^{\text{new}}$ as the solution of the following unconstrained optimization problem
 \begin{equation}\label{eq: objective function for beta}
       \underset{\pmb{\beta}}{\min}\sum_{i=1}^{q} \|\pmb{Y}\pmb{\theta}_{i}-\pmb{X}\pmb{\beta}_{i}\|^{2}+\gamma\|\pmb{\beta}_{i}\|^{2}+\lambda\|\pmb{\beta}_{i}\|_{1},
\end{equation} 
with $\pmb{\theta_i}$ fixed,
where $\pmb{\beta_i}$ denotes the $i$th column of $\pmb{\beta}$.
This objective function is separable in $\pmb{\beta_i}$. Each $\pmb{\beta_i}$ can be updated independently using the iterative methods outlined in Section~\ref{sec: BCD for SOS}. In the experimental analysis found later in this manuscript, we will focus on the use of the accelerated proximal gradient method proposed by Atkins~et~al.~\cite{atkins2023proximal}; however, any efficient iterative method for non-smooth convex minimization may be used.
Again, we can parallelize this update and compute all new $\pmb{\beta_i}$ simultaneously.

%++++++++++++++++++++++++++++++++++++++++++++++++++++++++++++++++++
%++++++++++++++++++++++++++++++++++++++++++++++++++++++++++++++++++
% Updating P.
\subsection{Updating $\pmb{P}$}\label{sec: P update}
%++++++++++++++++++++++++++++++++++++++++++++++++++++++++++++++++++
%++++++++++++++++++++++++++++++++++++++++++++++++++++++++++++++++++

Finally, we update $\P^{\text{new}}$ by minimizing the augmented Lagrangian of~\eqref{eq: DFSOS final splitting} with respect to $\P$ with $\T$ and $\b$ fixed.
The following theorem gives a closed-form solution for $\P^{\text{new}}$.

\begin{theorem}[{\cite[Theorem~2.1]{lai2014splitting}}]
\label{thm: closed form solution}
The constrained quadratic problem:
\begin{equation}\label{eqn: closed-form solution}
\pmb{Q}^{\ast} =  \underset{\pmb{Q}\in \mathbb{R}^{n\times m}}{\arg\min}~ \frac{1}{2}\|\pmb{Q}-\pmb{Y}\|_{F}^{2}, \quad \text{s.t.} \quad \pmb{Q}^{T}\pmb{Q} = \pmb{I}
\end{equation}
has an analytical solution
\begin{equation*} \pmb{Q}^{\ast} = \pmb{UI}_{n\times m}\pmb{V}^{T}, \end{equation*}
where $\pmb{U} \in \mathbb{R}^{n\times n}, \pmb{V} \in \mathbb{R}^{m\times m}$ are  orthogonal matrices and $\pmb{D} \in \mathbb{R}^{n\times m}$ is a diagonal matrix satisfying the singular value decomposition (SVD) $\pmb{Y} = \pmb{UDV}^{T}$. 

Moreover, if $\rank(\pmb{Y}) = m$, then 
\begin{equation*}
    \pmb{Q}^{\ast} = \pmb{Y}\Tilde{\pmb{V}}\Tilde{\pmb{D}}^{-1/2}\Tilde{\pmb{V}}^{T},
\end{equation*}
where $\pmb{V}\in \mathbb{R}^{m\times m}$ is an orthogonal matrix and $\Tilde{\pmb{D}}\in \mathbb{R}^{m\times m}$ is a diagonal matrix satisfying the SVD  $\pmb{Y}^{T}\pmb{Y} = \Tilde{\pmb{V}}\Tilde{\pmb{D}}\Tilde{\pmb{V}}^{T}$.
\end{theorem}

Note that
\begin{equation*}
    \P^{\text{new}} = \argmin_{\P\in \mathbb{R}^{k\times q}} \| \P - (\L\T + \B)\|^2_F \quad \text{s.t.} \quad \P^T \P = \I
\end{equation*}
for fixed iterates $(\T, \b, \B)$.
Theorem~\ref{thm: closed form solution} suggests an algorithm for updating $\P$:
compute the singular value decomposition        
\begin{equation*}
    \pmb{L}\T + \B = \pmb{U} \pmb{D} \pmb{V}^T ;
\end{equation*}
then let
\begin{equation} \label{eq: P update}
    \P^{\text{new}} = \pmb{U} \pmb{I}_{k\times q} \pmb{V}^T.
\end{equation}

%++++++++++++++++++++++++++++++++++++++++++++++++++++++++++++++++++
%++++++++++++++++++++++++++++++++++++++++++++++++++++++++++++++++++
%++++++++++++++++++++++++++++++++++++++++++++++++++++++++++++++++++
% Main Algorithm.
\subsection{A Heuristic for Deflation-Free Optimal Scoring}
\label{sec: DFSOS}
%++++++++++++++++++++++++++++++++++++++++++++++++++++++++++++++++++
%++++++++++++++++++++++++++++++++++++++++++++++++++++++++++++++++++
%++++++++++++++++++++++++++++++++++++++++++++++++++++++++++++++++++
Each iteration of our proposed algorithm DFSOS updates $(\pmb{\theta}, \pmb{\beta}, \pmb{P})$ using (\ref{eq: theta update linear system}), \eqref{eq: objective function for beta}, and \eqref{eq: P update}, followed by the approximate dual ascent step~\eqref{eq: beta update}. We summarize the steps of the DFSOS Algorithm in Alg.~\ref{DFSOS Algorithm}.
 
\begin{algorithm}[t]
	\caption{DFSOS Algorithm} 
     \label{DFSOS Algorithm}
	\begin{algorithmic}[1]
        \State Initialize: Given regularization parameter $\lambda >0$ and initial $\pmb{\theta}^{0}, \pmb{\beta}^{0}$, set $\pmb{B}^{0}=0$, $\pmb{P}^{0}=\pmb{LX}$, where $\pmb{L} = \frac{1}{\sqrt{n}}\pmb{Y}$.
\While {\text{``not converged"}}

    \State    
    ($\pmb{\theta}^{k}, \pmb{\beta}^{k})=\arg\min J(\pmb{\theta}^{k-1},\pmb{\beta}^{k-1})+\frac{\rho}{2}\|\frac{1}{\sqrt{n}}\pmb{Y}\pmb{\theta}^{k-1}-\pmb{P}^{k-1}+\pmb{B}^{k-1}\|_{F}^{2}$.

    \State Compute the SVD $\pmb{L}\pmb{\theta}^{k}+\pmb{B}^{k-1} = \pmb{UDV}^{T}$.

    \State $\pmb{P}^{k} =\pmb{U}\pmb{I}_{n\times m}\pmb{V}^{T}$.

    \State $\pmb{B}^{k}=\pmb{B}^{k-1}+\pmb{L}\pmb{\theta}^{k}-\pmb{P}^{k}$
     
     \State  $k \leftarrow k+1$

\EndWhile

	\end{algorithmic} 
\end{algorithm}

%++++++++++++++++++++++++++++++++++++++++++++++++++++++++++++++++++
%++++++++++++++++++++++++++++++++++++++++++++++++++++++++++++++++++
%++++++++++++++++++++++++++++++++++++++++++++++++++++++++++++++++++
% Convergence.
\subsection{Convergence of DFSOS}
\label{sec: convergence}
%++++++++++++++++++++++++++++++++++++++++++++++++++++++++++++++++++
%++++++++++++++++++++++++++++++++++++++++++++++++++++++++++++++++++
%++++++++++++++++++++++++++++++++++++++++++++++++++++++++++++++++++

Convergence of our algorithm follows directly from Corollary 2 given in \cite[Section 5]{wang2019global}. We have the following theorem:

\begin{theorem}\label{convergence for DFSOS}
Let $\{(\pmb{\theta}^{k},\pmb{\beta}^{k})\}_{k=1}^\infty$ be the sequence of points generated by Algorithm \ref{DFSOS Algorithm}.
Then DFSOS (Algorithm~\ref{DFSOS Algorithm}) is convergent for sufficiently large choice of $r$.
In this case, each set of iterates $\{(\pmb{\theta}^{k},\pmb{\beta}^{k})\}_{k=1}^\infty$ is bounded, has at least one limit point, and each limit point $(\pmb{\theta}^{\ast},\pmb{\beta}^{\ast})$ is a stationary point of the augmented Lagrangian $\mathcal{L}(\pmb{\theta}, \pmb{\beta}, \pmb{P})$, i.e., $\boldsymbol{0} \in \partial \mathcal{L}(\pmb{\theta}^{\ast}, \pmb{\beta}^{\ast}, \pmb{P}^{\ast})$.
\end{theorem}

\begin{proof}
    The objective function $J(\pmb{\theta}, \pmb{\beta})$ is strongly convex and differentiable. 
    For $l_{1}$-constrained optimization, Bregman iteration coincides with the general ADMM Algorithm. It is easy to see that both ADMM and Bregman Iteration are bounded and the augmented Lagrangian is lower bounded at a critical point \cite{wang2019global, lin2022alternating}.    
    Moreover, each set of the sequence $\{(\pmb{\theta}^{k},\pmb{\beta}^{k})\}_{k=1}^\infty$ generated by Algorithm \ref{DFSOS Algorithm} is a bounded set and $J(\pmb{\theta}, \pmb{\beta})$ is lower bounded on the feasible set. 

    This shows that Algorithm \ref{DFSOS Algorithm} and the problem (\ref{eq: DFSOS final splitting}) satisfy the hypothesis of Corollary 2 \cite{wang2019global}. This establishes the stated convergence of Algorithm \ref{DFSOS Algorithm}. 
\end{proof}

%++++++++++++++++++++++++++++++++++++++++++++++++++++++++++++++++++
%++++++++++++++++++++++++++++++++++++++++++++++++++++++++++++++++++
%++++++++++++++++++++++++++++++++++++++++++++++++++++++++++++++++++
% Implementation Details.
\subsection{Practical Implementation Details}
\label{sec: implementation}
%++++++++++++++++++++++++++++++++++++++++++++++++++++++++++++++++++
%++++++++++++++++++++++++++++++++++++++++++++++++++++++++++++++++++
%++++++++++++++++++++++++++++++++++++++++++++++++++++++++++++++++++

We conclude this section with a discussion of practical details guiding usage of our heuristic for deflation-free sparse optimal scoring.

%++++++++++++++++++++++++++++++++++++++++++++++++++++++++++++++++++
%++++++++++++++++++++++++++++++++++++++++++++++++++++++++++++++++++
% Useful range of regularization parameters.
\subsubsection{Range of Regularization Parameters}
\label{sec: params}
%++++++++++++++++++++++++++++++++++++++++++++++++++++++++++++++++++
%++++++++++++++++++++++++++++++++++++++++++++++++++++++++++++++++++
We want to choose $\lambda$ so we that are guaranteed a non-trivial solution in the $\pmb{\beta}$ subproblem. If we set $\lambda=0$, the $\pmb{\beta}$-subproblem (\ref{eq: objective function for beta}) becomes
\begin{equation}\label{eq: lambda range}
    \underset{\pmb{\beta}}{\min}\sum_{i=1}^{q} \|\pmb{Y}\pmb{\theta}_{i}-\pmb{X}\pmb{\beta}_{i}\|^{2}+\gamma\|\pmb{\beta}_{i}\|^{2},
\end{equation}
where $\pmb{\theta}$ is fixed. Expanding gives the equivalent problem
\begin{equation}\label{eq: lambda update original function}
    \underset{\pmb{\beta}}{\min}\sum_{i=1}^{q}\left( \pmb{\beta}_{i}^{T}\pmb{X}^{T}\pmb{X}\pmb{\beta}_{i}-2\pmb{\beta}_{i}^{T}\pmb{X}^{T}\pmb{Y}\pmb{\theta}_{i}\right)+\gamma\|\pmb{\beta}_{i}\|^{2}.
\end{equation}
Applying stationarity and taking partial derivatives with respect to $\pmb{\beta}$ shows that $\pmb{\beta}^{\ast}$ is a solution of the linear system
\begin{equation}
2\left( \pmb{X}^{T}\pmb{X}+\gamma\pmb{I}\right)\pmb{\beta}^{\ast}= 2\pmb{X}^{T}\pmb{Y}\pmb{\theta}
    \end{equation}
This linear system has unique solution
\begin{equation}\label{eq: optimal beta}
   \pmb{\beta}^{\ast} =  \left(\pmb{X}^{T}\pmb{X} +\gamma\pmb{I}\right)^{-1}\pmb{X}^{T}\pmb{Y}\pmb{\theta}_{i}. 
\end{equation}
Substituting $\pmb{\beta}^{\ast}$ into the original objective function (\ref{eq: lambda update original function}) gives
the minimum value
\begin{equation}
    \sum_{i=1}^{q}\left( \pmb{\beta}_{i}^{\ast T} \left( \pmb{X}^{T}\pmb{X}+\gamma\pmb{I}\right)\pmb{\beta}_{i}^{\ast}-2\pmb{\beta}_{i}^{\ast T}\pmb{X}^{T}\pmb{Y}\pmb{\theta}_{i}+\pmb{\theta}_{i}^{T}\pmb{Y}^{T}\pmb{Y}\pmb{\theta}_{i} + \lambda\|\pmb{\beta}_{i}\|_{1} \right),
\end{equation}
where we have $\pmb{\theta}_{i}^{T}\pmb{Y}^{T}\pmb{Y}\pmb{\theta}_{i}=n$ by feasibility of $\pmb{\theta}$. \\
Now, rearranging we have a nontrivial $\pmb{\beta}$ solution if $\pmb{\beta}^{\ast}$ has value at most $nq$. Therefore, we need
\begin{equation}
    \sum_{i=1}^{q}\left( \pmb{\beta}_{i}^{\ast T} \left( \pmb{X}^{T}\pmb{X}+\gamma\pmb{I}\right)\pmb{\beta}_{i}^{\ast}-2\pmb{\beta}_{i}^{\ast T}\pmb{X}^{T}\pmb{Y}\pmb{\theta}_{i}+ \lambda\|\pmb{\beta}_{i}\|_{1} \right) \leq 0.
\end{equation}
Thus, it is sufficient to choose
\begin{equation}\label{eq: lambda equation}
    \lambda \leq \frac{\sum_{i=1}^{q}\pmb{\beta}_{i}^{T}(\pmb{X}^{T}\pmb{X} + \gamma \pmb{I})\pmb{\beta}_{i}-2\pmb{\beta}_{i}^{T}\pmb{X}^{T}\pmb{Y\theta}_{i}}{\|\pmb{\beta}\|_{1}},
\end{equation}
to ensure there exists at least one solution $\pmb{\beta}^{\ast}$ with objective value strictly less than the value of the trivial solution $\pmb{\beta} = \pmb{0}$. 

A more flexible approach is to assign a different regularization term $\lambda_i$ for each column of $\pmb{\beta}$:
\begin{equation}
    \underset{\pmb{\beta}}{\min}\sum_{i=1}^{q} \left( 
    \|\pmb{Y}\pmb{\theta}_{i}-\pmb{X}\pmb{\beta}_{i}\|^{2}+\gamma\|\pmb{\beta}_{i}\|^{2}+\lambda_i\|\pmb{\beta}_{i}\|_{1}
    \right).
\end{equation}
In this case, we can ensure a non-trivial solution by choosing $\lambda_1, \lambda_2, \dots, \lambda_q$ to satisfy 
$$
\lambda_i \le \frac{\pmb{\beta}_{i}^{T}(\pmb{X}^{T}\pmb{X} + \gamma \pmb{I})\pmb{\beta}_{i}-2\pmb{\beta}_{i}^{T}\pmb{X}^{T}\pmb{Y\theta}_{i}}{\|\pmb{\beta}_i\|_{1}}
$$
for all $i=1,2,\dots, q$.

%++++++++++++++++++++++++++++++++++++++++++++++++++++++++++++++++++
%++++++++++++++++++++++++++++++++++++++++++++++++++++++++++++++++++
% Initialization.
\subsubsection{Initialization}
\label{sec: initialization}
%++++++++++++++++++++++++++++++++++++++++++++++++++++++++++++++++++
%++++++++++++++++++++++++++++++++++++++++++++++++++++++++++++++++++

%++++++++++++++++++++++++++++++++++++++++++++++++++++++++++++++++++
%++++++++++++++++++++++++++++++++++++++++++++++++++++++++++++++++++
% Initial Theta
%++++++++++++++++++++++++++++++++++++++++++++++++++++++++++++++++++
%++++++++++++++++++++++++++++++++++++++++++++++++++++++++++++++++++
We randomly initialize $\T$ and apply an orthogonalization process based on modification of the Gram-Schmidt process
to ensure that our initial iterate $\T$ satisfies $\T^T\Y^T \Y \T = n \I$.
Details of this process are given in~Algorithm~\ref{Pseudocode for initializing theta}.

\begin{algorithm} [t!] 
	\begin{algorithmic}[1]
        \State Input random $\pmb{\theta}\in\mathbb{R}^{k\times q}.$
		\For {$i =1,2,\ldots$}
			\State $\pmb{D} = \frac{1}{n}(\pmb{Y}^{T}\pmb{Y})$
			\State $\pmb t = \pmb u- \pmb Q_{j}(\pmb Q_{j}^{T}\pmb D \pmb u)$
            \State $a = \pmb t^{T} \pmb D  \pmb t$
			\State $\pmb{\theta}_{i} = \pmb t/\sqrt{a}$
		\EndFor
	\end{algorithmic}  
	\caption{Pseudocode for initializing $\pmb{\theta}$} 
 \label{Pseudocode for initializing theta}
\end{algorithm}

%++++++++++++++++++++++++++++++++++++++++++++++++++++++++++++++++++
%++++++++++++++++++++++++++++++++++++++++++++++++++++++++++++++++++
% Beta initialization.
%++++++++++++++++++++++++++++++++++++++++++++++++++++++++++++++++++
%++++++++++++++++++++++++++++++++++++++++++++++++++++++++++++++++++

On the other hand, we calculate 
initial $\pmb{\beta}$ from the initial orthogonal $\pmb{\theta}$ and the data matrix $\pmb{X}$ using the Sherman-Morrison-Woodbury process.

% SMW Lemma.
\begin{lemma}[Sherman-Morrison-Woodbury] 
If $\pmb{A} = \pmb{E} + \pmb{UV}$, then 
\begin{equation}\label{eq:Sherman formula}
   \pmb{A}^{-1} = (\pmb{E} + \pmb{UV})^{-1}
            = \pmb{E}^{-1} - \pmb{E}^{-1}\pmb{U}(\pmb{I} + \pmb{V}\pmb{E}^{-1}\pmb{U})^{-1}\pmb{V}\pmb{E}^{-1}.
\end{equation}
\end{lemma}

After temporarily setting the $\ell_1$-penalty parameter $\lambda=0$, \eqref{eq: optimal beta} is equivalent to
\begin{equation}\label{eq: initial beta}
    \underset{\pmb{\beta}}{\min}\sum_{i=1}^{q}\left( \pmb{\beta}_{i}^{T}\pmb{X}^{T}\pmb{X}\pmb{\beta}_{i}-2\pmb{\beta}_{i}^{T}\pmb{X}^{T}\pmb{Y}\pmb{\theta}_{i}^{\ast}\right)+\gamma\|\pmb{\beta}_{i}\|^{2}.
\end{equation}
Here $\pmb{\theta}_{i}^{\ast}$ is the orthogonal initial $\pmb{\theta}$. Applying stationarity and taking partial derivatives of  (\ref{eq: initial beta}) with respect to $\pmb{\beta}_{i}$ gives optimum $\pmb{\beta}^{\ast}$ such that
\begin{equation}\label{eq: initial beta equation}
 2\left( \pmb{X}^{T}\pmb{X}+\gamma\pmb{I}\right)\pmb{\beta}_{i}^{\ast}= 2\pmb{X}^{T}\pmb{Y}\pmb{\theta}_{i}^{\ast}.
\end{equation}
Applying the Sherman-Morrison-Woodbury formula (\ref{eq:Sherman formula}), shows that 
\begin{align*}
        \pmb{\beta}_{i}^{\ast} &=  \left(\pmb{X}^{T}\pmb{X} +\gamma\pmb{I}\right)^{-1}\pmb{X}^{T}\pmb{Y}\pmb{\theta}_{i}^{\ast}\\
        &= \left( \frac{1}{\gamma}\pmb{I} - \frac{1}{\gamma}^{2}\pmb{X}^{T}\left( \pmb{I} + \frac{1}{\gamma}\pmb{X} \pmb{X}^{T}\right)^{-1} \pmb{X} \right)\pmb{X}^{T}\pmb{Y}\pmb{\theta}_{i}^{\ast}\\
        &= \frac{1}{\gamma}\left( \pmb{M} - \frac{1}{\gamma}\pmb{X}^{T}\left( \pmb{I} + \frac{1}{\gamma} \pmb{X} \pmb{X}^{T}\right)^{-1}(\pmb{XM})\right),
\end{align*}
where $\pmb{M} = \pmb{X}^{T}\pmb{Y}\pmb{\theta}_{i}^{\ast}$. This process is summarized in Alg.~\ref{Pseudocode for initial beta}.

\begin{algorithm}[t]    
	\begin{algorithmic}[1]
        \State \textbf{Input} $\pmb{X}, \pmb{\theta}^{\ast}, \pmb{Y}, \gamma$
        \State Calculate $\pmb{B} = \pmb{X}^{T}\pmb{Y}\pmb{\theta}_{i}^{\ast}$
        \State Compute $\pmb{XB}$
        \State Compute $\pmb{I} + \frac{1}{\gamma}\pmb{XX}^{T} $
        \State Using Cholesky factorization solve $\left(\pmb{I} + \frac{1}{\gamma}\pmb{XX}^{T} \right)\pmb{V} = \pmb{XB}$
        \begin{enumerate}
            \item $\pmb{R}^{T}\pmb{R} = \pmb{I} + \frac{1}{\gamma}\pmb{XX}^{T}$
            \item Solve $\pmb{R}^{T}\pmb{RV} = \pmb{XB}$ for $\pmb{RV}$ $\Rightarrow \text{temp} = \pmb{RV} = \pmb{R}^{T}\backslash (\pmb{XB})$
            \item Solve $\pmb{V} = \pmb{R}\backslash \text{temp}$
        \end{enumerate}
        \State Compute $\pmb{\beta} = \frac{1}{\gamma}\left( \pmb{B} - \frac{1}{\gamma}\pmb{X}^{T}\pmb{V}\right)$.
		
	\end{algorithmic} 
	\caption{Pseudocode for $\pmb{\beta}$} 
        \label{Pseudocode for initial beta}
\end{algorithm}

%++++++++++++++++++++++++++++++++++++++++++++++++++++++++++++++++++
%++++++++++++++++++++++++++++++++++++++++++++++++++++++++++++++++++
% Augmented Lagrangian Parameter
\subsubsection{An Alternative Factorization}
\label{sec:L-form}

The matrix $\Y^T \Y$ is a diagonal matrix with $i$th diagonal entry equal to the number of observations $n_i$ belonging to class $i$. We can perform the matrix-splitting given in \eqref{eq: DFSOS final splitting} using the matrix 
\begin{equation} \label{eq: new L}
\L = \frac{1}{\sqrt{n}}
\begin{pmatrix}
    \sqrt{n_1} \\ & \sqrt{n_2} \\ & & \ddots \\ &&& \sqrt{n_K}.
\end{pmatrix}
\end{equation}
This factorization can be more computationally efficient than that outlined above as the matrix $\L$ is $k\times k$, rather than $n\times n$. Moreover, products involving $\L$ only require scaling according to the diagonal entries of $\L$.

%++++++++++++++++++++++++++++++++++++++++++++++++++++++++++++++++++
%++++++++++++++++++++++++++++++++++++++++++++++++++++++++++++++++++
% Augmented Lagrangian Parameter
\subsubsection{The Augmented Lagrangian Parameter}
\label{sec:rho-update}

We dynamically update the augmented Lagrangian parameter $r$ each iteration.
If the residual $\| \P^k - \L \T^k\|_F^2$ following iteration $k$ does not decrease a sufficient amount, then we increase $r$ to accelerate convergence of the Bregman iteration scheme as suggested in \cite[Section 3.1]{burer2003nonlinear}. Details for this update process are given in Alg.~\ref{Pseudocode for parameter r}.

\begin{algorithm}[t]
	\begin{algorithmic}[1]
        \State \textbf{Input} $\eta < 1, \sigma > 1$, initial $\rho$, initial orthogonal primal variable $ \pmb{P^0} = \L \T^0$.
        \State Initialize $v_{0} = 2 \|\pmb{P^0}\|_{F}^{2} = 2q.$
        \State At iteration $k$:  Compute $v=\|\pmb{P} - \L \pmb{\theta}\|_{F}^{2}$
        \If{$v< \eta v_{k}$}
            \State 
            $v_{k+1} = v$
        \Else
            \State
             $\rho_{k+1} = \sigma \rho_{k}$
             \State
            $v_{k+1} = v_{k}   $            
        \EndIf
	\end{algorithmic}  
	\caption{Update of penalty parameter $\rho$} 
    \label{Pseudocode for parameter r}
\end{algorithm}
%++++++++++++++++++++++++++++++++++++++++++++++++++++++++++++++++++
%++++++++++++++++++++++++++++++++++++++++++++++++++++++++++++++++++

%++++++++++++++++++++++++++++++++++++++++++++++++++++++++++++++++++
%++++++++++++++++++++++++++++++++++++++++++++++++++++++++++++++++++
%++++++++++++++++++++++++++++++++++++++++++++++++++++++++++++++++++
%++++++++++++++++++++++++++++++++++++++++++++++++++++++++++++++++++
\section{Empirical Analysis}
\label{sec:empirical_analysis}
%++++++++++++++++++++++++++++++++++++++++++++++++++++++++++++++++++
%++++++++++++++++++++++++++++++++++++++++++++++++++++++++++++++++++
%++++++++++++++++++++++++++++++++++++++++++++++++++++++++++++++++++
%++++++++++++++++++++++++++++++++++++++++++++++++++++++++++++++++++

We next compare the empirical performance of our deflation-free optimal scoring methods with existing methods for SOS, as well as standard classification methods, for a variety of data.

%++++++++++++++++++++++++++++++++++++++++++++++++++++++++++++++++++
%++++++++++++++++++++++++++++++++++++++++++++++++++++++++++++++++++
\subsection{Gaussian Data}
\label{sec:gaussians}
%++++++++++++++++++++++++++++++++++++++++++++++++++++++++++++++++++
%++++++++++++++++++++++++++++++++++++++++++++++++++++++++++++++++++

%++++++++++++++++++++++++++++++++++++++++++++++++++++++++++++++++++
% Description of Gaussian distribution.
%++++++++++++++++++++++++++++++++++++++++++++++++++++++++++++++++++

We first investigate the performance of our algorithm for classifying Gaussian data.
We performed the following experiment $10$ times for each $(K,r)$-pair for $K \in \{3,6\}$ and $r \in \{0.1, 0.5, 0.9\}$.
In each experiment, we sample $p$-dimensional vectors from $K$ multivariate normal distributions for $p=1000$. For each dataset, we apply sparse optimal scoring to perform dimension reduction, and validate this transformation by performing nearest centroid classification following projection onto the span of the learned discriminant vectors.

\subsubsection{Distribution of Gaussian Data}
\label{sec:gaussian-framework}

We generate training data corresponding to the $i$th class, $i = 1, 2, ..., K$, by sampling $n_{\text{train}}=100$ observations from the multivariate normal distribution with mean $\pmb\mu_{i} \in \mathbb{R}^{p}$ satisfying    
\begin{equation*} 
[\pmb\mu_{i}]_{j} = 
\begin{cases}
			& 0.7 \quad \text{if} ~100(i-1) < j\leq 100i \\
			& 0 \quad \text{otherwise}
\end{cases}
\end{equation*}
for all $j = 1, 2,..., K$ and
covariance matrix 
$\pmb\Sigma\in \mathbb{R}^{p\times p}$ with all features correlated defined by 
\begin{equation*}
    \Sigma_{ij} = 
    \begin{cases}
        r, & \text{if } i \neq j, \\ 
        1, & \text{if } i = j
    \end{cases}
\end{equation*}
for fixed scalar $r \in [0, 1)$.
In each experiment, we sampled $n_{\text{test}} = 1000$ testing observations from each class.  For each combination of $(K, r)$, we produced 10 training-validation pairs of data.

% All computations were executed on a standard node of the High Performance Computing (HPC) system located at the Alabama Supercomputer Center with 120 GB of RAM using Matlab R2022a. 

%++++++++++++++++++++++++++++++++++++++++++++++++++++++++++++++++++
% Comparators
%++++++++++++++++++++++++++++++++++++++++++++++++++++++++++++++++++
\subsubsection{Methods for Comparison and Choice of Parameters}

We apply the DFSOS method proposed in Section~\ref{sec: dfsos} to fit discriminant vectors using each training data set. We consider two variants based on solving~\eqref{eq: DFSOS final splitting} with 
$\L = \frac{1}{\sqrt{n}} \Y$ (\emph{DFSOS-1}) and $\L = \frac{1}{\sqrt{n}} \left(\Y^T \Y \right)^{1/2}$ (\emph{DFSOS-2}). We use the accelerated proximal gradient heuristic proposed in~\cite{atkins2023proximal} to solve~\eqref{eq: beta update} and update $\pmb{\beta}$ in each version of DFSOS. We compare our deflation-free approach with the accelerated proximal gradient (APG) and alternating direction method of multipliers (ADMM) heuristics for sparse optimal scoring proposed in~\cite{atkins2023proximal}. 
We use 5-fold cross validation to choose $\ell_1$-penalty parameter $\lambda$ in DFSOS, APG, and ADMM from 7~possible choices on the exponential grid
\begin{equation*}
    \left\{2^{-3} , 2^{-2}, 2^{-1}, 1, 2, 2^2, 2^3  \right\} \times \lambda_{\max},
\end{equation*}
where $\lambda_{\max}$ is chosen as in Section~\ref{sec: params}.
We also compare our classification results using \emph{support vector machine (SVM)} and \emph{$k$-nearest neighbour (KNN)} classifiers for $k=1,5,10$. 

%++++++++++++++++++++++++++++++++++++++++++++++++++++++++++++++++++
% Termination criteria and other parameters.
%++++++++++++++++++++++++++++++++++++++++++++++++++++++++++++++++++
% DFSOS
DFSOS requires choice of stopping criteria for the inner optimization of $\pmb{\beta}$ and the overall optimization iteration. We terminate the inner optimization if a $10^{-5}$ suboptimal solution is found or a maximum of $50$ iterations are performed; we stop the $\boldsymbol{\beta}$-update proximal gradient step after a $10^{-4}$ suboptimal solution is found or after $100$ iterations.
We terminate the outer loop after a maximum of 500 iterations or a $10^{-4}$ suboptimal solution is found. 
We use scaling parameters $\eta =0.25$, $\sigma =2$, and initial $\rho=5$ to dynamically update the augmented Lagrangian parameter in DFSOS as described in Section~\ref{sec:rho-update}.
We choose initial $\pmb\beta$ and $\pmb\theta$ as in Section~\ref{sec: initialization}.

% APG/ADMM
The ADMM and APG heuristics require choice of the regularization parameters $\boldsymbol{\Omega}, \gamma$, and $\lambda$. In each experiment, we set $\pmb{\Omega} = \pmb{I}$ and $\gamma = 10^{-1}$. We choose $\lambda$ by $5$-fold cross validation from an exponential grid equivalent with that used for DFSOS.
These heuristics require stopping criteria for inner optimization of $\pmb{\beta}$ and the outer optimization loop.
We terminate the inner loop after a maximum of 100 iterations or if a $10^{-4}$ suboptimal solution is found. We terminate the outer loop if a $10^{-4}$ suboptimal solution is found or a maximum of 500 iterations have been performed.
The augmented Lagrangian parameter for ADMM was set at $\mu =2$ for all experiments.

% Table summarizing parameters.

\subsubsection{Computational Environments}

All experiments were performed with Matlab R2023a~\cite{MATLAB} using a single serial node of the IRIDIS 5 High Performance Computing Facility at the University of Southampton. Matlab implementations of the DFSOS and ASDA methods are shared at Dr.~Ames's Github repository\footnote{\url{https://github.com/bpames/Deflation-Free-SOS}}. We used default settings for the support vector machine and nearest neighbours functions from Matlab's Statistics and Machine Learning Toolbox \cite{SML-toolbox, matlab-sml-toolbox-doc}.

%++++++++++++++++++++++++++++++++++++++++++++++++++++++++++++++++++
%++++++++++++++++++++++++++++++++++++++++++++++++++++++++++++++++++
% Results
\subsection{Discussion of Experiments involving Gaussian Data}
\label{sec:results-Gaussians}
%++++++++++++++++++++++++++++++++++++++++++++++++++++++++++++++++++
%++++++++++++++++++++++++++++++++++++++++++++++++++++++++++++++++++

Summary statistics for our experiments involving Gaussian data can be found in Table~\ref{tab:synth-1}.

%++++++++++++++++++++++++++++++++++++++++++++++++
%++++++++++++++++++++++++++++++++++++++++++++++++
%++++++++++++++++++++++++++++++++++++++++++++++++
% SYNTH TABLE
%++++++++++++++++++++++++++++++++++++++++++++++++
%++++++++++++++++++++++++++++++++++++++++++++++++
%++++++++++++++++++++++++++++++++++++++++++++++++
\begin{table}[t!]
\adjustbox{max width=\textwidth}{%
\begin{tabular}{llllllllll}
\toprule
  Data & Measure &      DFSOS-1 &      DFSOS-2 &          APG &         ADMM &         SVM &       1-NN &       5-NN &      10-NN \\
\midrule
$K=3$ &     Accuracy &    1.0 (0.0) &    0.999 (0.0) &    0.999 (0.0) &      1.0 (0.0) &     1.0 (0.0) &  0.92 (0.0) &   0.978 (0.0) & 0.988 (0.0) \\ $r=0.1$ 
       &   Run-time (s) &   4.325 (1.107) &  3.948 (0.512) &  2.462 (0.127) &   2.626 (0.25) & 0.174 (0.014) & 0.021 (0.0) & 0.029 (0.001) & 0.022 (0.0) \\
       &  Cardinality &    0.119 (0.0) &    0.113 (0.0) &    0.118 (0.0) &  0.166 (0.023) &            -- &          -- &            -- &          --  \\\midrule
$K=3$ &     Accuracy &      1.0 (0.0) &      1.0 (0.0) &    0.999 (0.0) &    0.998 (0.0) &     1.0 (0.0) &  0.98 (0.0) &   0.981 (0.0) & 0.973 (0.0) \\ $r=0.5$ 
       &   Run-time (s) &  5.084 (1.373) &  4.783 (0.635) &   2.87 (0.068) &  3.099 (0.265) &  0.17 (0.012) & 0.021 (0.0) & 0.027 (0.001) & 0.021 (0.0) \\
       &  Cardinality &    0.084 (0.0) &    0.085 (0.0) &    0.097 (0.0) &    0.093 (0.0) &            -- &          -- &            -- &          -- \\\midrule
$K=3$ &     Accuracy &      1.0 (0.0) &      1.0 (0.0) &  0.588 (0.002) &  0.614 (0.007) &     1.0 (0.0) & 0.991 (0.0) &   0.976 (0.0) & 0.962 (0.0) \\ $r=0.9$
       &   Run-time (s) &  5.022 (1.783) &  4.695 (0.656) &  2.446 (0.059) &  2.554 (0.234) & 0.167 (0.012) & 0.022 (0.0) & 0.028 (0.001) & 0.021 (0.0) \\
       &  Cardinality &  0.097 (0.001) &  0.113 (0.001) &    0.012 (0.0) &    0.012 (0.0) &            -- &          -- &            -- &          -- \\ \midrule\midrule
$K=6$ &     Accuracy &    0.996 (0.0) &    0.994 (0.0) &    0.992 (0.0) &    0.993 (0.0) &     1.0 (0.0) & 0.855 (0.0) &   0.949 (0.0) & 0.978 (0.0) \\  $r=0.1$
       &   Run-time (s) & 12.095 (0.561) & 11.283 (0.482) & 12.761 (0.177) & 12.737 (0.545) &  0.66 (0.015) & 0.036 (0.0) & 0.045 (0.001) & 0.034 (0.0) \\
       &  Cardinality &    0.168 (0.0) &     0.17 (0.0) &    0.141 (0.0) &    0.144 (0.0) &            -- &          -- &            -- &          --  \\ \midrule
$K=6$ &     Accuracy &    0.999 (0.0) &    0.999 (0.0) &  0.942 (0.001) &  0.951 (0.001) &     1.0 (0.0) & 0.969 (0.0) &   0.977 (0.0) & 0.972 (0.0) \\ $r=0.5$ 
       &   Run-time (s) &  15.44 (0.876) & 14.614 (0.536) & 16.253 (0.276) & 16.179 (1.285) & 0.655 (0.014) & 0.038 (0.0) & 0.046 (0.001) & 0.035 (0.0) \\
       &  Cardinality &  0.191 (0.001) &  0.203 (0.001) &    0.069 (0.0) &    0.069 (0.0) &            -- &          -- &            -- &          -- \\ \midrule
$K=6$ &     Accuracy &      1.0 (0.0) &      1.0 (0.0) &  0.942 (0.003) &  0.943 (0.003) &     1.0 (0.0) &  0.99 (0.0) &   0.975 (0.0) & 0.966 (0.0) \\ $r=0.9$
       &   Run-time (s) & 17.329 (0.992) & 16.439 (0.668) & 14.416 (0.401) &  14.24 (0.935) &  0.64 (0.014) & 0.036 (0.0) & 0.044 (0.001) & 0.032 (0.0) \\
       &  Cardinality &  0.123 (0.002) &  0.124 (0.002) &    0.094 (0.0) &    0.095 (0.0) &            -- &          -- &            -- &          -- \\
\bottomrule
\end{tabular}
}% close adjustbox.
\caption{Out-of-sample prediction accuracy, run-time in seconds, and fraction of non-zero features for discriminant vectors and classifiers for a selection of Gaussian data constructed as described in Section~\ref{sec:gaussian-framework}. Data is reported in the format \texttt{mean (variance)} over $10$ experiments for each $(K,r)$-pair.}
\label{tab:synth-1}
\end{table}

%++++++++++++++++++++++++++++++++++++++++++++++++
% Discussion / Illustration.
%++++++++++++++++++++++++++++++++++++++++++++++++
\subsubsection{Comparison of Classifiers}

Our proposed deflation-free methods consistently provide more accurate classifiers than the deflationary 
approaches based on the ASDA algorithm, with comparable run-times.
Run-times for all four methods were within a few seconds of each other in all experiments,
while discriminant vectors computed using DFSOS tended to be slightly more dense than those given by ASDA.
Moreover, both DFSOS methods matched or exceeded the accuracy of the SVM and KNN classifiers in all experiments. We must note that SVM and KNN classifiers required much less time to fit, as they do not require any training of regularization parameters.

This phenomenon was starkest when the data was highly correlated, e.g., when $r=0.9$. For example, when $r=0.9$ and $K=3$, the ASDA classifiers both struggled to identify meaningful discriminant directions. In each case, the ASDA method converged to a pair of discriminant vectors with a modest number of relatively small non-zero entries. 
To illustrate, we consider a random set of Gaussian data constructed as described in Section~\ref{sec:gaussians} with $K=3$ and $r=0.9$.
The discriminant vectors $\pmb{\beta}$ given by each method for these data are visualised in Fig.~\ref{fig:DV-9}.
We can see that the discriminant vectors found using ASDA have much smaller magnitude compared to those found using DFSOS. Moreover, the support of the pairs of discriminant vectors found by ASDA differ from those found by DFSOS, which have many more nonzero entries; we should also note that while the discriminant vectors found by DFSOS-1 and DFSOS-2 appear to differ, this is due to rotational symmetry in the paired scoring vectors and the dimension reduction and classifiers for the two DFSOS approaches behave very similarly.

This shrinkage of ASDA discriminant vectors to $0$ leads to significant degradation of classification performance. We visualise decision boundaries of the nearest centroid classifier following projection onto the span of discriminant vectors in Fig.~\ref{fig:DV-9-boundaries}. We can clearly observe that there is minimal separation between classes in the space spanned by the discriminant vectors found by ASDA (see Fig.~\ref{fig:DV-9-APG-boundaries} and Fig.~\ref{fig:DV-9-ADMM-boundaries}).
This is largely avoided when we use DFSOS. Here, observations in each class are well-separated in the discriminant vector space.

We compared this behaviour with that for a set of Gaussian data constructed as described in Section~\ref{sec:gaussians} with $K=3$ and $r=0.5$.
Visualisations of discriminant vectors and classification boundaries are given in Fig.~\ref{fig:DV-5} and Fig.~\ref{fig:DV-5-boundaries} respectively. We can see that all found methods give roughly the same discriminant vectors, decision boundaries, and predictions, up to rotational symmetry. However, even in this case, we see that the test data is better separated in the span of the discriminant vectors given by DFSOS.

\begin{figure}
    \centering
    \begin{subfigure}[t]{0.49\textwidth}
        \centering
        \includegraphics[width=0.99\linewidth]{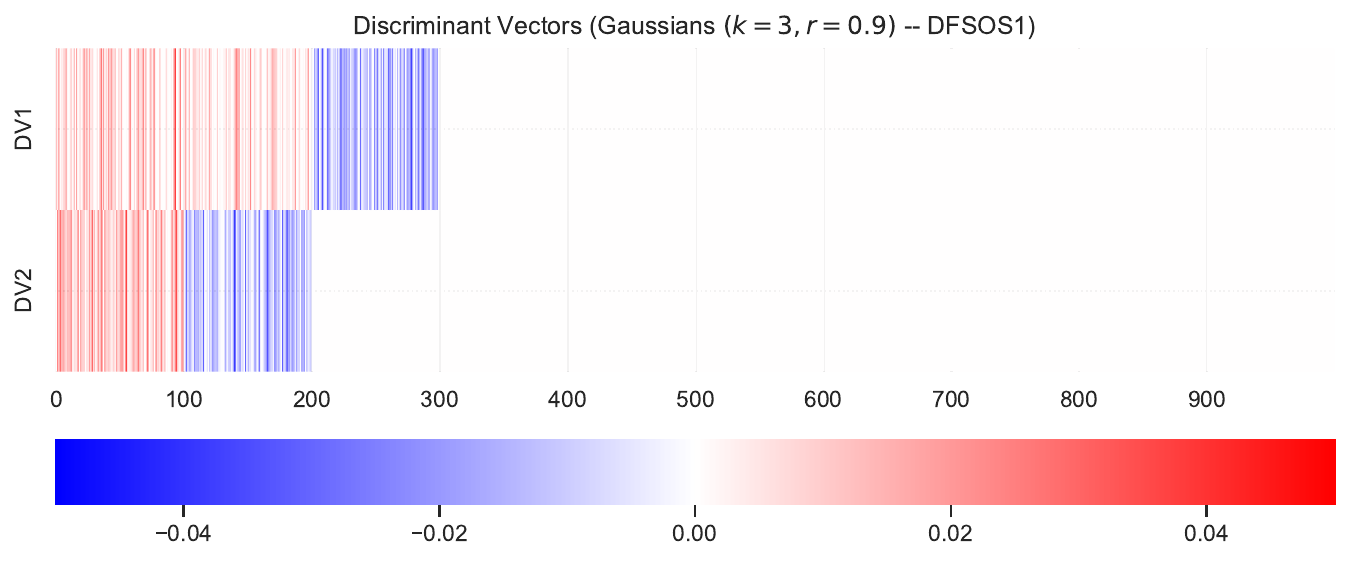}
        \caption{DFSOS-1}
        \label{fig:DV-9-DF1}
    \end{subfigure}
    \hfill
    \begin{subfigure}[t]{0.49\textwidth}
        \centering
        \includegraphics[width=0.99\linewidth]{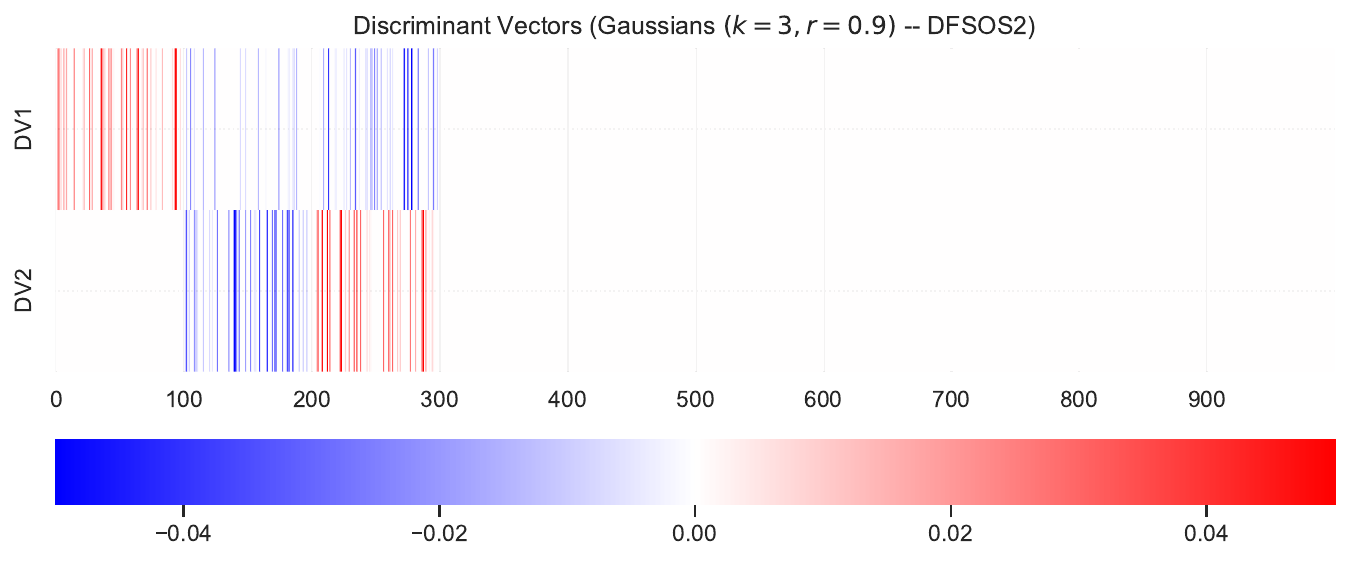}
        \caption{DFSOS-2}
        \label{fig:DV-9-DF2}
    \end{subfigure}

    \begin{subfigure}[t]{0.49\textwidth}
        \centering
        \includegraphics[width=0.99\linewidth]{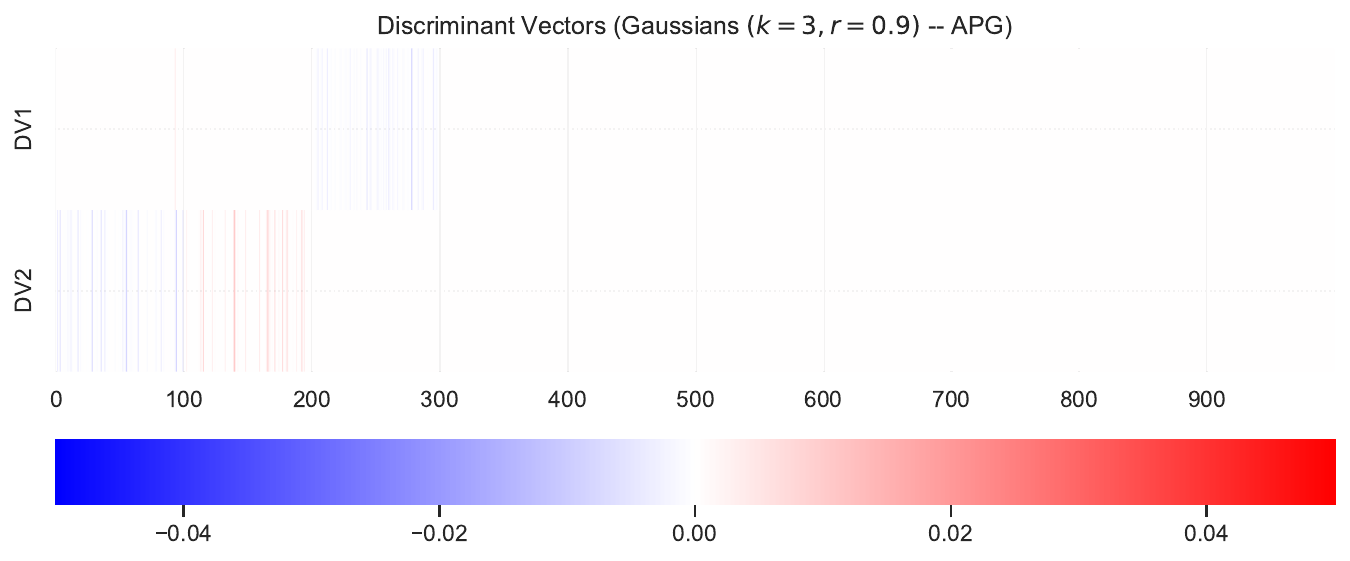}
        \caption{ASDA-APG}
        \label{fig:DV-9-APG}
    \end{subfigure}
    \hfill
    \begin{subfigure}[t]{0.49\textwidth}
        \centering
        \includegraphics[width=0.99\linewidth]{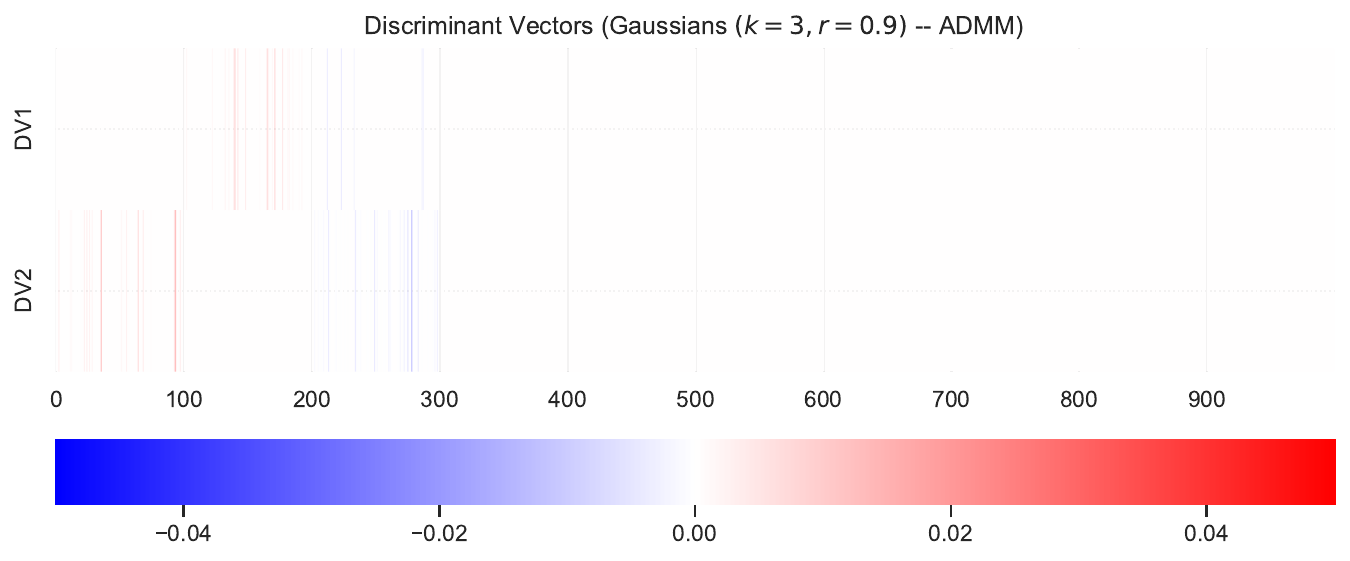}
        \caption{ASDA-ADMM}
        \label{fig:DV-9-ADMM}
    \end{subfigure}
    \caption{Visualisation of discriminant vectors returned by each of DFSOS-1, DFSOS-2, and ASDA with APG and ADMM  subproblem solvers. Note that the nonzero entries of the discriminant vectors found by ASDA have much small magnitude than that of the discriminant vectors found by each variant of DFSOS; this is indicated by very pale colours in Fig.~\ref{fig:DV-9-APG} and Fig.~\ref{fig:DV-9-ADMM}.}
    \label{fig:DV-9}
\end{figure}

%++++++++++++++++++++++++++++++++++++++++++++++++++++++++++++++++++++++++++
%++++++++++++++++++++++++++++++++++++++++++++++++++++++++++++++++++++++++++
% DV - r = 0.5
%++++++++++++++++++++++++++++++++++++++++++++++++++++++++++++++++++++++++++
%++++++++++++++++++++++++++++++++++++++++++++++++++++++++++++++++++++++++++

\begin{figure}
    \centering
    \begin{subfigure}[t]{0.49\textwidth}
        \centering
        \includegraphics[width=0.99\linewidth]{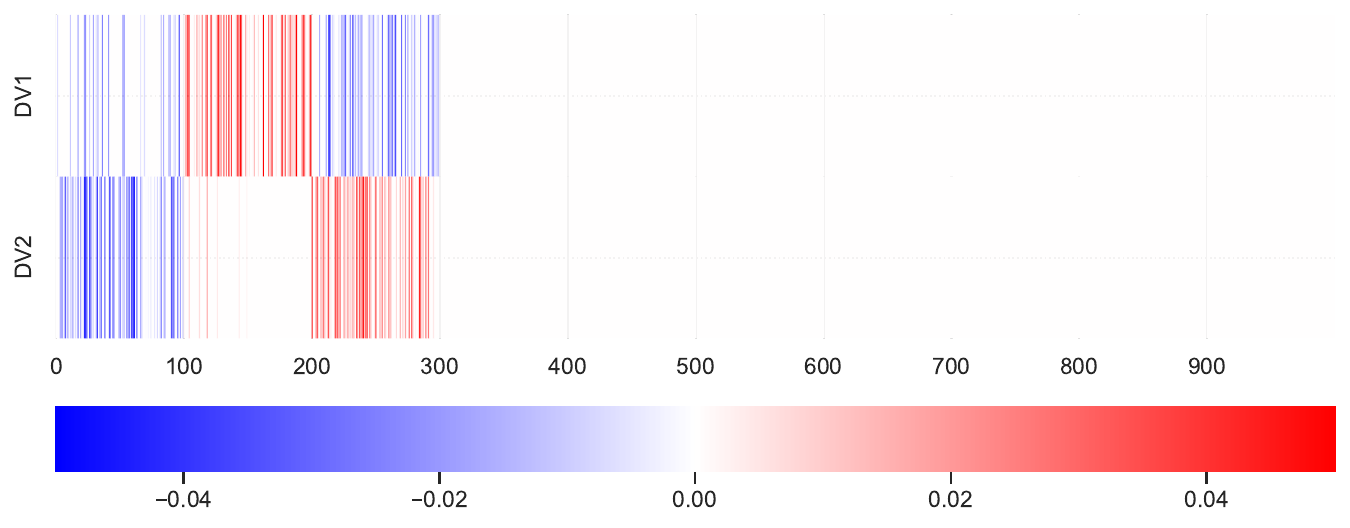}
        \caption{DFSOS-1}
        \label{fig:DV-5-DF1}
    \end{subfigure}
    \hfill
    \begin{subfigure}[t]{0.49\textwidth}
        \centering
        \includegraphics[width=0.99\linewidth]{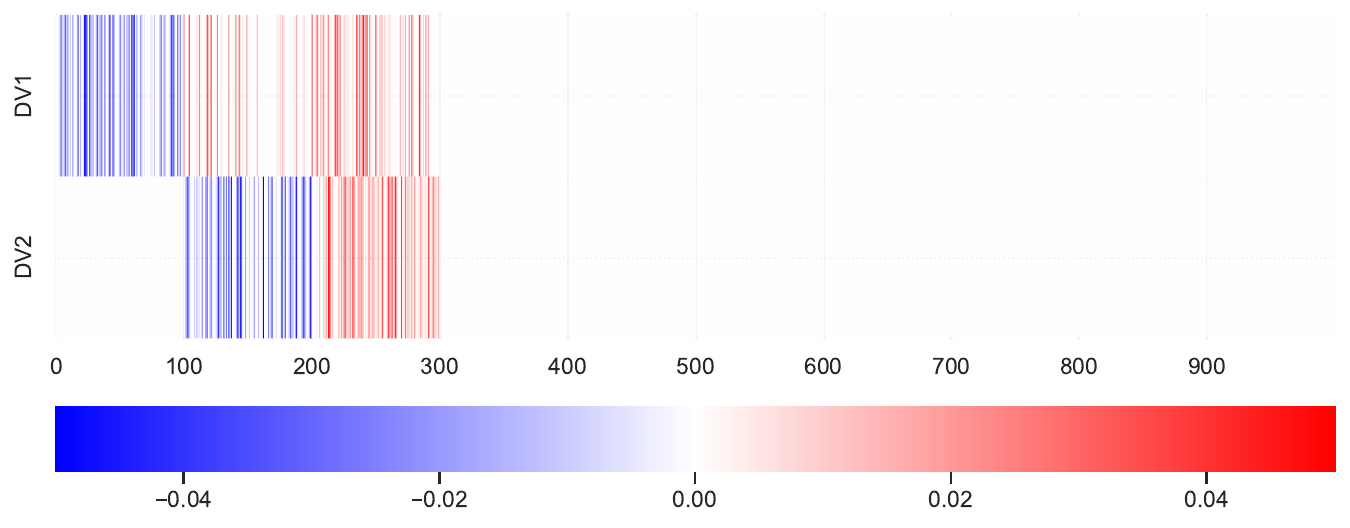}
        \caption{DFSOS-2}
        \label{fig:DV-5-DF2}
    \end{subfigure}

    \begin{subfigure}[t]{0.49\textwidth}
        \centering
        \includegraphics[width=0.99\linewidth]{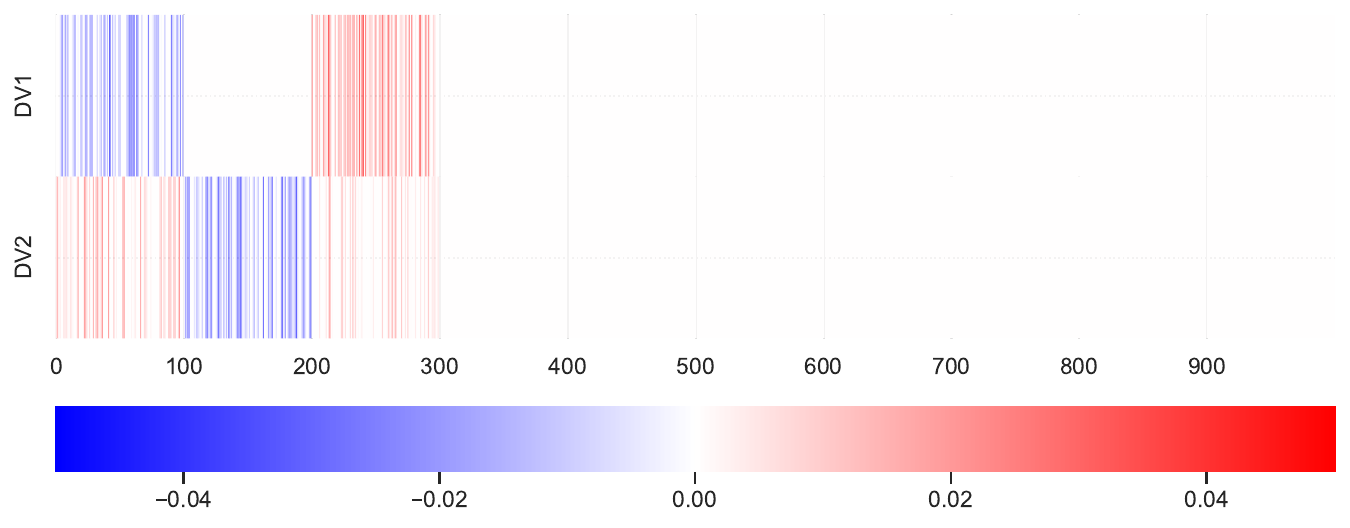}
        \caption{ASDA-APG}
        \label{fig:DV-5-APG}
    \end{subfigure}
    \hfill
    \begin{subfigure}[t]{0.49\textwidth}
        \centering
        \includegraphics[width=0.99\linewidth]{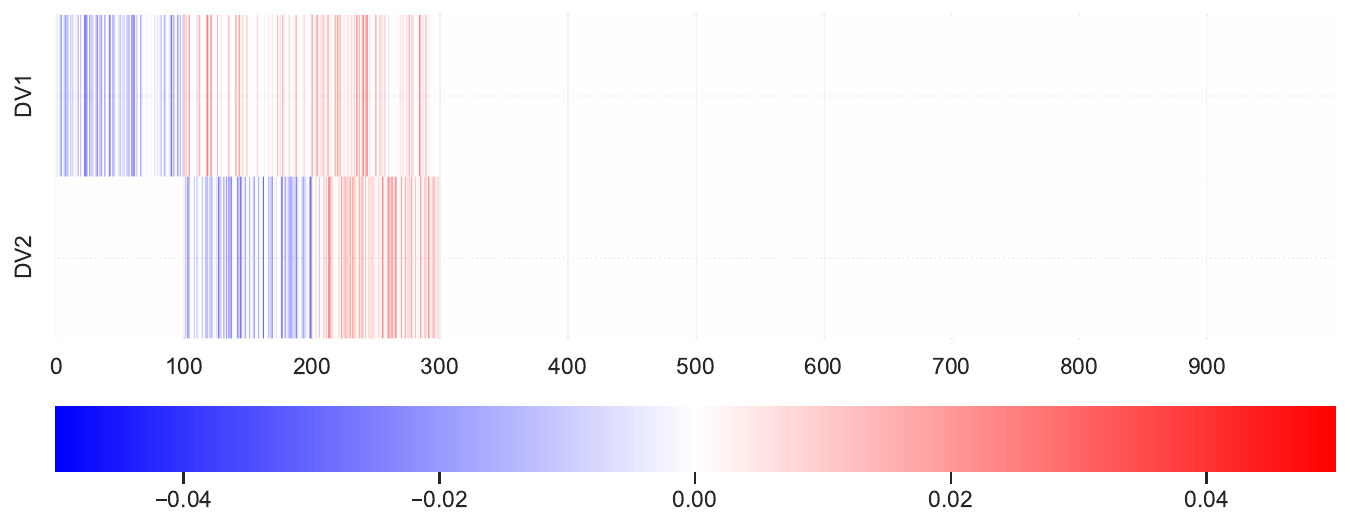}
        \caption{ASDA-ADMM}
        \label{fig:DV-5-ADMM}
    \end{subfigure}
    \caption{Visualisation of discriminant vectors returned by each of DFSOS-1, DFSOS-2, and ASDA with APG and ADMM  subproblem solvers. Note that the nonzero entries of the discriminant vectors found by ASDA, unlike in Fig.~\ref{fig:DV-9}, have similar magnitude to the discriminant vectors found by each variant of DFSOS.}
    \label{fig:DV-5}
\end{figure}

%++++++++++++++++++++++++++++++++++++++++++++++++++++++++++++++++++++++++++
%++++++++++++++++++++++++++++++++++++++++++++++++++++++++++++++++++++++++++
% Boundaries - r = 0.9
%++++++++++++++++++++++++++++++++++++++++++++++++++++++++++++++++++++++++++
%++++++++++++++++++++++++++++++++++++++++++++++++++++++++++++++++++++++++++

\begin{figure}
    \centering
    \begin{subfigure}[t]{0.49\textwidth}
        \centering
        \includegraphics[width=0.99\linewidth]{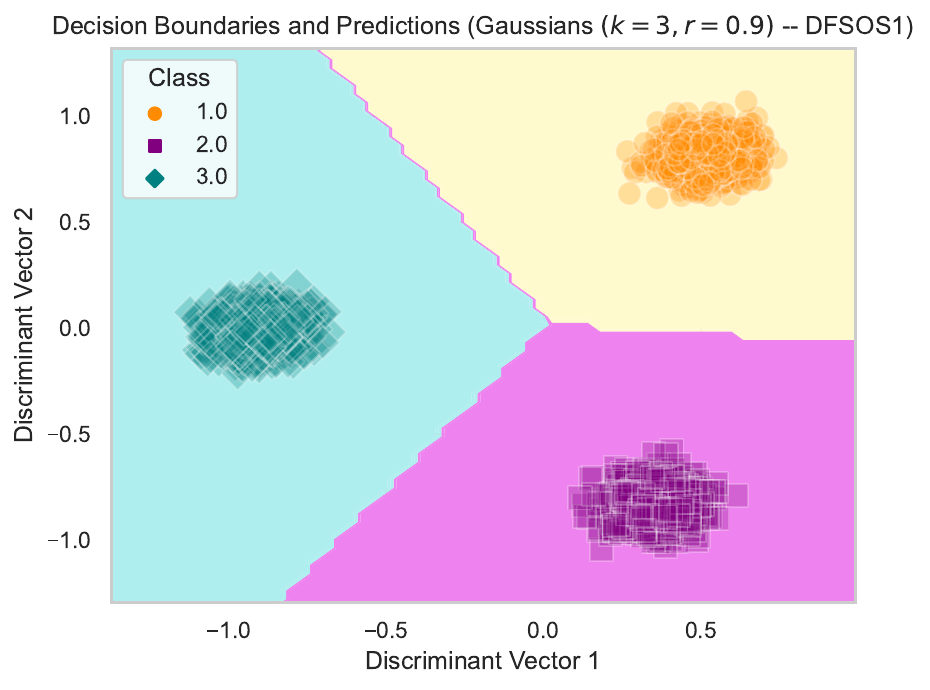}
        \caption{DFSOS-1}
        \label{fig:DV-9-DF1-boundaries}
    \end{subfigure}
    \hfill
    \begin{subfigure}[t]{0.49\textwidth}
        \centering
        \includegraphics[width=0.99\linewidth]{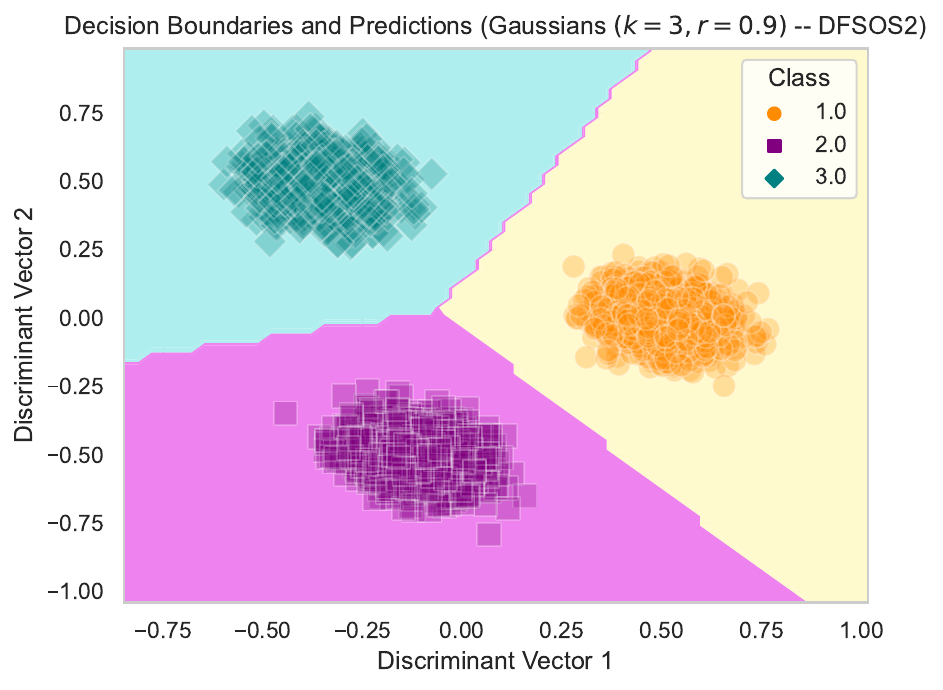}
        \caption{DFSOS-2}
        \label{fig:DV-9-DF2-boundaries}
    \end{subfigure}

    \begin{subfigure}[t]{0.49\textwidth}
        \centering
        \includegraphics[width=0.99\linewidth]{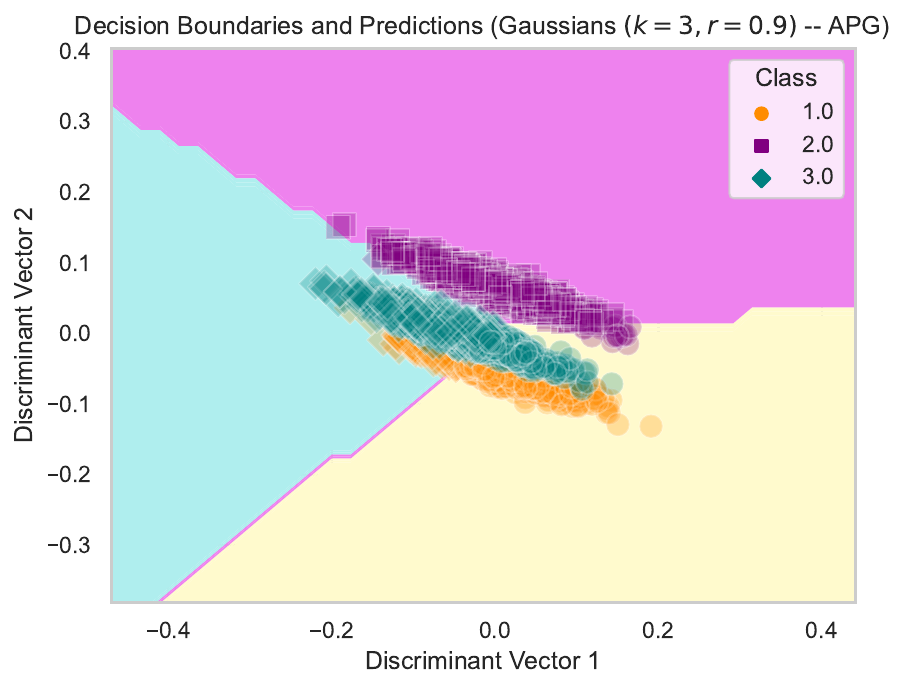}
        \caption{ASDA-APG}
        \label{fig:DV-9-APG-boundaries}
    \end{subfigure}
    \hfill
    \begin{subfigure}[t]{0.49\textwidth}
        \centering
        \includegraphics[width=0.99\linewidth]{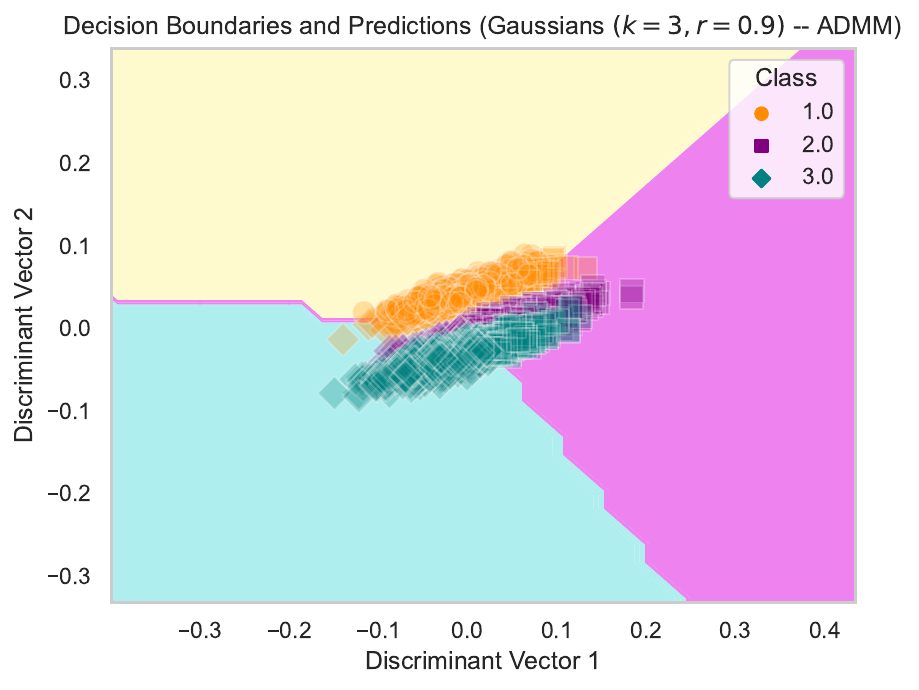}
        \caption{ASDA-ADMM}
        \label{fig:DV-9-ADMM-boundaries}
    \end{subfigure}
    \caption{Nearest centroid classification decision boundaries and predictions for testing data  in the span of each discriminant vector pair for $K=3$, $r=0.9$. Actual class label is indicated by colour and predicted class label is indicated by shape. Note that test data is not separable when using the discriminant vectors given by ASDA.}
    \label{fig:DV-9-boundaries}
\end{figure}

%++++++++++++++++++++++++++++++++++++++++++++++++++++++++++++++++++++++++++
%++++++++++++++++++++++++++++++++++++++++++++++++++++++++++++++++++++++++++
% Boundaries - r = 0.5
%++++++++++++++++++++++++++++++++++++++++++++++++++++++++++++++++++++++++++
%++++++++++++++++++++++++++++++++++++++++++++++++++++++++++++++++++++++++++

\begin{figure}
    \centering
    \begin{subfigure}[t]{0.49\textwidth}
        \centering
        \includegraphics[width=0.99\linewidth]{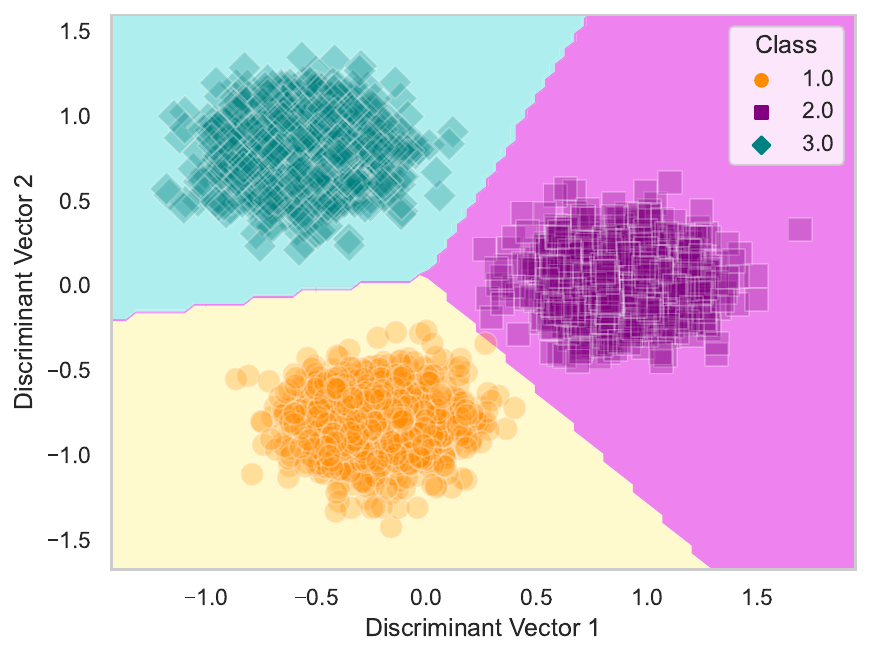}
        \caption{DFSOS-1}
        \label{fig:DV-5-DF1-boundaries}
    \end{subfigure}
    \hfill
    \begin{subfigure}[t]{0.49\textwidth}
        \centering
        \includegraphics[width=0.99\linewidth]{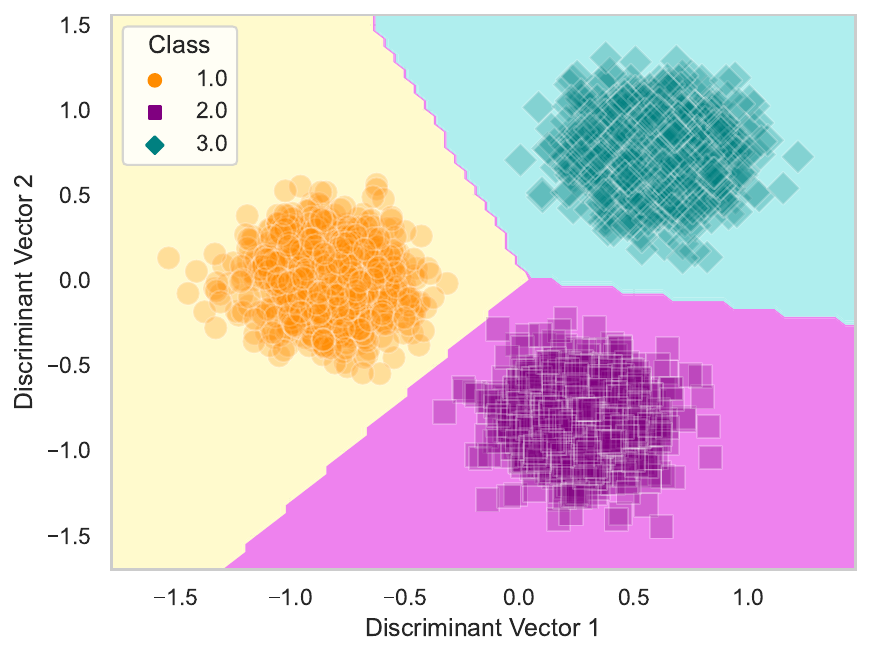}
        \caption{DFSOS-2}
        \label{fig:DV-5-DF2-boundaries}
    \end{subfigure}

    \begin{subfigure}[t]{0.49\textwidth}
        \centering
        \includegraphics[width=0.99\linewidth]{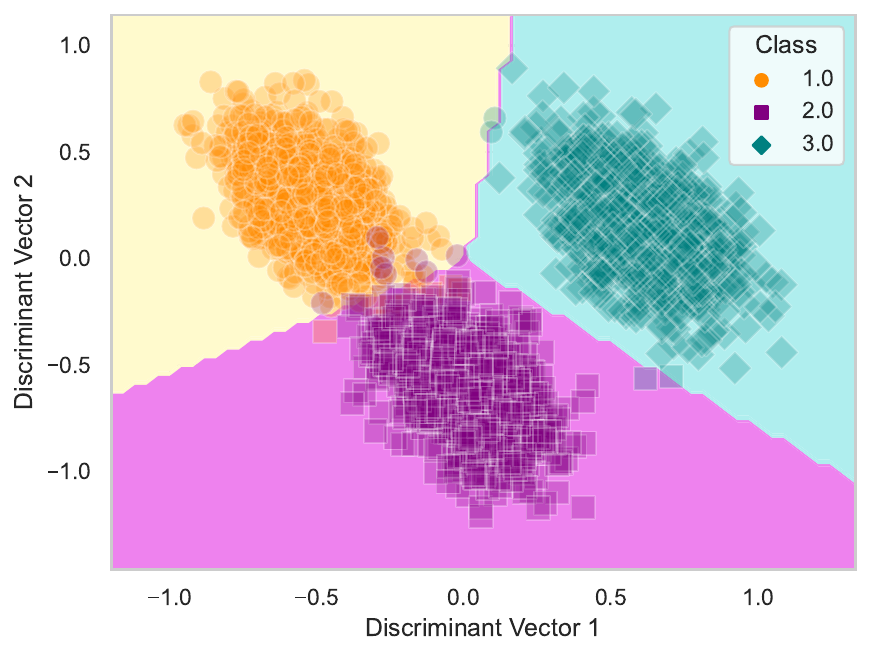}
        \caption{ASDA-APG}
        \label{fig:DV-5-APG-boundaries}
    \end{subfigure}
    \hfill
    \begin{subfigure}[t]{0.49\textwidth}
        \centering
        \includegraphics[width=0.99\linewidth]{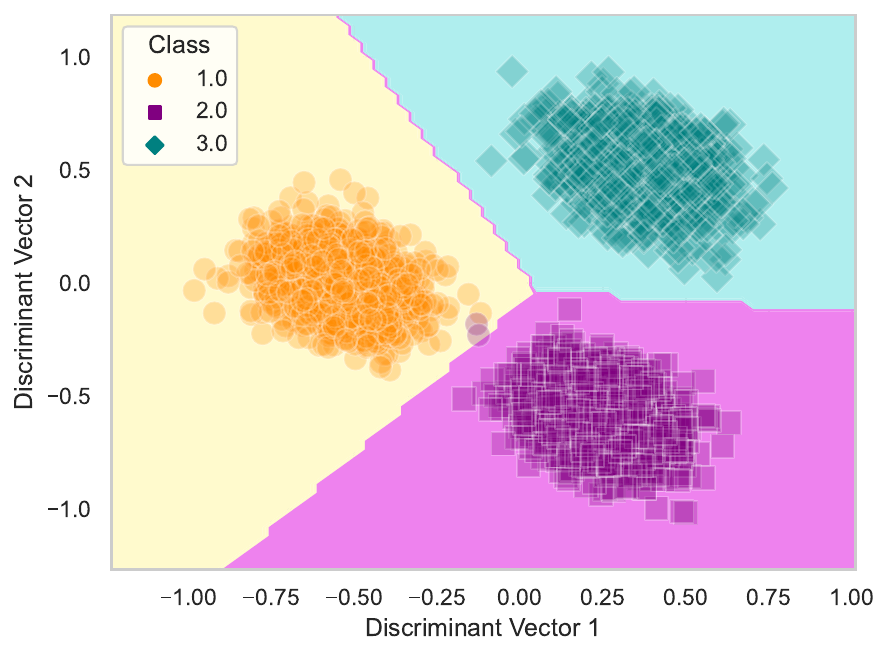}
        \caption{ASDA-ADMM}
        \label{fig:DV-5-ADMM-boundaries}
    \end{subfigure}
    \caption{Nearest centroid classification decision boundaries and predictions for testing data in the span of each discriminant vector pair  for $K=3$, $r=0.5$. Actual class label is indicated by colour and predicted class label is indicated by shape. Note that test data is more separable when using the discriminant vectors given by DFSOS.}
    \label{fig:DV-5-boundaries}
\end{figure}

%++++++++++++++++++++++++++++++++++++++++++++++++++++++++++++++++++
%++++++++++++++++++++++++++++++++++++++++++++++++++++++++++++++++++
% Real-Data
\subsubsection{Hypothesis Tests for Gaussian Data}
\label{sec:synth-hypothesis-tests}
%++++++++++++++++++++++++++++++++++++++++++++++++++++++++++++++++++
%++++++++++++++++++++++++++++++++++++++++++++++++++++++++++++++++++

For each $(K,r)$-pair, we performed one-sided Wilcoxon tests to compare pairs of classification methods in terms of prediction error, run-times, and cardinality.
For measure $f$, we test the null hypothesis $H_0: f(i) = f(j)$ against the one-sided alternative $H_a: f(i) < f(j)$ for each pair of methods $i$ and $j$.
The number of times we reject $H_0$ for each $(i,j)$-pair using an $\alpha = 0.05$ level of significance for each measure is given in Fig.~\ref{fig:synth-hypothesis-tests}. We observe the following:
\begin{itemize}
    \item Classifiers calculated using the two variants of DFSOS were significantly more accurate than all methods except for SVM in most trials. DFSOS classifiers were more accurate than all KNN classifiers for all $(K,r)$; DFSOS-1 was significantly more accurate than both ASDA methods for all but one $(K,r)$-pair, and DFSOS-2 was significantly more accurate than both ASDA methods for all but two $(K,r)$-pairs.
    We should note that we observe similar phenomena using other accuracy metrics such as recall, precision, and $f$-score.
    \item DFSOS tended to find denser discriminant vectors than ASDA; ASDA had significantly sparse discriminant vectors for four $(K,r)$ pairs.
    \item DFSOS required, on average, more time to converge than the ASDA methods, albeit by a small amount of actual time (1-2 seconds in most trials). ASDA had significantly less run-time for four $(K,r)$-pairs.
    On the other hand, DFSOS-2 required less computation on average than DFSOS-1 for all $(K,r)$-pairs.
\end{itemize}
From these hypothesis tests, we can conclude that DFSOS provides significantly better classification accuracy than ASDA, at a modest cost of increased run-time and cardinality of solutions. Moreover, this classification performance matches the state of the art for linearly separable data given by support vector machines, with increased interpretability due to the sparsity of discriminant vectors.

%++++++++++++++++++++++++++++++++++++++++++++++++++++++++++++++++++++++++++
%++++++++++++++++++++++++++++++++++++++++++++++++++++++++++++++++++++++++++
% Hypothesis Tests -- Synthetic Data
%++++++++++++++++++++++++++++++++++++++++++++++++++++++++++++++++++++++++++
%++++++++++++++++++++++++++++++++++++++++++++++++++++++++++++++++++++++++++

\begin{figure}
    \centering

    \begin{subfigure}[t]{0.3\textwidth}
        \centering
        \includegraphics[width=0.99\linewidth]{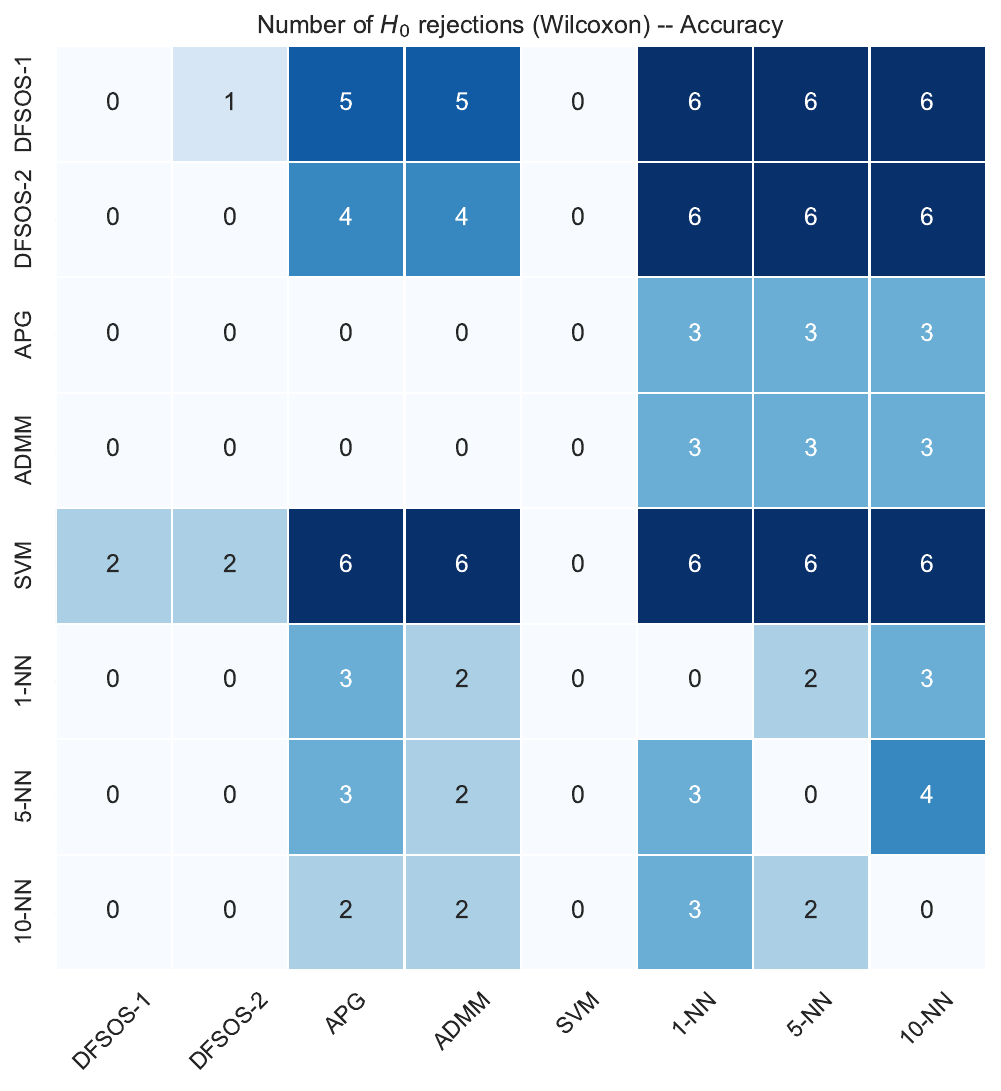}
        \caption{Classification Accuracy}
        \label{fig:DV-5-DF2-boundaries}
    \end{subfigure}
        \hfill
    \begin{subfigure}[t]{0.3\textwidth}
        \centering
        \includegraphics[width=0.99\linewidth]{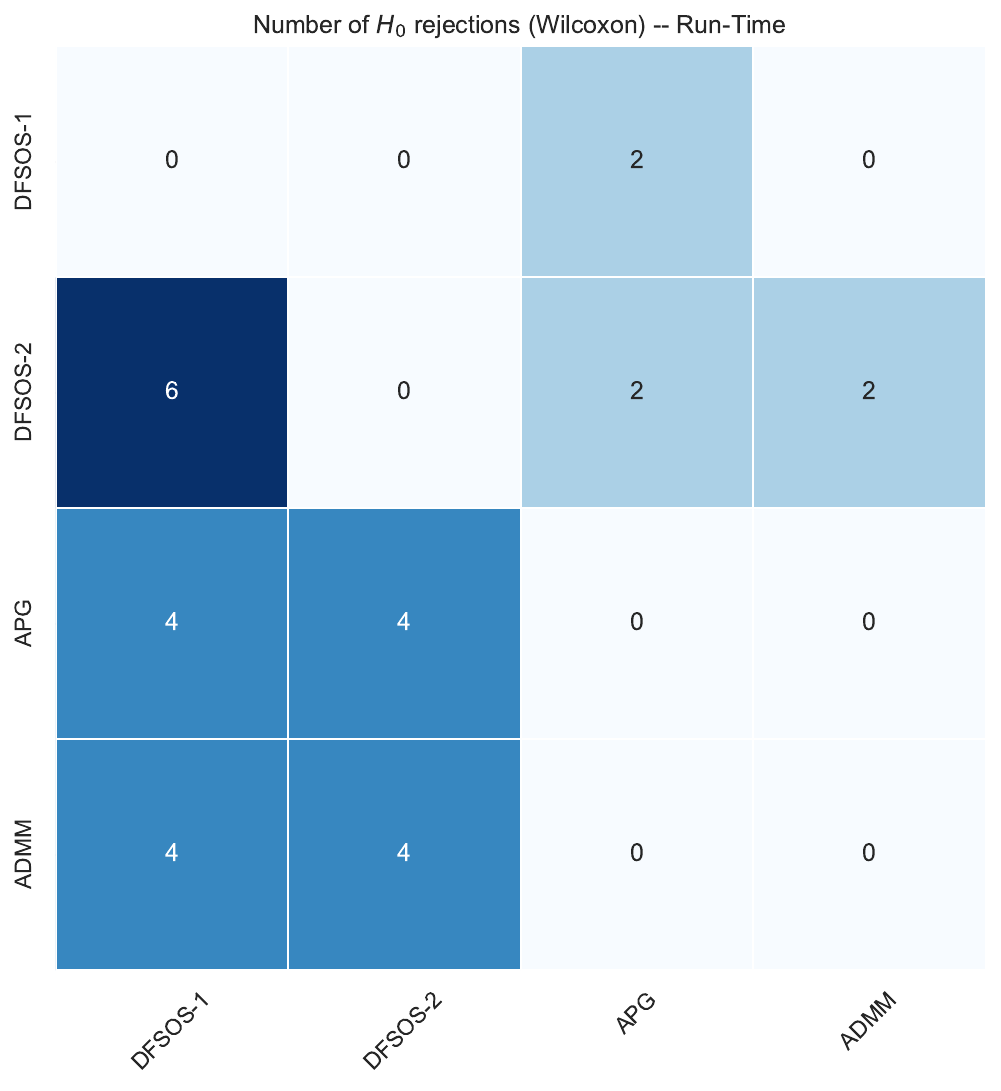}
        \caption{Run-times}
        \label{fig:synth-test-times}
    \end{subfigure}
    \hfill
    \begin{subfigure}[t]{0.3\textwidth}
        \centering
        \includegraphics[width=0.99\linewidth]{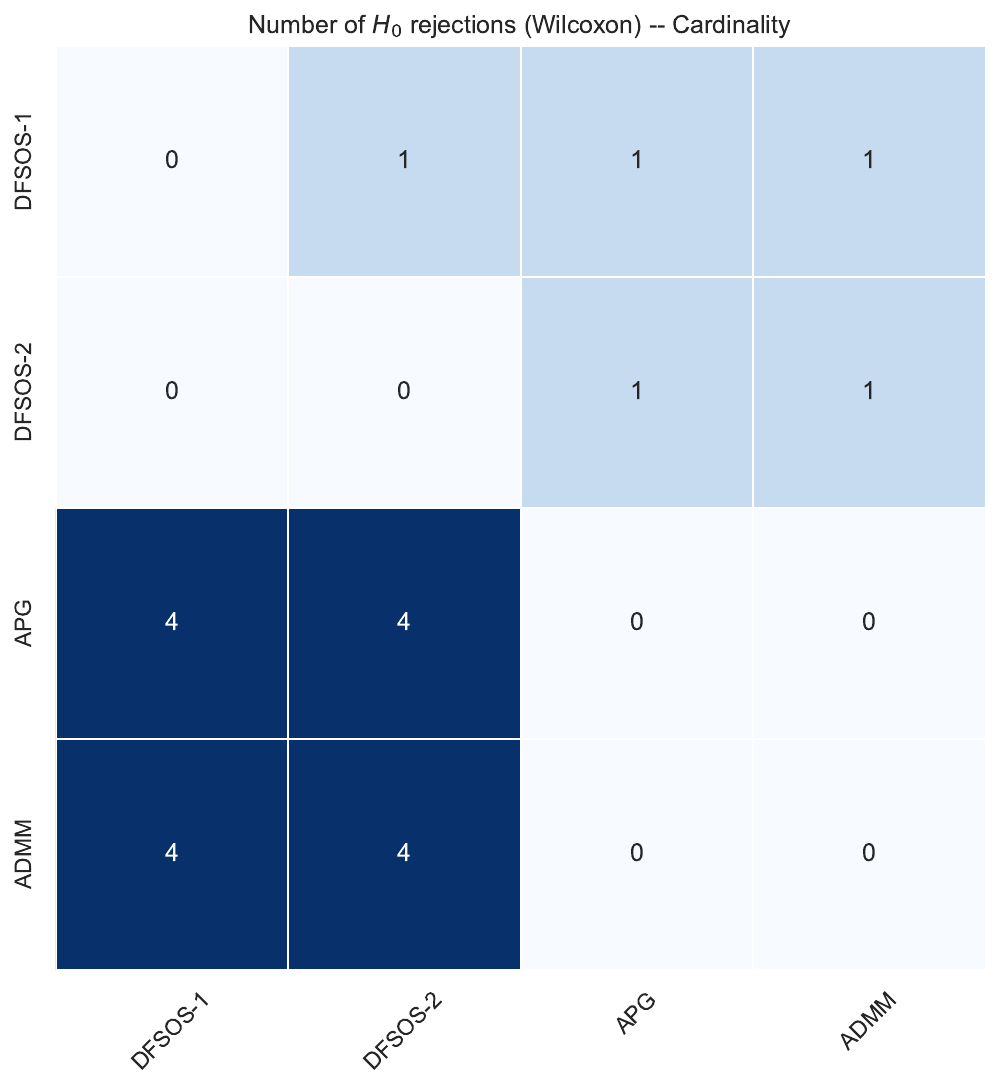}
        \caption{Cardinality}
        \label{fig:synth-test-card}
    \end{subfigure}
    \caption{Number of $(K,r)$-pairs for which we observe statistically significant improvement in classification accuracy, computational efficiency, and cardinality of discriminants when using row method compared to column method. That is, the $(i,j)$-entry is the number of $(K,r)$-pair where we reject the null hypothesis $H_0:$ there is no difference between method $i$ and method $j$ using the Wilcoxon test with one-sided alternative hypothesis $H_a:$ method $i$ is better than method $j$.}
    \label{fig:synth-hypothesis-tests}
\end{figure}

%++++++++++++++++++++++++++++++++++++++++++++++++++++++++++++++++++
%++++++++++++++++++++++++++++++++++++++++++++++++++++++++++++++++++
% Prediction Analysis
\subsubsection{Consistency of Predictions}
\label{sec:real-data}
%++++++++++++++++++++++++++++++++++++++++++++++++++++++++++++++++++
%++++++++++++++++++++++++++++++++++++++++++++++++++++++++++++++++++

Although we have observed an improvement in classification performance using DFSOS (and SVM) over ASDA in these experiments, we should caution that the actual classification performance of all of the observed methods is remarkably similar.
To illustrate this phenomenon, we calculate the \emph{cosine similarity} for the out-of-sample predictions made by each classifier. The cosine similarity between the vector of predictions for each pair of classifiers, averaged over the 10 problem instances sampled for each $(K,r)$-pair, is given in Fig.~\ref{fig:synth-cosine}.
The diagonal entries of Fig.~\ref{fig:synth-cosine} indicate the average cosine similarity between all pairs of prediction vectors made by each method across all $10$ Gaussian observations.

The cosine similarities for the class membership predictions summarized in Fig.~\ref{fig:synth-cosine} are almost all nearly equal to $1$, even when the Wilcoxon test suggests a statistically significant difference in classification accuracy. 
This suggests that when there is a difference in classification prediction, it is due to a disagreement in a small number of assigned class labels, with all other predictions agreeing for the two classifiers. This can be observed in Fig.~\ref{fig:DV-5-boundaries}. In this setting, we have $(K,r)=(3,0.5)$ and our hypothesis tests suggest that DFSOS classifiers provided a statistically significant improvement in classification performance over ASDA classifiers. However, the decision boundaries for the classifiers are similar and the predicted class labels agree for all four DFSOS and ASDA classifiers except for a relatively small number of misclassified observations near the decision boundary for the two ASDA classifiers.

\begin{figure}
    \centering
    \begin{subfigure}[t]{0.32\textwidth}
        \centering
        \includegraphics[width=0.99\linewidth]{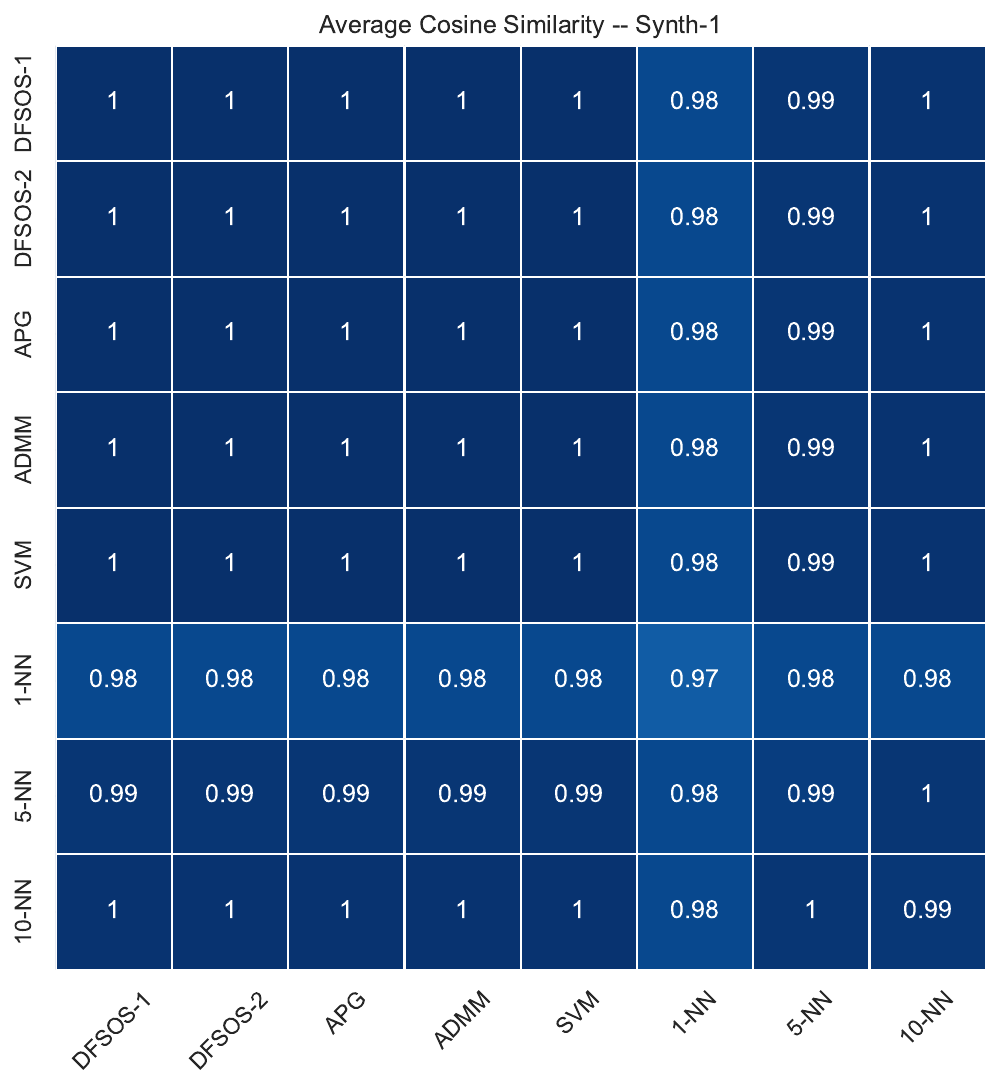}
        \caption{$(K,r) = (3, 0.1)$}
        \label{fig:3-1-J}
    \end{subfigure}
        \hfill
    \begin{subfigure}[t]{0.32\textwidth}
        \centering
        \includegraphics[width=0.99\linewidth]{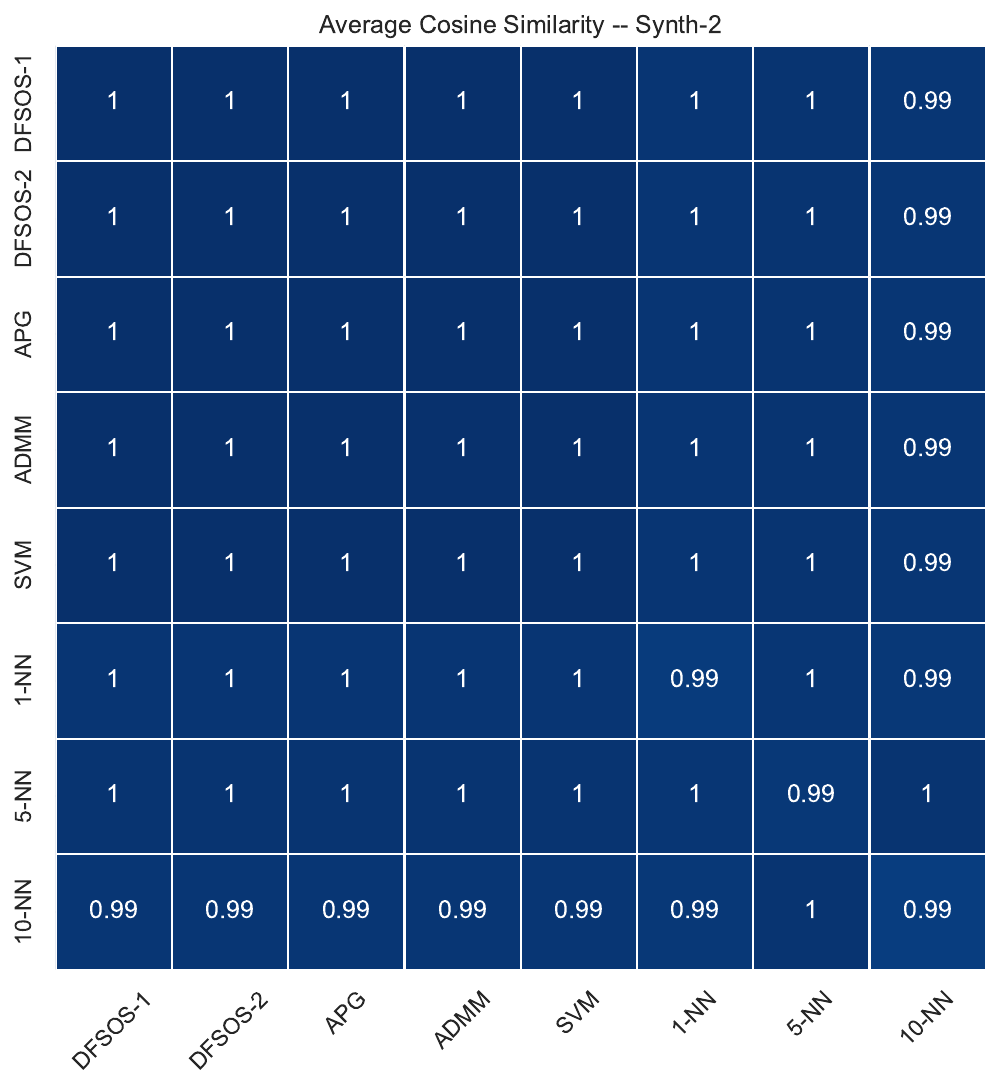}
        \caption{$(K,r) = (3, 0.5)$}
        \label{fig:3-2-J}
    \end{subfigure}
    \hfill
    \begin{subfigure}[t]{0.32\textwidth}
        \centering
        \includegraphics[width=0.99\linewidth]{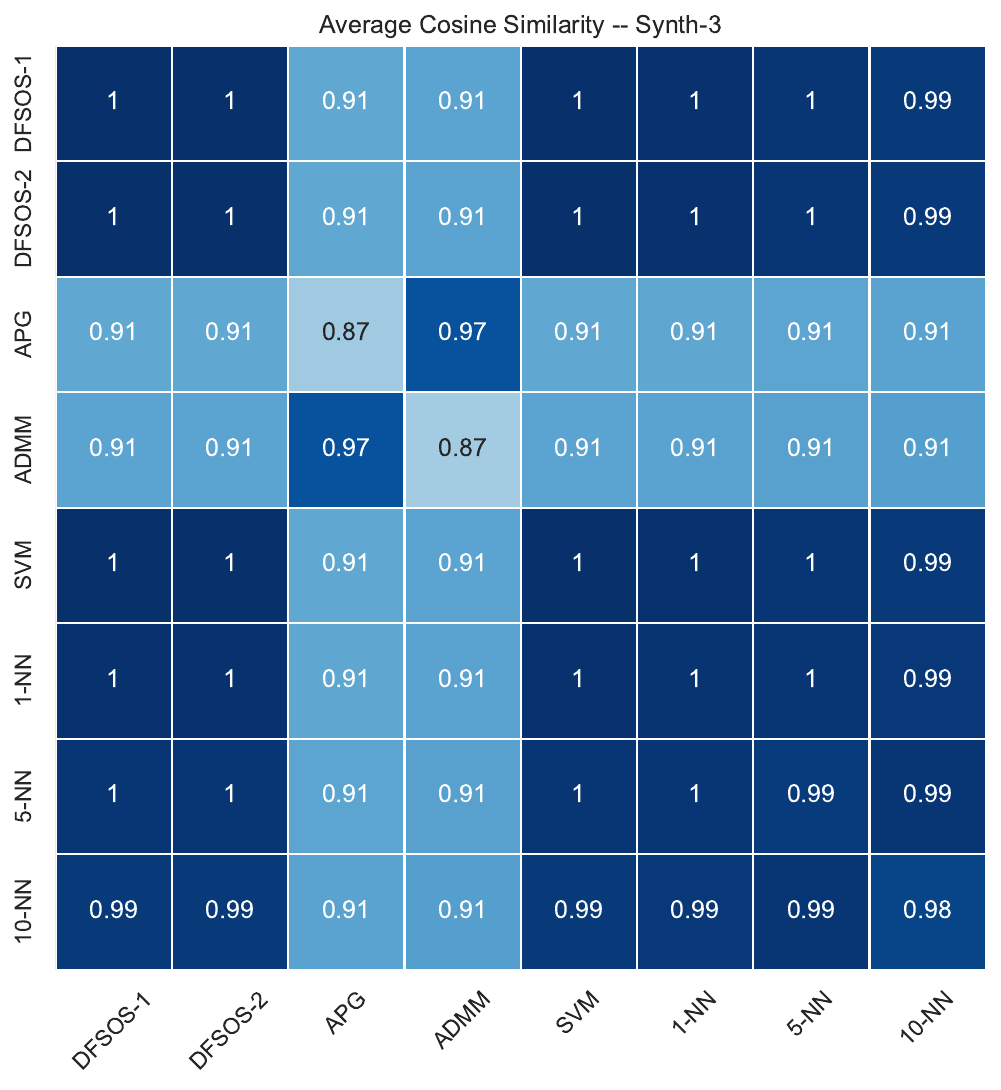}
        \caption{$(K,r) = (3, 0.9)$}
        \label{fig:3-3-J}
    \end{subfigure}
    
     \begin{subfigure}[t]{0.32\textwidth}
        \centering
        \includegraphics[width=0.99\linewidth]{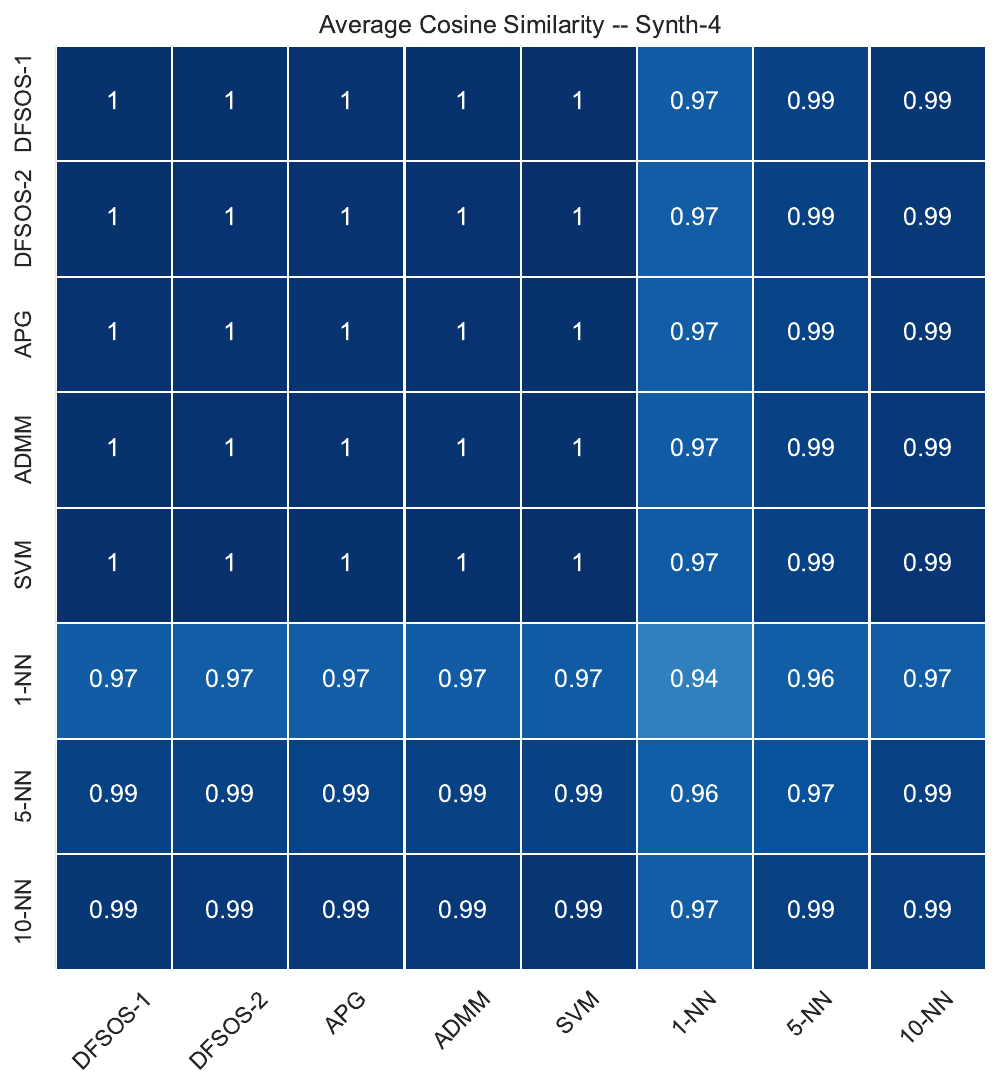}
        \caption{$(K,r) = (6, 0.1)$}
        \label{fig:3-4-J}
    \end{subfigure}
        \hfill
    \begin{subfigure}[t]{0.32\textwidth}
        \centering
        \includegraphics[width=0.99\linewidth]{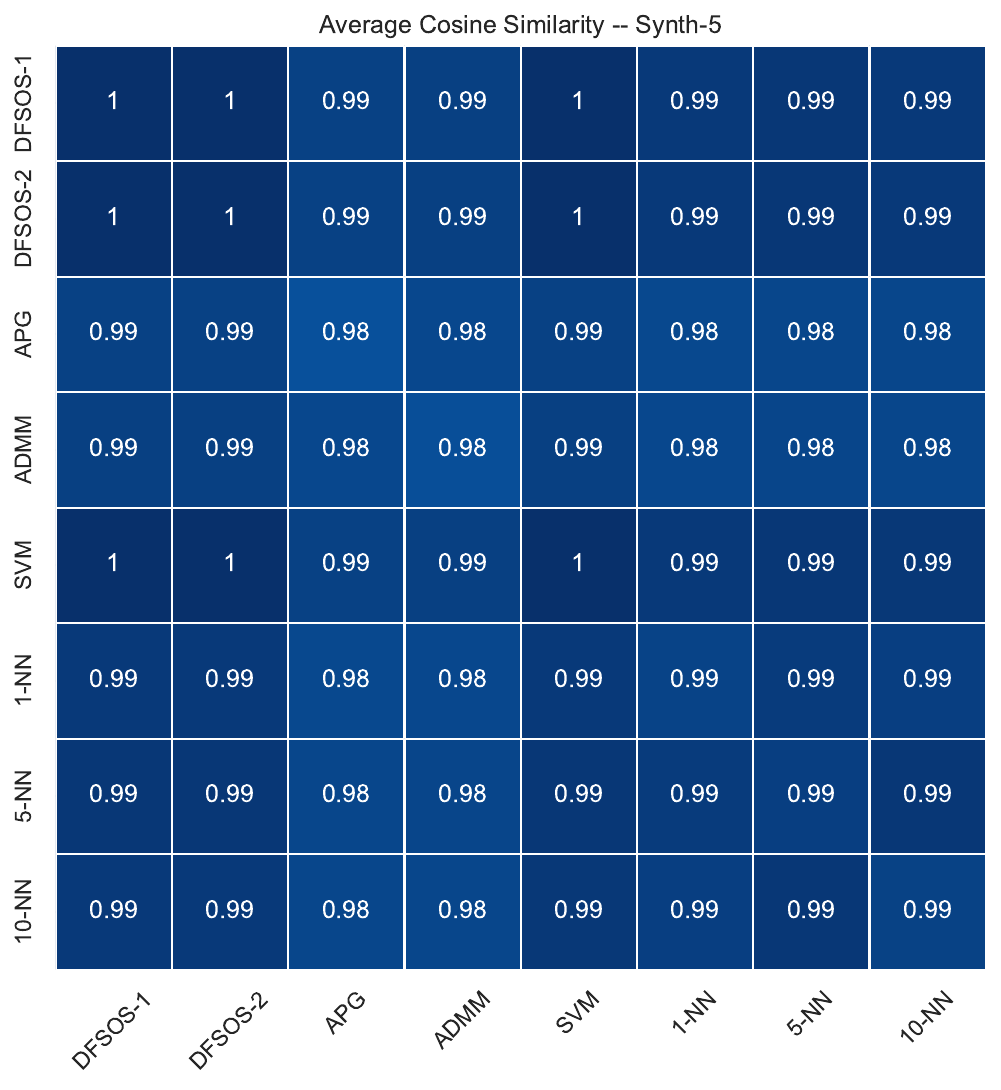}
        \caption{$(K,r) = (6, 0.5)$}
        \label{fig:3-5-J}
    \end{subfigure}
    \hfill
    \begin{subfigure}[t]{0.32\textwidth}
        \centering
        \includegraphics[width=0.99\linewidth]{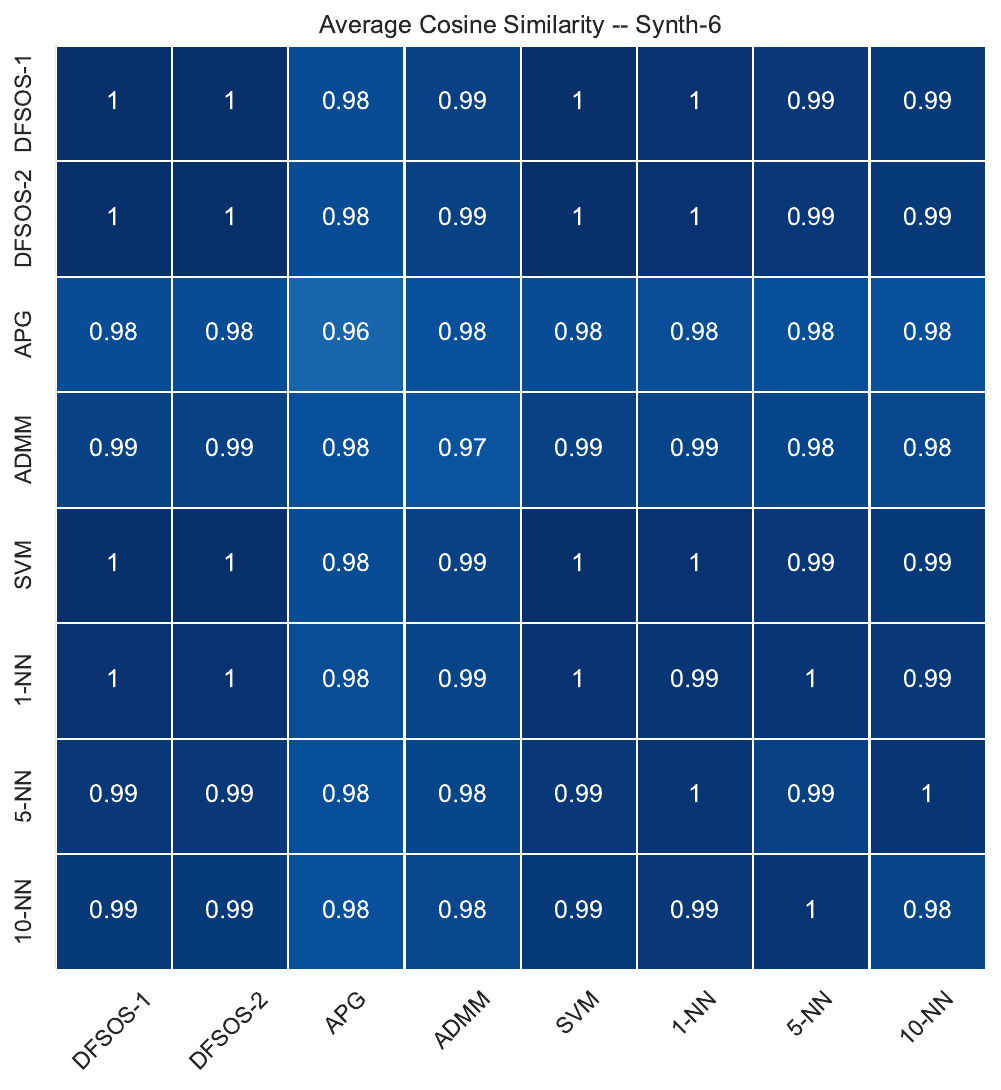}
        \caption{$(K,r) = (6, 0.9)$}
        \label{fig:3-6-J}
    \end{subfigure}
    \caption{Average cosine similarity between out-of-sample predictions for each pair of methods. The diagonal entries indicate the average cosine similarity between all pairs of prediction vectors made by each method across all $10$ Gaussian observations.}
    \label{fig:synth-cosine}
\end{figure}

%++++++++++++++++++++++++++++++++++++++++++++++++++++++++++++++++++
%++++++++++++++++++++++++++++++++++++++++++++++++++++++++++++++++++
% Real-Data
\subsection{Empirical Analysis of Time-Series Data}
\label{sec:real-data}
%++++++++++++++++++++++++++++++++++++++++++++++++++++++++++++++++++
%++++++++++++++++++++++++++++++++++++++++++++++++++++++++++++++++++

We also evaluated performance of DFSOS using a subset of data from the UC-Riverside classification data repository~\cite{dau2019ucr}.
We limited our analysis to data containing more predictor variables than observations ($p > n$) and at least $K=3$ classes; this yielded a collection of 15 data sets.
We used DFSOS, APG, and ADMM to obtain $q=K-1$ sparse discriminant vectors and performed nearest-centroid classification after projection onto the subspace spanned by these discriminant vectors. We compare classification accuracy of these methods with each other and standard classifiers SVM, KNN as before. For each data set we train each classifier using $10$ random initializations and compare average run-time and classification accuracy across these $10$ trials.

%++++++++++++++++++++++++++++++++++++++++++++++++++++++++++++++++++
\subsubsection{Experimental Parameters}
%++++++++++++++++++++++++++++++++++++++++++++++++++++++++++++++++++
As before, we use $5$-fold cross validation to choose the $\ell_1$-penalty $\lambda$ from the exponential grid $$\{2^{-4}, 2^{-3}, 2^{-2}, 2^{-1}, 1\} \times \lambda_{\max}$$ in DFSOS, APG, and APG.
We terminate the iterative update of $\pmb{\beta}$ in DFSOS, APG, and ADMM after $50$ iterations or a $10^{-5}$ suboptimal solution is found and we terminate the outer loop after $50$ iterations or a $10^{-4}$ suboptimal solution is found.
We use $\eta =0.25$, $\sigma =2$, and initial $\rho=5$ to update the augmented Lagrangian parameter $\rho$ in DFSOS.
We use augmented Lagrangian parameter $\mu = 2$ in ADMM.

%++++++++++++++++++++++++++++++++++++++++++++++++++++++++++++++++++
\subsubsection{Classification of UCR Time-Series Data}
%++++++++++++++++++++++++++++++++++++++++++++++++++++++++++++++++++

Summary statistics for our analysis of the UCR time-series data can be found in %Table~\ref{tab:UCR-1} and 
Table~\ref{tab:UCR-2}.

As before, we performed one-sided Wilcoxon tests to compare pairs of classification methods in terms of classification error, run-time, and cardinality. Figure~\ref{fig:UCR-hypothesis-tests} summarizes the results of these hypothesis tests. We observe similar outcomes to that of the experiments involving Gaussian data. DFSOS has similar run-time and classification performance to that of the two ASDA methods, especially ADMM. However, inspecting Table~\ref{tab:UCR-2} suggests that when DFSOS does provide an improvement in classification accuracy over ASDA then this improvement can be substantial, e.g., for the \emph{Beef} and \emph{EthanolLevel} data. In this case, we can infer that the improved stability obtained by eliminating deflationary errors allows computation of more accurate discriminant-based classifiers.

%++++++++++++++++++++++++++++++++++++++++++++++++
%++++++++++++++++++++++++++++++++++++++++++++++++
%++++++++++++++++++++++++++++++++++++++++++++++++
% UCR RESULTS TABLE 1
%++++++++++++++++++++++++++++++++++++++++++++++++
%++++++++++++++++++++++++++++++++++++++++++++++++
%++++++++++++++++++++++++++++++++++++++++++++++++
\begin{table}[t!]
\adjustbox{max width=\textwidth}{%
\begin{tabular}{llllllllll}
\toprule
        Data &      Measure &        DFSOS-1 &        DFSOS-2 &            APG &          ADMM &           SVM &          1-NN &          5-NN &         10-NN \\
\midrule
   Arrowhead &     Accuracy &    0.701 (0.0) &    0.701 (0.0) &    0.688 (0.0) &    0.71 (0.0) &   0.754 (0.0) &    0.76 (0.0) &   0.714 (0.0) &   0.566 (0.0) \\
             & Run-time (s) &  0.924 (0.695) &  0.819 (0.243) &   0.67 (0.259) &  0.271 (0.04) & 0.362 (0.499) & 0.026 (0.003) &   0.017 (0.0) &   0.013 (0.0) \\
             &  Cardinality &  0.554 (0.066) &  0.554 (0.066) &  0.891 (0.006) & 0.572 (0.004) &            -- &            -- &            -- &            -- \\ \midrule
        Beef &     Accuracy &    0.867 (0.0) &    0.867 (0.0) &    0.579 (0.0) & 0.642 (0.001) &   0.867 (0.0) &     0.7 (0.0) &   0.633 (0.0) &   0.433 (0.0) \\
             & Run-time (s) &  0.908 (0.726) &  0.842 (0.285) &  1.343 (0.665) & 0.459 (0.131) & 0.364 (0.394) & 0.029 (0.002) &   0.019 (0.0) &   0.017 (0.0) \\
             &  Cardinality &    0.979 (0.0) &    0.979 (0.0) &  0.918 (0.007) & 0.587 (0.012) &            -- &            -- &            -- &            -- \\ \midrule
         BME &     Accuracy &    0.896 (0.0) &    0.896 (0.0) &    0.924 (0.0) &   0.926 (0.0) &   0.893 (0.0) &    0.84 (0.0) &    0.74 (0.0) &   0.573 (0.0) \\
             & Run-time (s) &  0.955 (0.692) &   0.725 (0.19) &  0.275 (0.061) & 0.231 (0.029) & 0.405 (0.645) &  0.04 (0.007) & 0.026 (0.002) &   0.015 (0.0) \\
             &  Cardinality &  0.271 (0.029) &  0.271 (0.029) &  0.463 (0.025) & 0.404 (0.004) &            -- &            -- &            -- &            -- \\ \midrule
         Car &     Accuracy &   0.77 (0.001) &   0.77 (0.001) &     0.75 (0.0) &   0.765 (0.0) &   0.817 (0.0) &   0.717 (0.0) &   0.583 (0.0) &     0.6 (0.0) \\
             & Run-time (s) &  1.018 (0.514) &  0.937 (0.268) &  0.971 (0.491) &  0.48 (0.184) & 0.338 (0.429) & 0.025 (0.003) & 0.023 (0.001) & 0.016 (0.001) \\
             &  Cardinality &  0.793 (0.085) &  0.793 (0.085) &    0.977 (0.0) & 0.612 (0.003) &            -- &            -- &            -- &            -- \\ \midrule
         CBF &     Accuracy &  0.688 (0.003) &  0.688 (0.003) &    0.852 (0.0) &   0.851 (0.0) &   0.862 (0.0) &   0.844 (0.0) &   0.764 (0.0) &   0.649 (0.0) \\
             & Run-time (s) &  0.713 (0.234) &  0.669 (0.171) &  0.498 (0.302) & 0.285 (0.068) & 0.442 (0.723) & 0.049 (0.015) & 0.032 (0.003) &  0.02 (0.001) \\
             &  Cardinality &  0.294 (0.036) &  0.294 (0.036) &  0.282 (0.002) & 0.309 (0.004) &            -- &            -- &            -- &            -- \\ \midrule
EthanolLevel &     Accuracy &     0.74 (0.0) &     0.74 (0.0) &  0.391 (0.001) & 0.417 (0.001) &   0.838 (0.0) &   0.334 (0.0) &   0.356 (0.0) &   0.334 (0.0) \\
             & Run-time (s) &  2.098 (0.647) &  1.946 (0.307) & 13.223 (0.826) & 6.632 (0.847) & 1.284 (0.105) & 0.064 (0.002) & 0.061 (0.001) &  0.06 (0.001) \\
             &  Cardinality &    0.997 (0.0) &    0.997 (0.0) &    0.952 (0.0) & 0.827 (0.005) &            -- &            -- &            -- &            -- \\ \midrule
        Fish &     Accuracy &    0.823 (0.0) &    0.823 (0.0) &    0.802 (0.0) &   0.797 (0.0) &   0.834 (0.0) &   0.811 (0.0) &   0.766 (0.0) &   0.754 (0.0) \\
             & Run-time (s) &  0.533 (0.096) &   0.571 (0.14) &  3.696 (2.815) & 1.425 (0.256) & 0.413 (0.108) &  0.03 (0.002) &   0.025 (0.0) &   0.018 (0.0) \\
             &  Cardinality &    0.997 (0.0) &    0.997 (0.0) &     0.97 (0.0) & 0.759 (0.001) &            -- &            -- &            -- &            -- \\ \midrule
      Mallat &     Accuracy &    0.787 (0.0) &  0.709 (0.001) &    0.919 (0.0) &    0.93 (0.0) &   0.888 (0.0) &   0.738 (0.0) &   0.623 (0.0) &     0.5 (0.0) \\
             & Run-time (s) & 14.538 (1.918) & 13.451 (0.356) & 16.903 (1.541) & 9.987 (1.241) & 0.491 (0.103) & 0.027 (0.002) & 0.022 (0.001) & 0.016 (0.001) \\
             &  Cardinality &    0.963 (0.0) &  0.719 (0.001) &    0.918 (0.0) & 0.463 (0.004) &            -- &            -- &            -- &            -- \\
\midrule
                  Meat &     Accuracy &     0.95 (0.0) &      0.95 (0.0) &    0.983 (0.0) &    0.983 (0.0) &   0.933 (0.0) &   0.933 (0.0) &   0.883 (0.0) &   0.883 (0.0) \\
                       & Run-time (s) &  0.508 (0.331) &    0.444 (0.12) &  0.829 (0.501) &  0.366 (0.079) & 0.395 (0.586) &  0.05 (0.016) & 0.023 (0.001) &   0.012 (0.0) \\
                       &  Cardinality &    0.997 (0.0) &     0.997 (0.0) &    0.956 (0.0) &  0.523 (0.008) &            -- &            -- &            -- &            -- \\ \midrule
              OliveOil &     Accuracy &    0.867 (0.0) &     0.867 (0.0) &      0.9 (0.0) &      0.9 (0.0) &   0.867 (0.0) &   0.833 (0.0) &     0.8 (0.0) &   0.767 (0.0) \\
                       & Run-time (s) &  0.202 (0.012) &   0.241 (0.023) &  0.922 (0.546) &    0.505 (0.2) & 0.547 (1.269) &  0.048 (0.01) & 0.034 (0.004) &   0.015 (0.0) \\
                       &  Cardinality &    0.995 (0.0) &     0.995 (0.0) &    0.944 (0.0) &  0.596 (0.005) &            -- &            -- &            -- &            -- \\ \midrule
                 Plane &     Accuracy &    0.978 (0.0) &     0.978 (0.0) &    0.981 (0.0) &    0.981 (0.0) &   0.981 (0.0) &   0.962 (0.0) &   0.952 (0.0) &   0.952 (0.0) \\
                       & Run-time (s) &  1.624 (1.753) &   1.282 (0.413) &  3.093 (1.615) &  0.731 (0.137) & 0.462 (0.279) & 0.028 (0.003) &   0.022 (0.0) &   0.016 (0.0) \\
                       &  Cardinality &  0.518 (0.026) &   0.518 (0.026) &  0.623 (0.003) &  0.573 (0.003) &            -- &            -- &            -- &            -- \\ \midrule
                  Rock &     Accuracy &     0.66 (0.0) &      0.66 (0.0) &  0.673 (0.019) &  0.782 (0.007) &    0.84 (0.0) &    0.76 (0.0) &    0.42 (0.0) &    0.46 (0.0) \\
                       & Run-time (s) &  2.255 (1.919) &   1.961 (0.178) &  2.911 (0.565) &  1.858 (0.075) & 0.251 (0.134) &  0.03 (0.003) & 0.021 (0.001) &   0.016 (0.0) \\
                       &  Cardinality &  0.199 (0.001) &   0.199 (0.001) &  0.552 (0.158) &  0.064 (0.001) &            -- &            -- &            -- &            -- \\  \midrule
       StarLight- &     Accuracy &    0.774 (0.0) &     0.774 (0.0) &    0.774 (0.0) &    0.769 (0.0) &   0.926 (0.0) &   0.844 (0.0) &   0.837 (0.0) &   0.816 (0.0) \\
        Curves               & Run-time (s) & 18.869 (19.86) & 17.648 (12.908) &  4.218 (0.281) & 11.876 (6.043) & 3.034 (0.102) & 0.074 (0.003) & 0.068 (0.001) &    0.06 (0.0) \\
                       &  Cardinality &  0.478 (0.086) &   0.478 (0.086) &  0.822 (0.003) &  0.399 (0.008) &            -- &            -- &            -- &            -- \\ \midrule
               Symbols &     Accuracy &    0.824 (0.0) &   0.827 (0.001) &  0.796 (0.002) &  0.824 (0.001) &   0.837 (0.0) &   0.885 (0.0) &   0.677 (0.0) &   0.183 (0.0) \\
                       & Run-time (s) &  4.107 (3.751) &   3.287 (0.892) &  3.296 (0.527) &   1.779 (0.34) & 0.319 (0.089) & 0.023 (0.002) & 0.019 (0.001) & 0.014 (0.001) \\
                       &  Cardinality &    0.179 (0.0) &     0.206 (0.0) &  0.783 (0.069) &  0.463 (0.023) &            -- &            -- &            -- &            -- \\ \midrule
UWave- &     Accuracy &      0.5 (0.0) &       0.5 (0.0) &    0.834 (0.0) &    0.828 (0.0) &   0.881 (0.0) &   0.949 (0.0) &    0.94 (0.0) &   0.931 (0.0) \\
       Gesture-                & Run-time (s) &  1.571 (0.257) &   1.282 (0.168) & 20.479 (0.884) & 23.313 (1.142) & 1.034 (0.122) & 0.059 (0.002) & 0.055 (0.001) & 0.048 (0.001) \\
         LibraryALL            &  Cardinality &    0.998 (0.0) &     0.998 (0.0) &    0.946 (0.0) &    0.736 (0.0) &            -- &            -- &            -- &            -- \\
\bottomrule
\end{tabular}
}
\caption{
Out-of-sample prediction accuracy, run-time in seconds, and fraction of non-zero features for discriminant vectors and classifiers for second subset of data from the UCR Time Series repository.}
\label{tab:UCR-2}
\end{table}

%++++++++++++++++++++++++++++++++++++++++++++++++++++++++++++++++++++++++++++++++++++++++++++++++++++++++
%++++++++++++++++++++++++++++++++++++++++++++++++++++++++++++++++++++++++++++++++++++++++++++++++++++++++
% Hyptothesis test results.
%++++++++++++++++++++++++++++++++++++++++++++++++++++++++++++++++++++++++++++++++++++++++++++++++++++++++
%++++++++++++++++++++++++++++++++++++++++++++++++++++++++++++++++++++++++++++++++++++++++++++++++++++++++
\begin{figure}
    \centering

    \begin{subfigure}[t]{0.3\textwidth}
        \centering
        \includegraphics[width=0.99\linewidth]{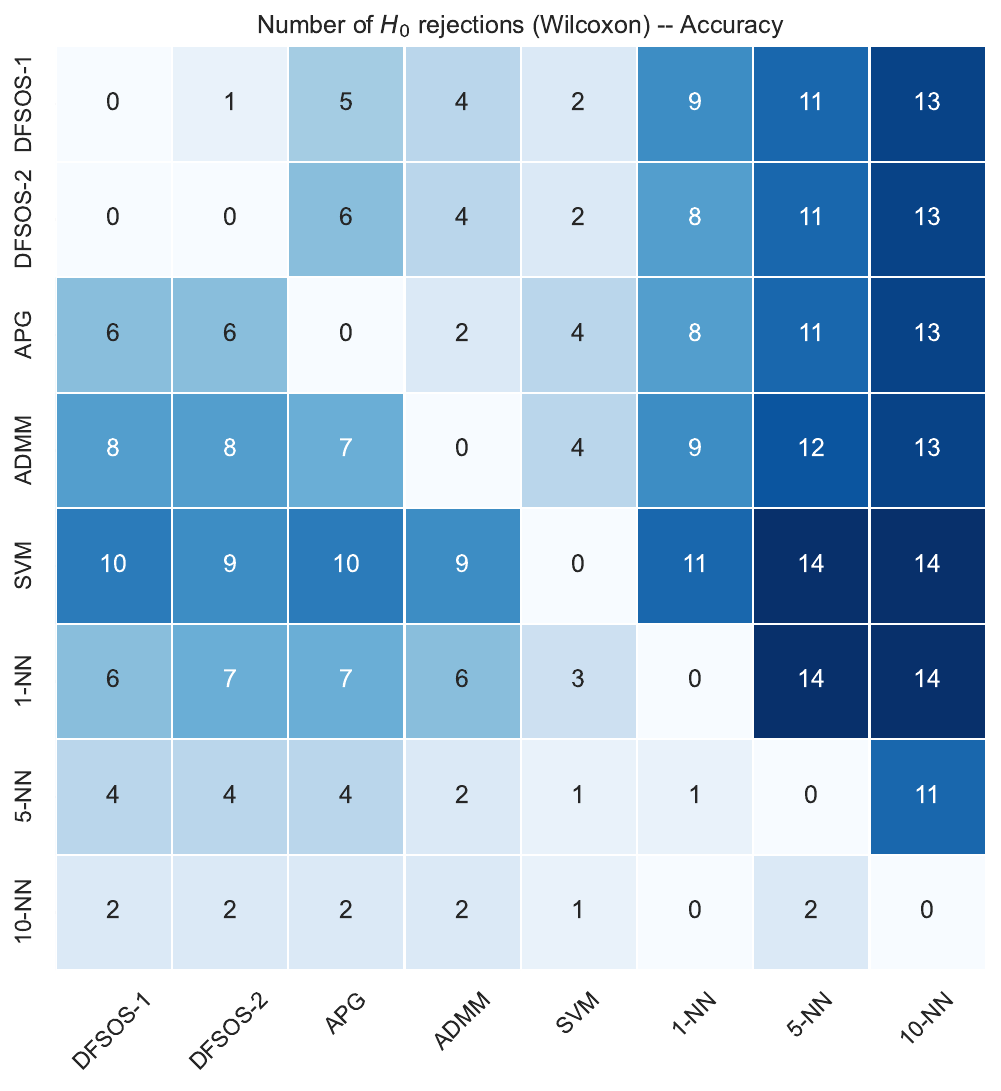}
        \caption{Classification Accuracy}
        \label{fig:DV-5-DF2-boundaries}
    \end{subfigure}
        \hfill
    \begin{subfigure}[t]{0.3\textwidth}
        \centering
        \includegraphics[width=0.99\linewidth]{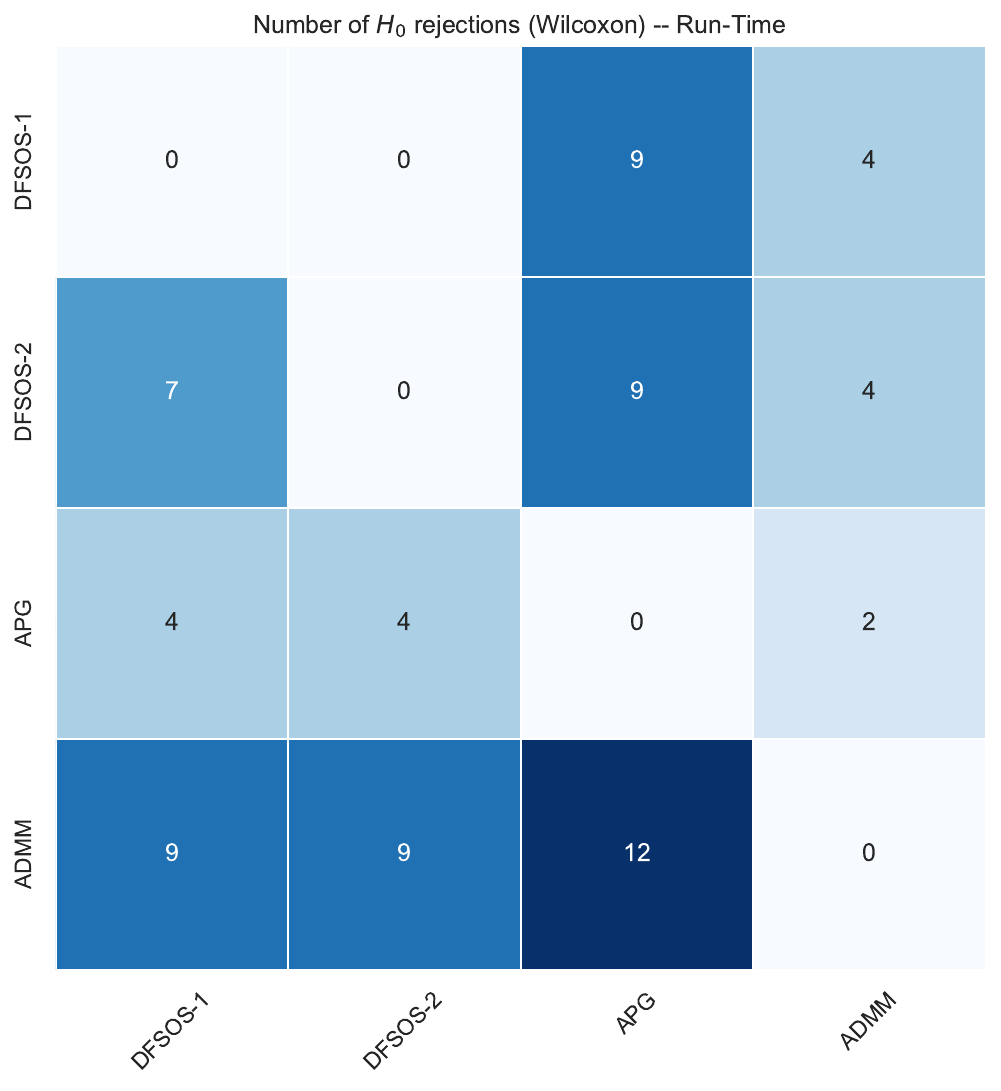}
        \caption{Run-times}
        \label{fig:synth-test-times}
    \end{subfigure}
    \hfill
    \begin{subfigure}[t]{0.3\textwidth}
        \centering
        \includegraphics[width=0.99\linewidth]{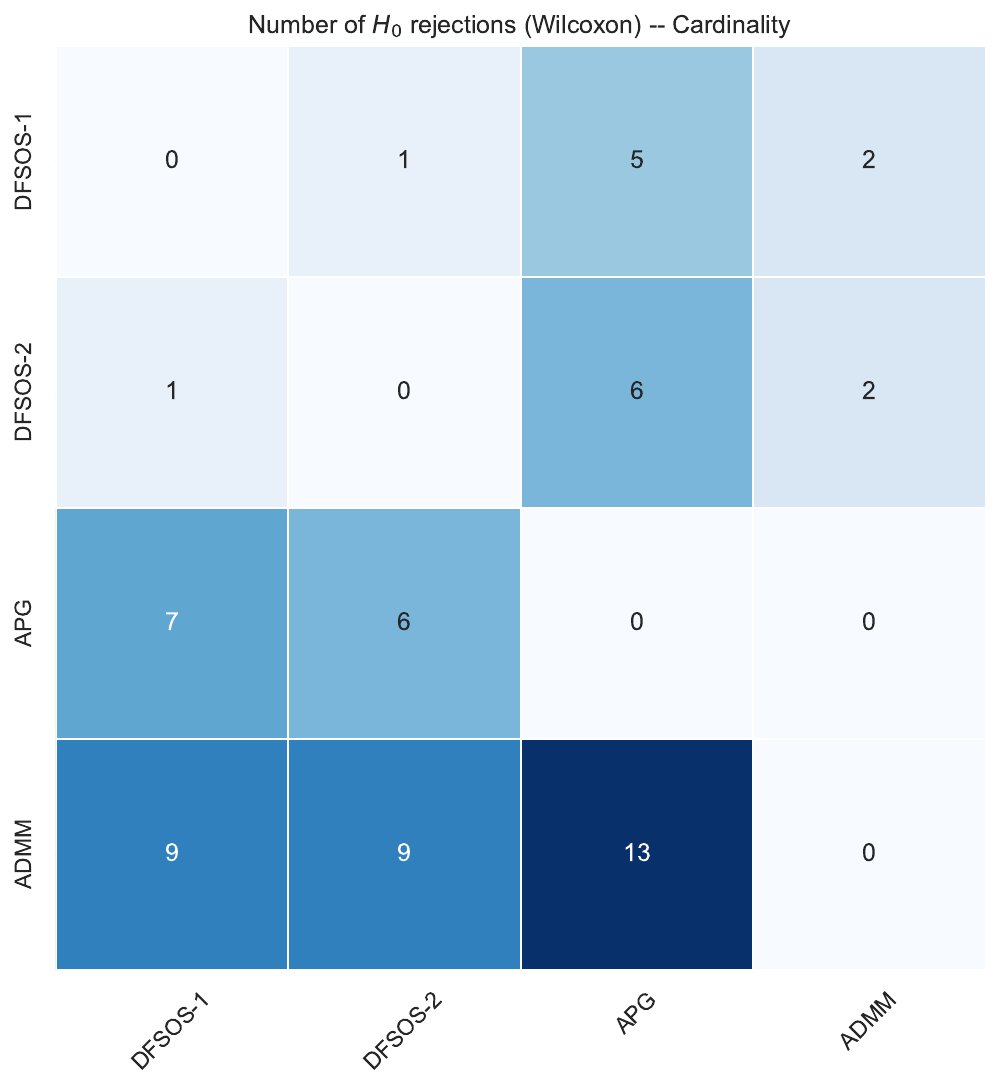}
        \caption{Cardinality}
        \label{fig:synth-test-card}
    \end{subfigure}
    \caption{Number of time-series datasets from UCR repository for which we observe statistically significant improvement in classification accuracy, computational efficiency, and cardinality of discriminants when using row method compared to column method. That is, the $(i,j)$-entry is the number of data sets where we reject the null hypothesis $H_0:$ there is no difference between method $i$ and method $j$ using the Wilcoxon test with one-sided alternative hypothesis $H_a:$ method $i$ is better than method $j$.}
    \label{fig:UCR-hypothesis-tests}
\end{figure}

%++++++++++++++++++++++++++++++++++++++++++++++++++++++++++++++++++++++++++++++++++++++++++++++++++++++++
%++++++++++++++++++++++++++++++++++++++++++++++++++++++++++++++++++++++++++++++++++++++++++++++++++++++++
% Prediction Similarity.
%++++++++++++++++++++++++++++++++++++++++++++++++++++++++++++++++++++++++++++++++++++++++++++++++++++++++
%++++++++++++++++++++++++++++++++++++++++++++++++++++++++++++++++++++++++++++++++++++++++++++++++++++++++
\begin{figure}
    \centering

    \begin{subfigure}[t]{0.19\textwidth}
        \centering
        \includegraphics[width=0.99\linewidth]{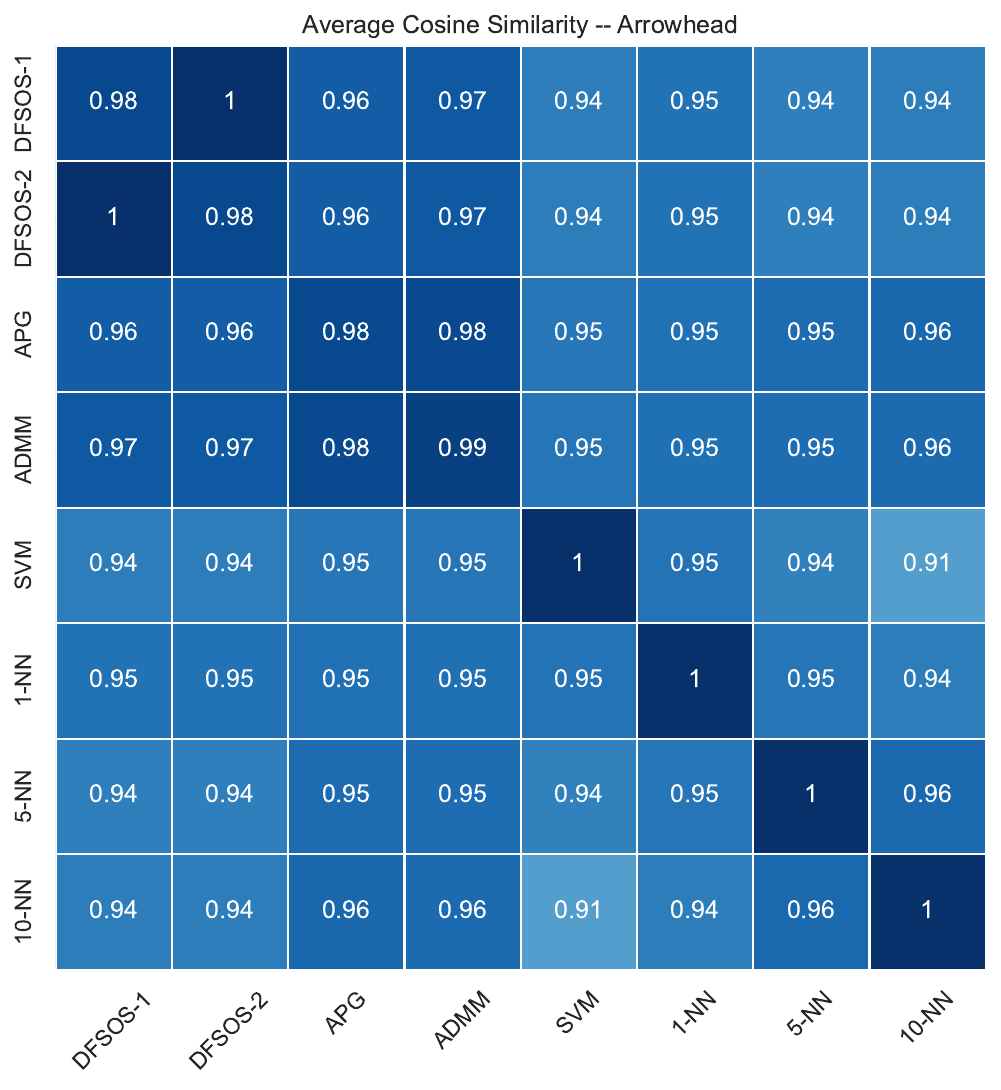}
        \caption{Arrowhead}
        \label{fig:arrowhead}
    \end{subfigure}
        \hfill
    \begin{subfigure}[t]{0.19\textwidth}
        \centering
        \includegraphics[width=0.99\linewidth]{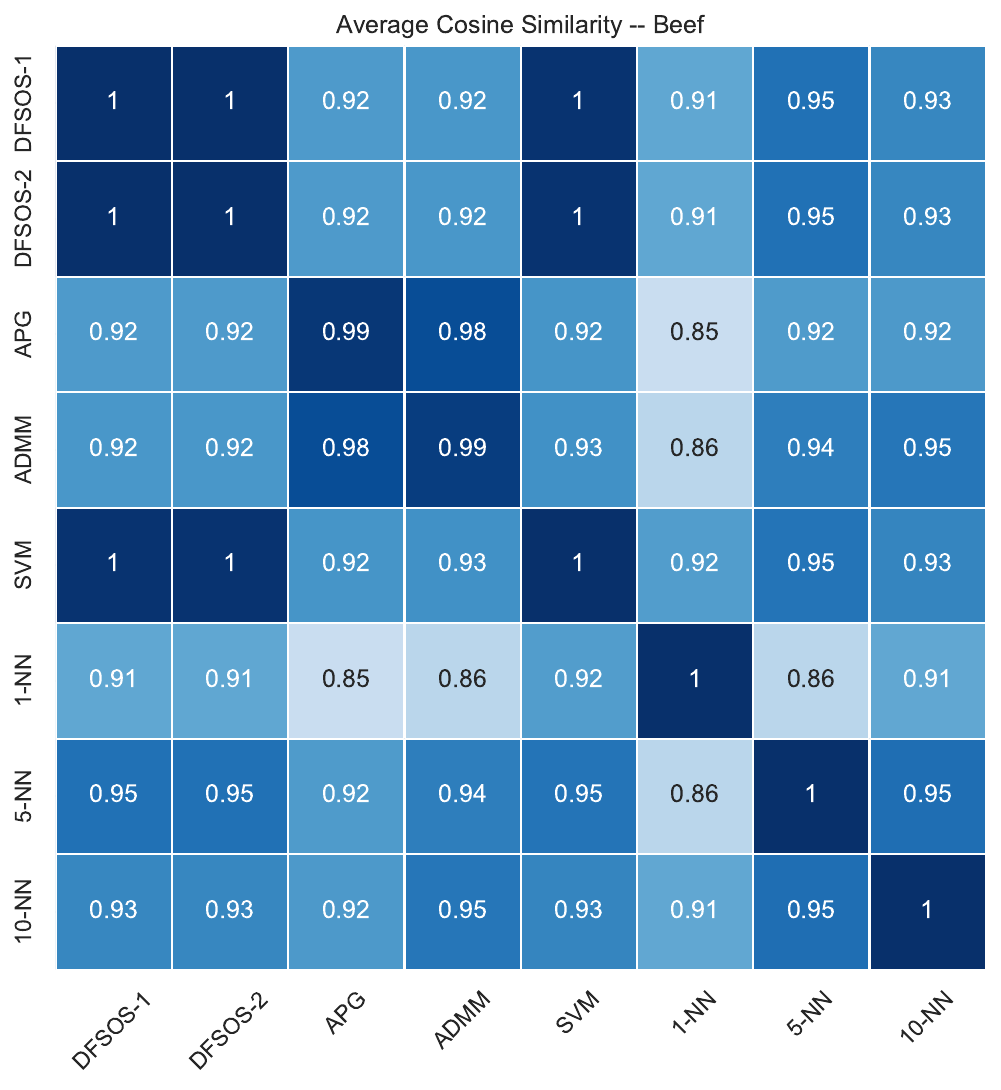}
        \caption{Beef}
        \label{fig:beef}
    \end{subfigure}
    \hfill
    \begin{subfigure}[t]{0.19\textwidth}
        \centering
        \includegraphics[width=0.99\linewidth]{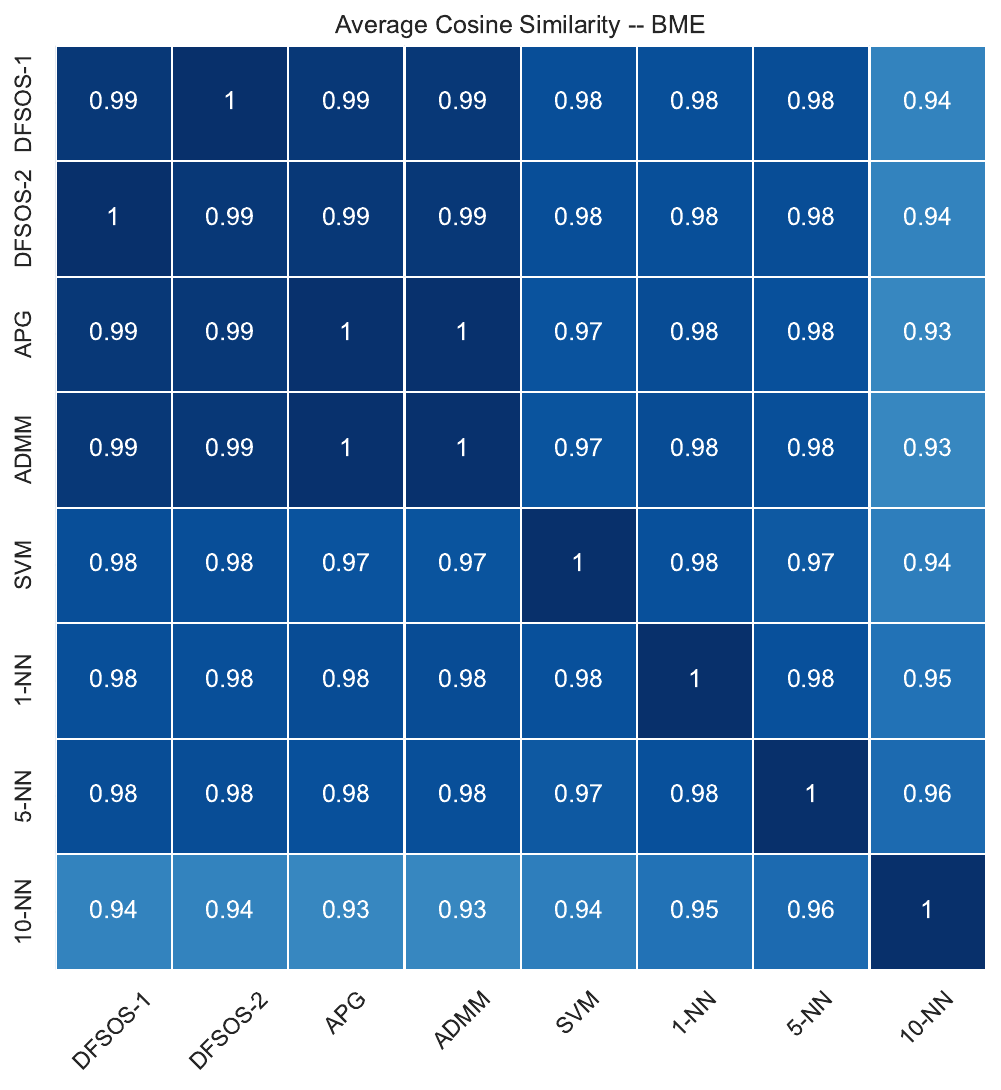}
        \caption{BME}
        \label{fig:BME}
    \end{subfigure}
    \hfill
    \begin{subfigure}[t]{0.19\textwidth}
        \centering
        \includegraphics[width=0.99\linewidth]{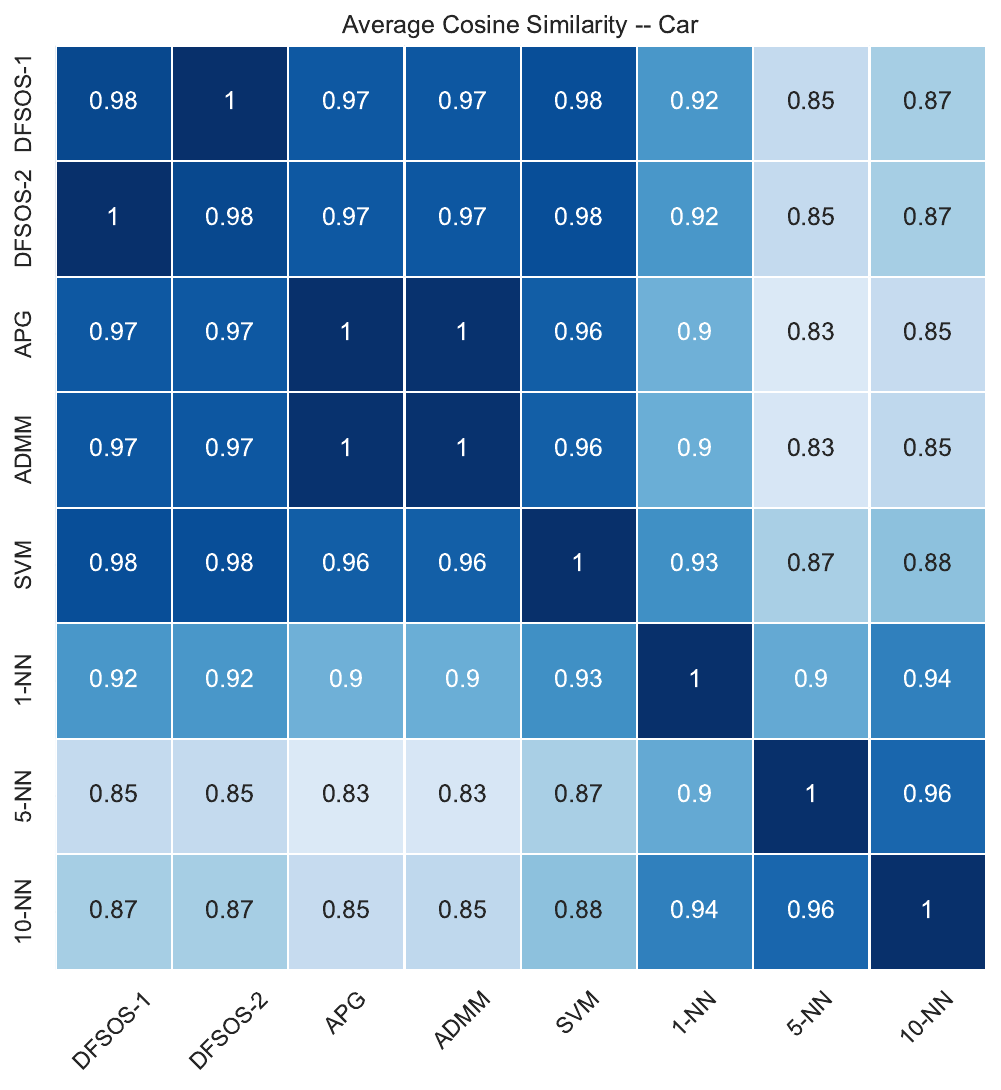}
        \caption{Car}
        \label{fig:Car}
    \end{subfigure}
    \hfill
    \begin{subfigure}[t]{0.19\textwidth}
        \centering
        \includegraphics[width=0.99\linewidth]{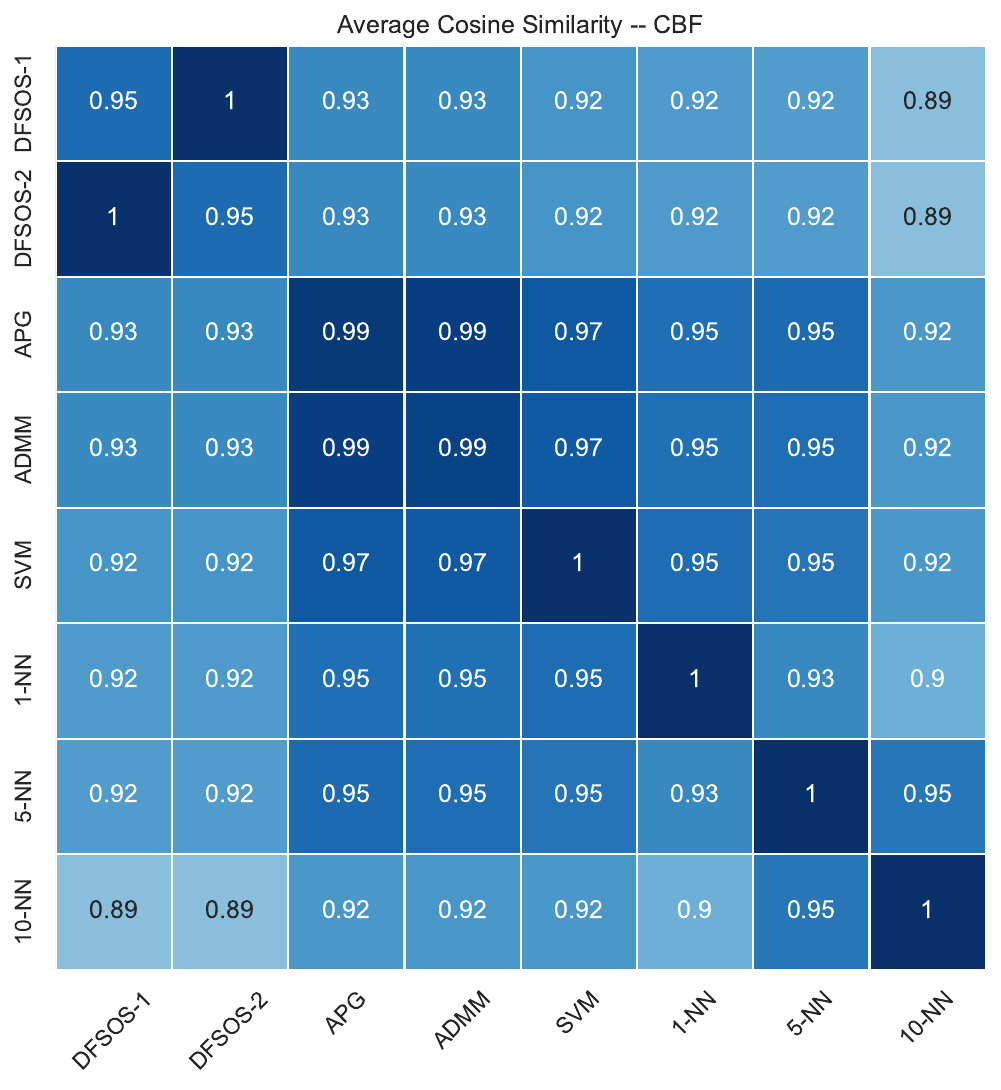}
        \caption{CBF}
        \label{fig:CBF}
    \end{subfigure}

    \begin{subfigure}[t]{0.19\textwidth}
        \centering
        \includegraphics[width=0.99\linewidth]{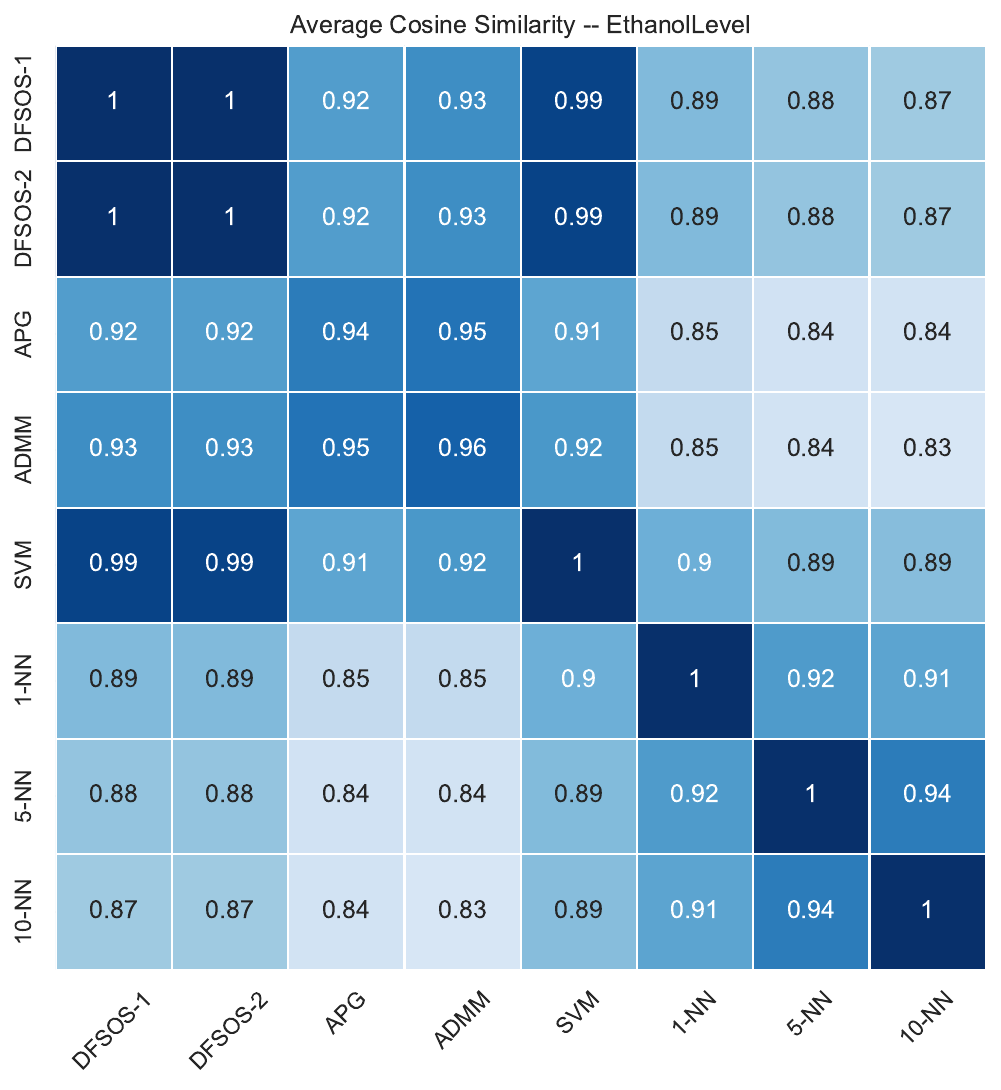}
        \caption{EthanolLevel}
        \label{fig:EthanolLevel}
    \end{subfigure}
        \hfill
    \begin{subfigure}[t]{0.19\textwidth}
        \centering
        \includegraphics[width=0.99\linewidth]{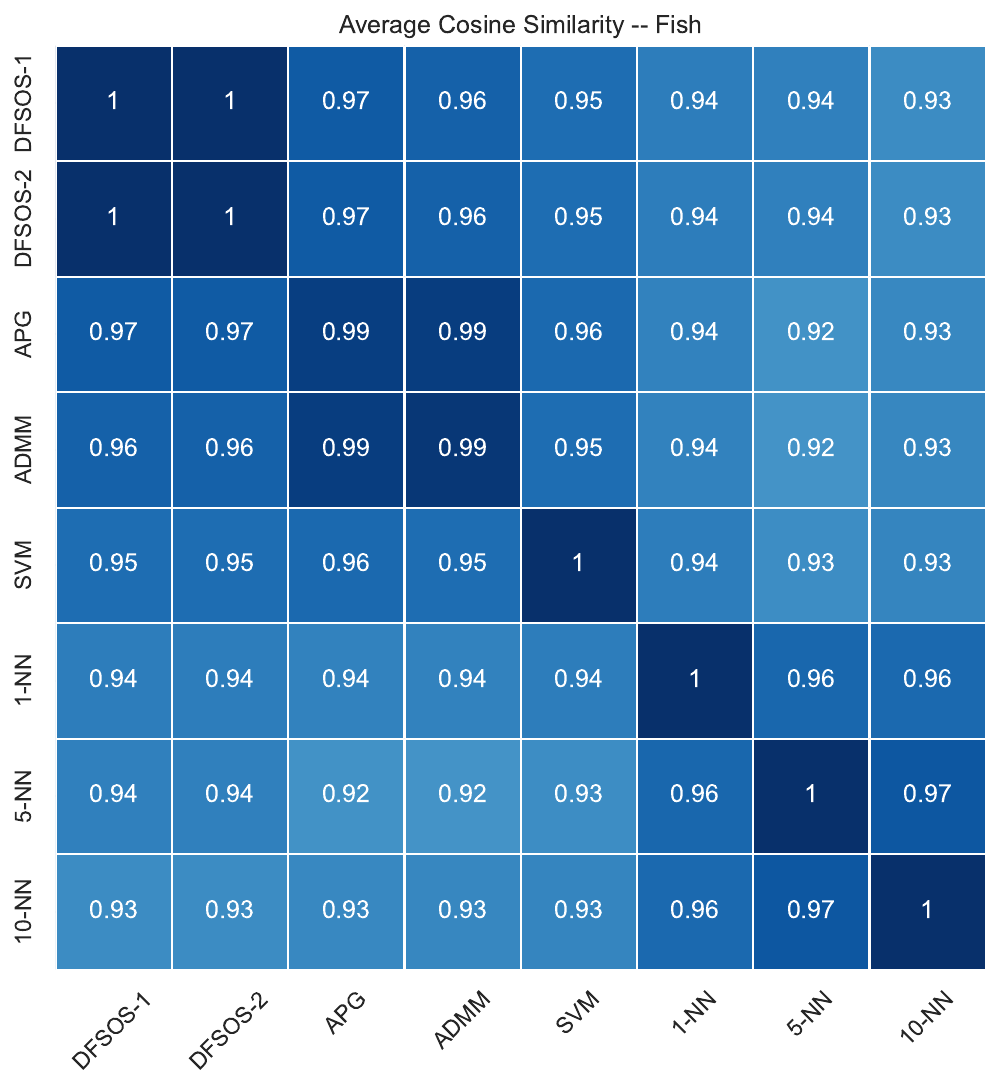}
        \caption{Fish}
        \label{fig:Fish}
    \end{subfigure}
    \hfill
    \begin{subfigure}[t]{0.19\textwidth}
        \centering
        \includegraphics[width=0.99\linewidth]{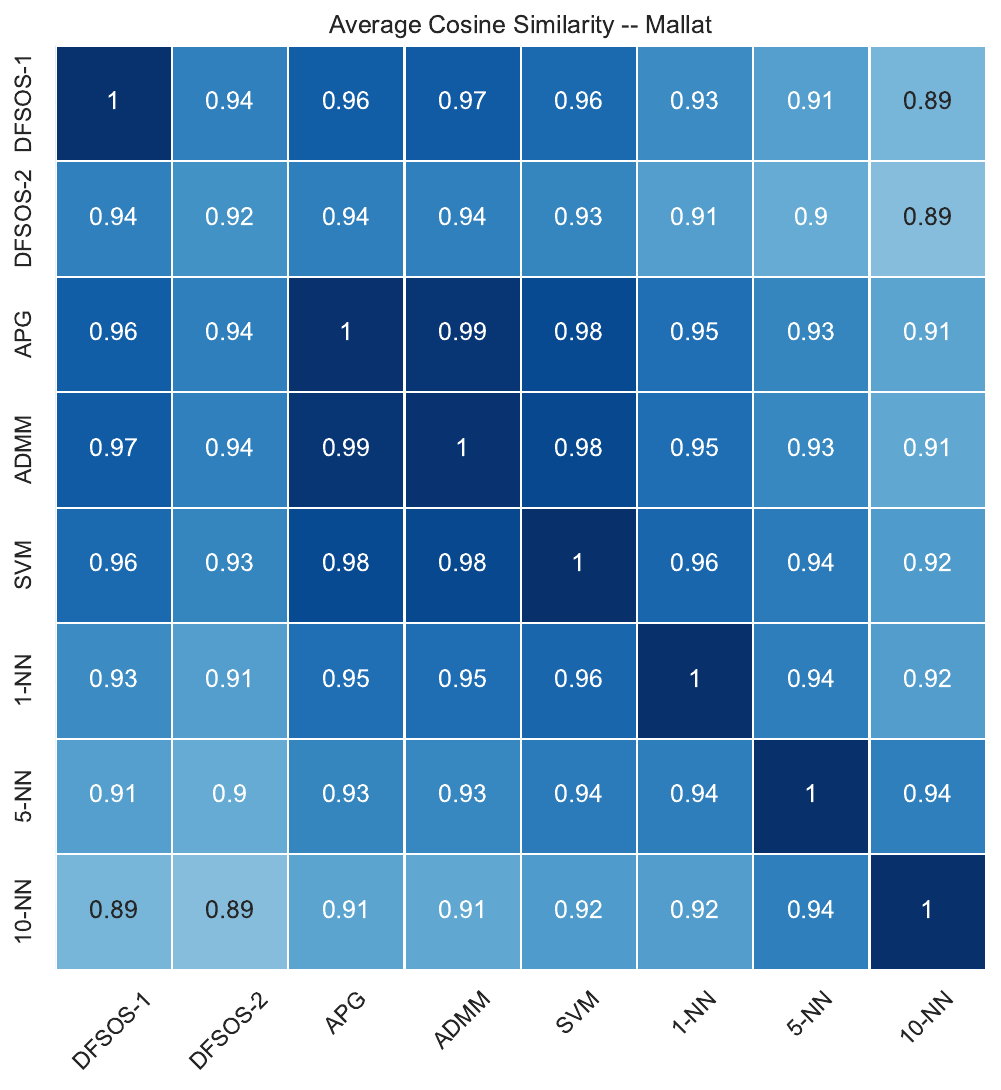}
        \caption{Mallat}
        \label{fig:Mallat}
    \end{subfigure}
    \hfill
    \begin{subfigure}[t]{0.19\textwidth}
        \centering
        \includegraphics[width=0.99\linewidth]{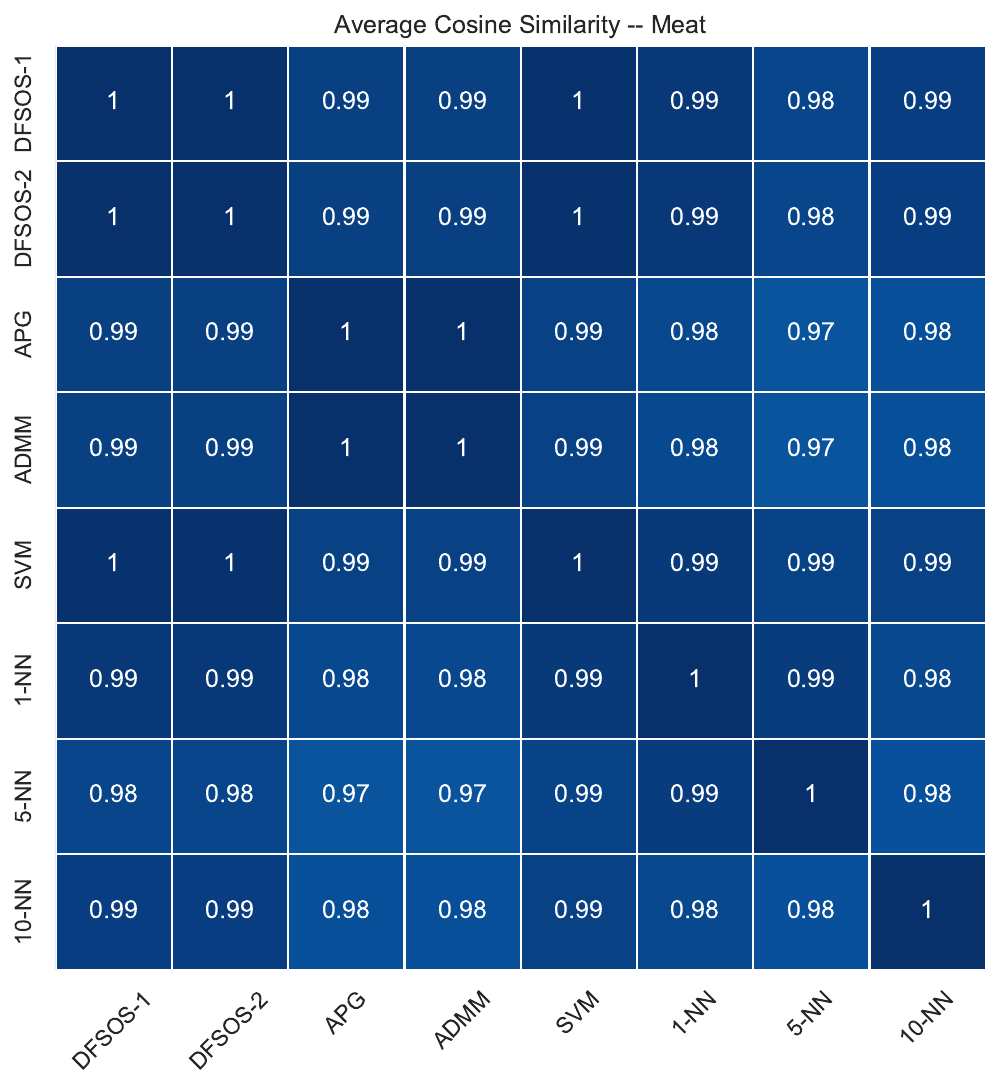}
        \caption{Meat}
        \label{fig:Meat}
    \end{subfigure}
    \hfill
    \begin{subfigure}[t]{0.19\textwidth}
        \centering
        \includegraphics[width=0.99\linewidth]{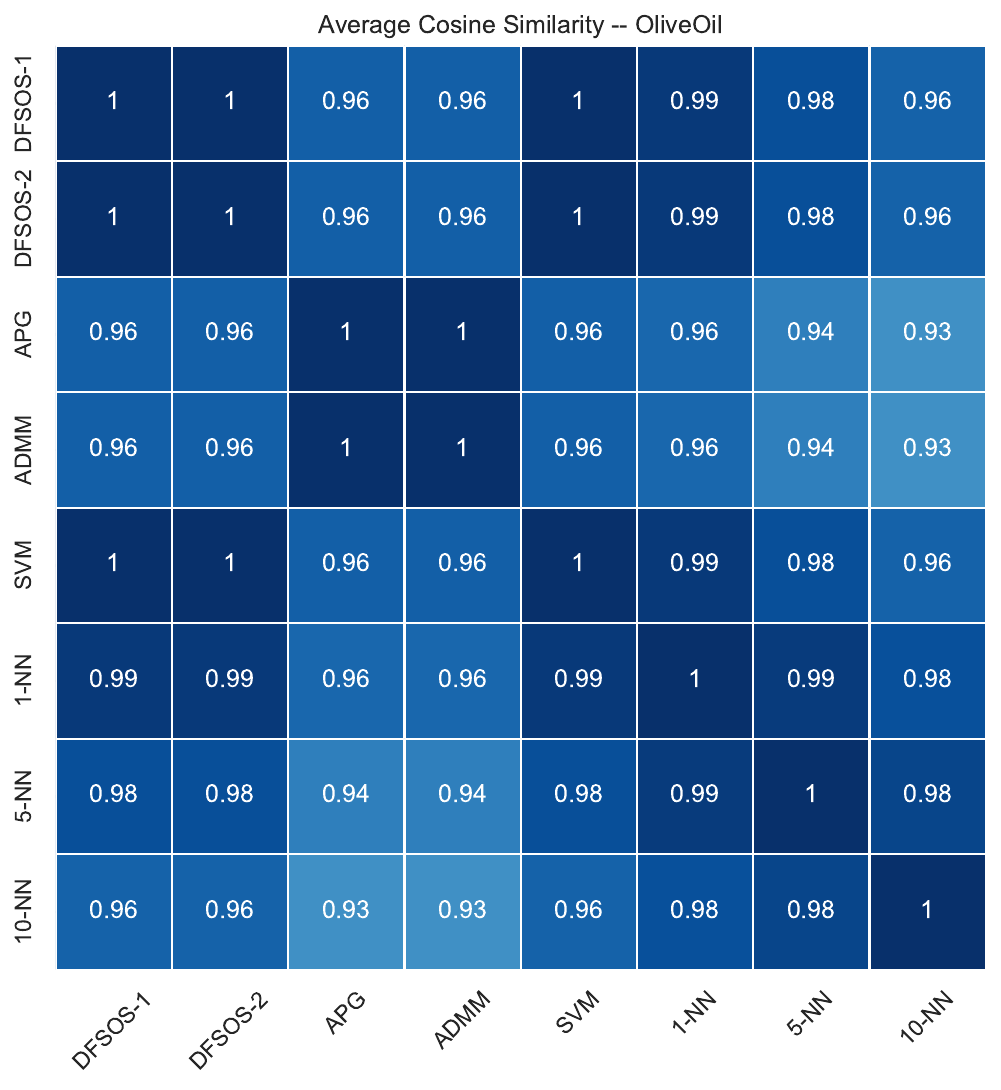}
        \caption{OliveOil}
        \label{fig:OliveOil}
    \end{subfigure}

    \begin{subfigure}[t]{0.19\textwidth}
        \centering
        \includegraphics[width=0.99\linewidth]{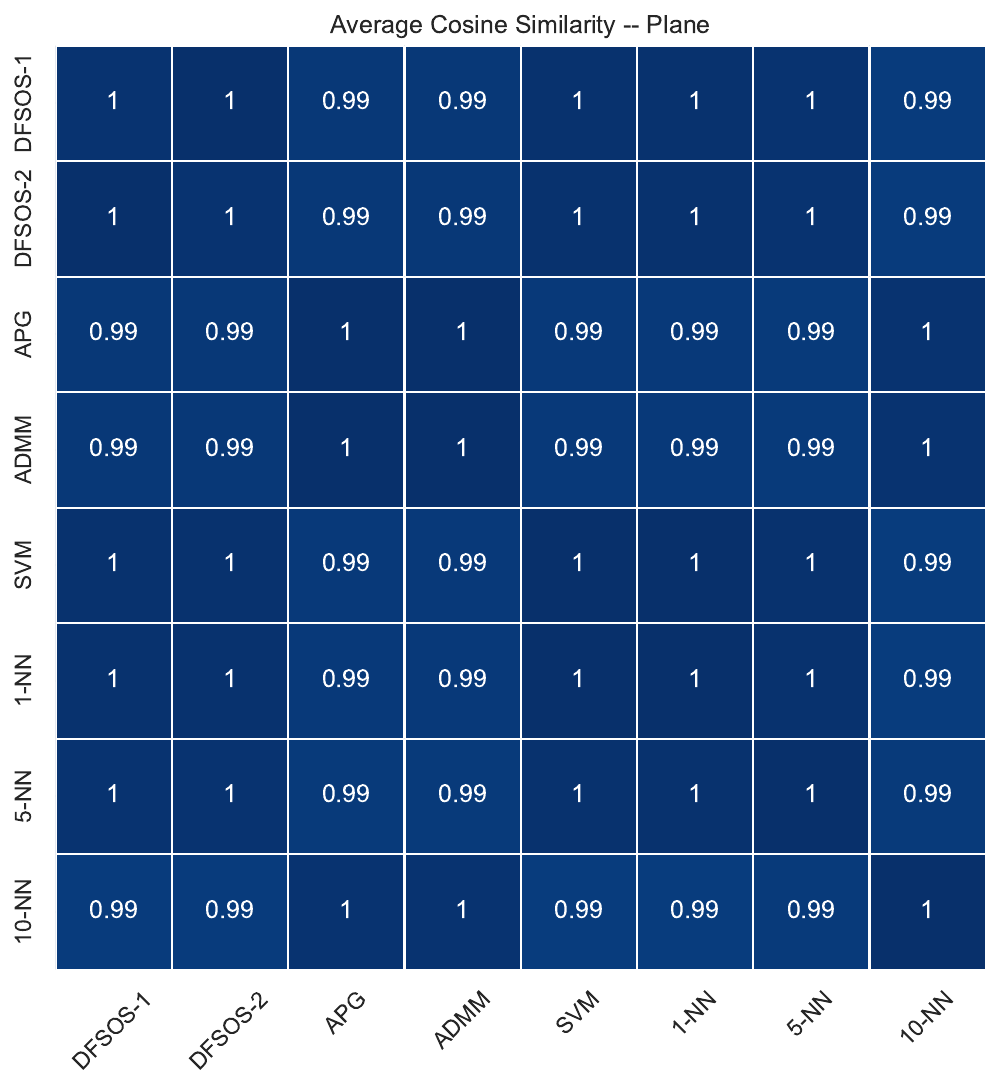}
        \caption{Plane}
        \label{fig:Plane}
    \end{subfigure}
    \hfill
    \begin{subfigure}[t]{0.19\textwidth}
        \centering
        \includegraphics[width=0.99\linewidth]{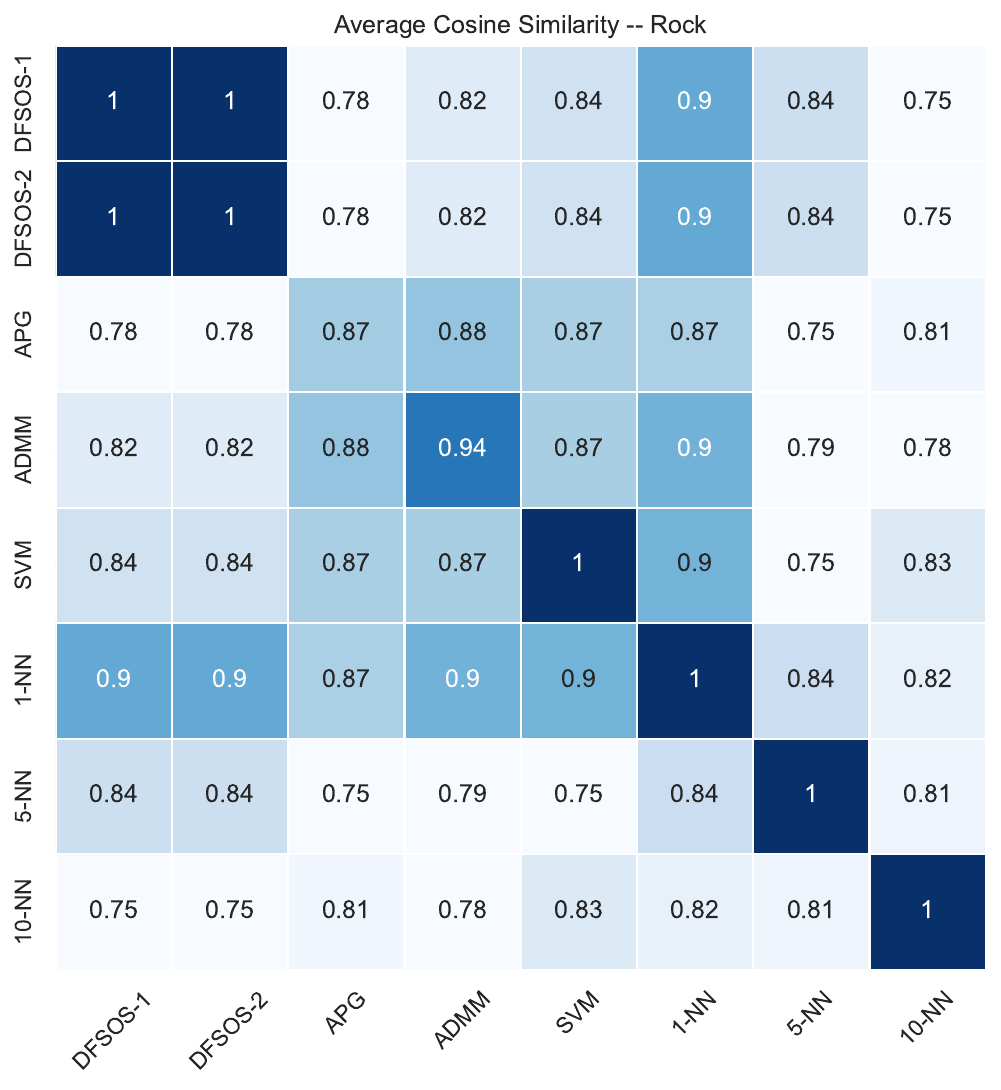}
        \caption{Rock}
        \label{fig:Rock}
    \end{subfigure}
    \hfill
    \begin{subfigure}[t]{0.19\textwidth}
        \centering
        \includegraphics[width=0.99\linewidth]{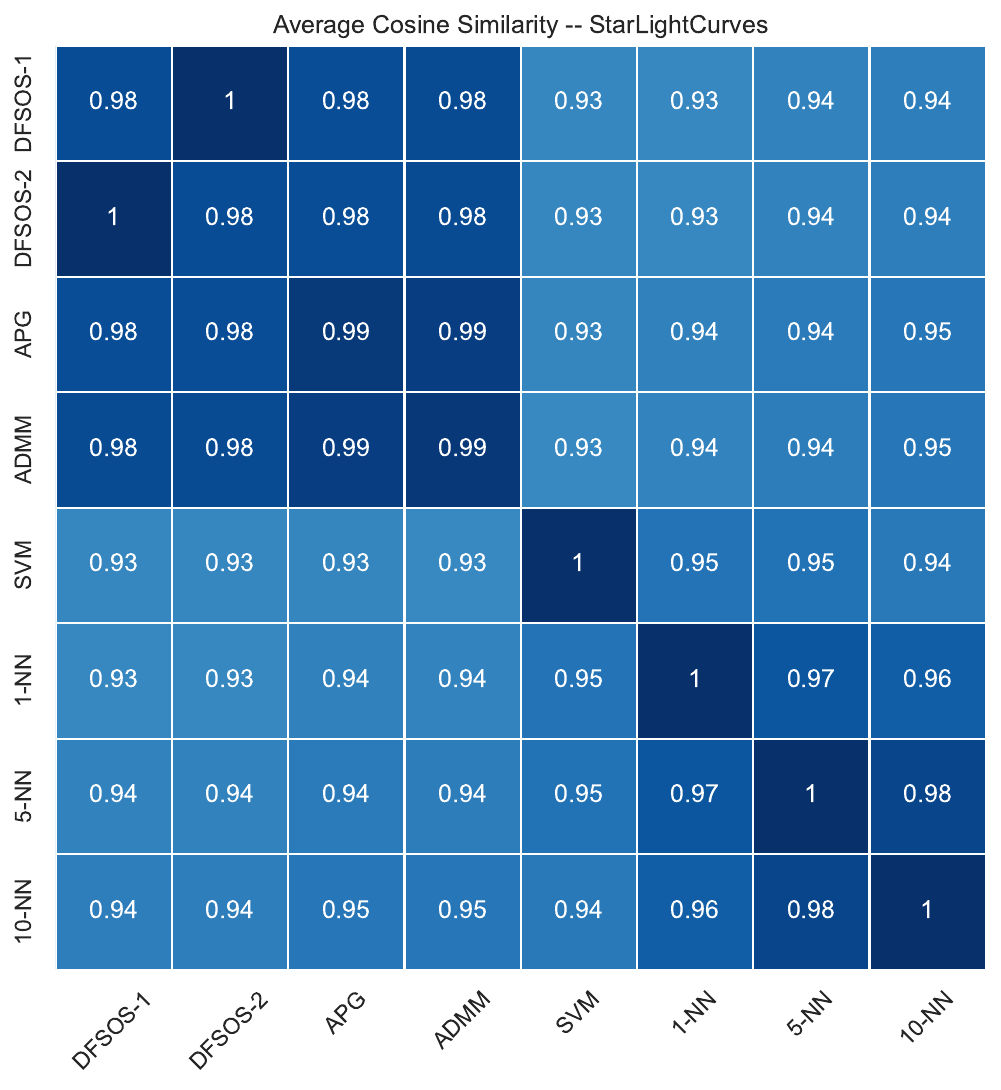}
        \caption{StarLightCurves}
        \label{fig:StarLightCurves}
    \end{subfigure}
        \hfill
    \begin{subfigure}[t]{0.19\textwidth}
        \centering
        \includegraphics[width=0.99\linewidth]{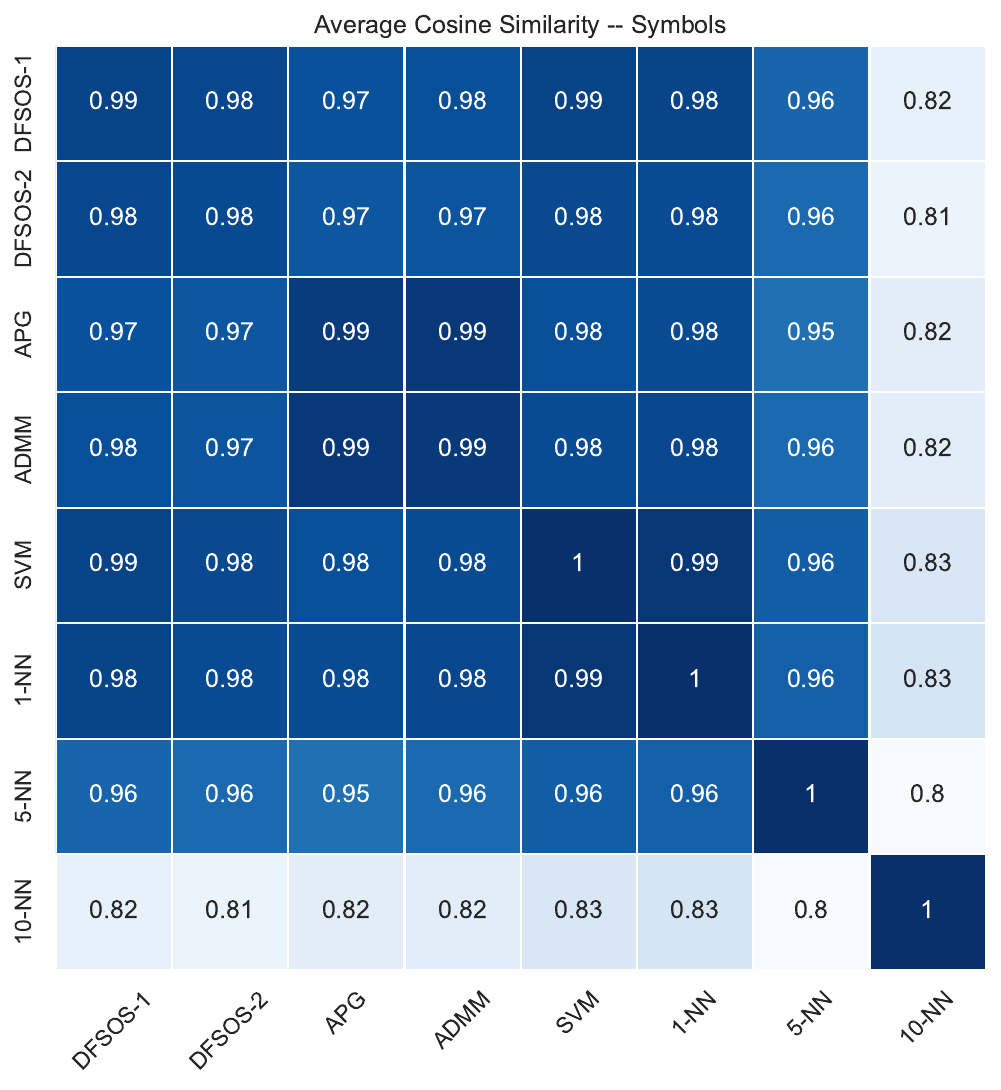}
        \caption{Symbols}
        \label{fig:Symbols}
    \end{subfigure}    
    \hfill
    \begin{subfigure}[t]{0.19\textwidth}
        \centering
        \includegraphics[width=0.99\linewidth]{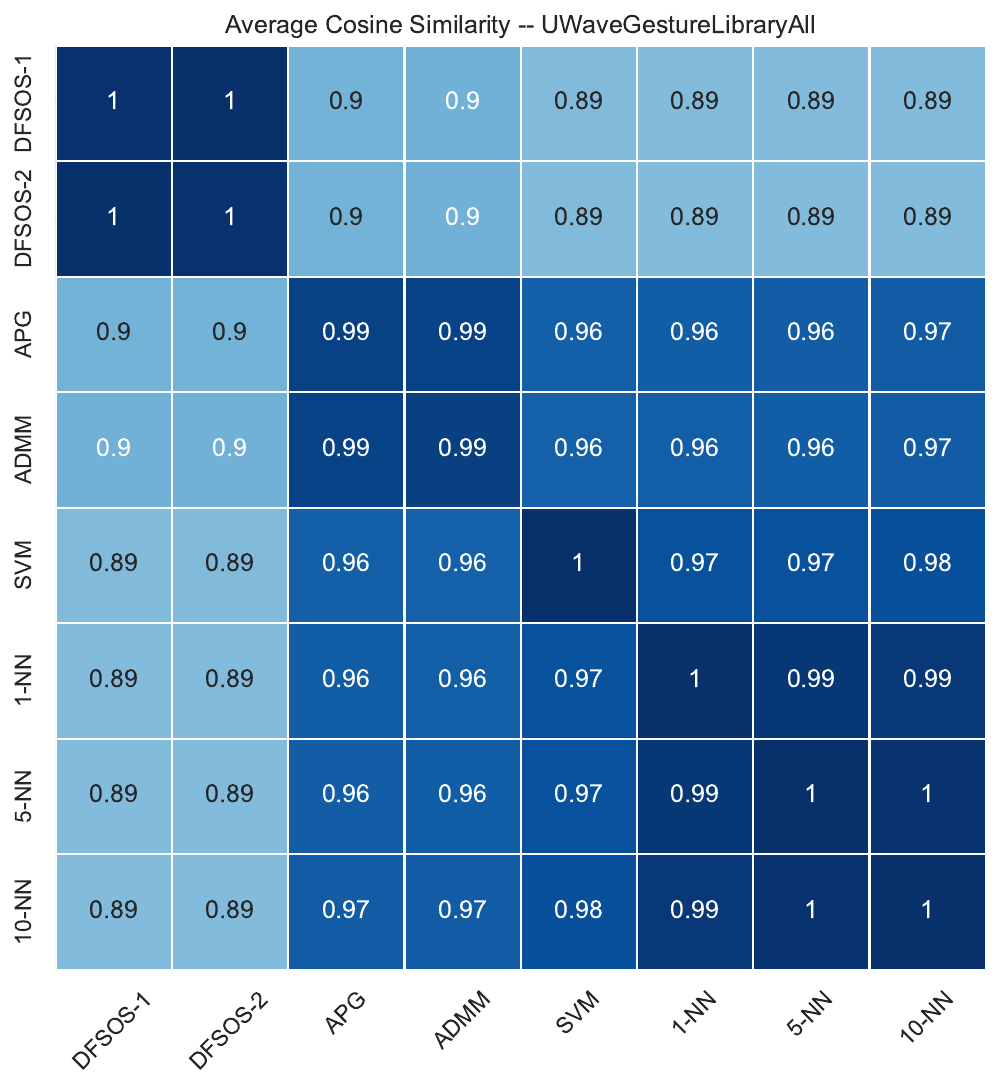}
        \caption{UWaveGestureLibraryAll}
        \label{fig:UWaveGestureLibraryAll}
    \end{subfigure}
        \caption{Average cosine similarity between out-of-sample predictions for each pair of methods. The diagonal entries indicate the average cosine similarity between all pairs of prediction vectors made by each method across all $10$ repetitions for each data set.}

    \label{fig:UCR-predictions}
\end{figure}

%%%%%%%%%%%%%%%%%%%%%%%%%%%%%%%%%%%%%%%%%%%%%%%%%%%%%%%
% Conclusion
%%%%%%%%%%%%%%%%%%%%%%%%%%%%%%%%%%%%%%%%%%%%%%%%%%%%%%%
\section{Conclusion}
We have proposed an algorithmic framework for \emph{deflation-free sparse optimal scoring (DFSOS)}, a novel approach to sparse discriminant analysis that addresses fundamental limitations of existing deflation-based methods. By reformulating the sparse optimal scoring problem with an explicit global orthogonality constraint, DFSOS computes all discriminant vectors simultaneously rather than sequentially, eliminating error propagation inherent in deflation-based approaches.

The proposed algorithm decomposes the problem into tractable subproblems for scoring vectors, discriminant vectors, and orthogonality enforcement. We establish convergence of DFSOS to stationary points of the augmented Lagrangian under mild conditions on the penalty parameter, providing theoretical grounding for the method.
Extensive numerical experiments on both synthetic Gaussian data and real-world time series from the UCR Time-Series Repository demonstrate that DFSOS achieves classification accuracy comparable to or better than existing deflation-based methods (e.g., APG and ADMM variants of ASDA), particularly in challenging settings with high feature correlation.

While DFSOS requires modestly increased computational time compared to ASDA methods (typically 1-2 seconds additional per trial), it delivers statistically significant improvements in classification accuracy across most experimental settings. The method produces slightly denser discriminant vectors than ASDA, reflecting a favorable trade-off between sparsity and discriminative power. Notably, DFSOS achieves classification performance matching state-of-the-art support vector machines on linearly separable data while maintaining the interpretability advantages of sparse linear methods.

The observed increase in computation used by DFSOS compared to the existing ASDA methods suggests an important avenue for potential research: the development of faster algorithms for solution of~\eqref{eq: DFSOS final splitting}. The past few years have seen the development many alternatives to the splitting method of Lai and Osher~\cite{lai2014splitting} for orthogonality constrained optimization; see the recent manuscripts~\cite{absil2008optimization,boumal2023introduction,sato2021riemannian} for comprehensive treatments of optimization on manifolds.
Specialization of these methods to the deflation-free sparse optimal scoring problem~\eqref{eq: DFSOS final splitting} could lead to improved computational efficiency. We plan a rigorous empirical comparison of specialization of a variety of state of the art algorithms for manifold optimization to deflation-free sparse optimal scoring.
Moreover, further study of the convergence and consistency properties of the proposed DFSOS algorithm is needed; for example, analysis of the rate of convergence of DFSOS would greatly improve our understanding of the algorithm.

The consistent improvements observed in our experiments, particularly for highly correlated data, suggest that deflation-free methods offer a more principled and effective approach to sparse discriminant analysis. By directly enforcing global orthogonality constraints rather than sequential deflation, DFSOS provides a coherent framework for feature selection and dimension reduction in high-dimensional classification problems. This demonstrates that the additional computational structure required to handle orthogonality constraints globally yields tangible benefits in classification performance and numerical stability, making deflation-free approaches a valuable addition to the sparse discriminant analysis toolbox.

% %%%%%%%%%%%%%%%%%%%%%%%%%%%%%%%%%%%%%%%%%%%%%%%%%%%%%%%
\subsection*{Acknowledgments}
% %%%%%%%%%%%%%%%%%%%%%%%%%%%%%%%%%%%%%%%%%%%%%%%%%%%%%%%

The authors acknowledge the use of the IRIDIS High Performance Computing Facility, and associated support services at the University of Southampton, in the completion of this work.
B.~Ames was supported in part by the National Science Foundation Grants \#20212554 and \#2108645, as well as the University of Alabama Research Grants RG14678 and RG14838.

%%%%%%%%%%%%%%%%%%%%%%%%%%%%%%%%%%%%%%%%%%%%%%%%%%%%%%%
% Bibliography
%%%%%%%%%%%%%%%%%%%%%%%%%%%%%%%%%%%%%%%%%%%%%%%%%%%%%%%

\bibliography{dfsos_bib}

\end{document}